\documentclass{article} 
\usepackage{hyperref}
\usepackage{url}

\usepackage[utf8]{inputenc} 
\usepackage[T1]{fontenc}    
\usepackage{hyperref}       
\usepackage{url}            
\usepackage{booktabs}       
\usepackage{amsfonts}       
\usepackage{nicefrac}       
\usepackage{microtype}      
\usepackage{hyperref}
\usepackage{graphicx}
\usepackage{url}
\usepackage{bm}
\usepackage[export]{adjustbox}
\usepackage{amsmath}
\usepackage{amsbsy}
\usepackage{calc}
\usepackage{amsthm}
\usepackage{amssymb}

\newtheorem{theorem}{Theorem}
\newtheorem{definition}{Definition}

\usepackage[accepted]{icml2018}

\icmltitlerunning{Can Deep Reinforcement Learning Solve Erdos-Selfridge-Spencer Games?}

\begin{document}

\twocolumn[
\icmltitle{Can Deep Reinforcement Learning
           Solve Erdos-Selfridge-Spencer Games?}



\icmlsetsymbol{equal}{*}

\begin{icmlauthorlist}
\icmlauthor{Maithra Raghu}{google,cornell}
\icmlauthor{Alex Irpan}{google}
\icmlauthor{Jacob Andreas}{berkeley}
\icmlauthor{Robert Kleinberg}{cornell}
\icmlauthor{Quoc Le}{google}
\icmlauthor{Jon Kleinberg}{cornell}
\end{icmlauthorlist}

\icmlaffiliation{cornell}{Cornell University}
\icmlaffiliation{google}{Google Brain}
\icmlaffiliation{berkeley}{University of California, Berkeley}

\icmlcorrespondingauthor{Maithra Raghu}{maithrar@gmail.com}

\icmlkeywords{Machine Learning, ICML, reinforcement learning}

\vskip 0.3in
]



\printAffiliationsAndNotice{}  

\begin{abstract}
Deep reinforcement learning has achieved many recent successes, but our understanding of its strengths and limitations is hampered by the lack of rich environments in which we can fully characterize optimal behavior, and correspondingly diagnose individual actions against such a characterization. Here we consider a family of combinatorial games, arising from work of Erdos, Selfridge, and Spencer, and we propose their use as environments for evaluating and comparing different approaches to reinforcement learning. These games have a number of appealing features: they are challenging for current learning approaches, but they form (i) a low-dimensional, simply parametrized environment where (ii) there is a linear closed form solution for optimal behavior from any state, and (iii) the difficulty of the game can be tuned by changing environment parameters in an interpretable way. We use these Erdos-Selfridge-Spencer games not only to compare different algorithms, but test for generalization, make comparisons to supervised learning, analyze multiagent play, and even develop a self play algorithm.  
\end{abstract}

\section{Introduction}
Deep reinforcement learning has seen many remarkable successes over
the past few years \cite{mnih2015dqn,silver2017alphagozero}.
But developing learning algorithms that are robust
across tasks and policy representations remains a challenge \cite{hendersondeeprlmatters}. 
Standard benchmarks like MuJoCo and Atari provide rich settings
for experimentation, but the specifics of the underlying environments
differ from each other in multiple ways, and hence
determining the principles underlying any particular form of 
sub-optimal behavior is difficult.
Optimal behavior in these environments is generally
complex and not fully characterized, so algorithmic success is generally
associated with high scores, typically on a copy of the training environment making it hard to analyze where errors
are occurring or evaluate generalization.

An ideal setting for studying the strengths and limitations of
reinforcement learning algorithms would be  (i) a simply parametrized
family of environments where (ii) optimal behavior can be completely
characterized and (iii) the environment is rich enough to support interaction
and multiagent play.

To produce such a family of environments, we look in a novel direction -- 
to a set of two-player combinatorial games with their roots in work of Erdos and Selfridge \citep{erdos1973game}, and placed on a general footing by Spencer (\citeyear{spencer1994game}). Roughly speaking, these {\em Erdos-Selfridge-Spencer (ESS) games} are games in which two players take turns selecting objects from some combinatorial structure, with the feature that optimal strategies can be defined by potential functions derived from conditional expectations over random future play. 

These ESS games thus provide an opportunity to capture the general desiderata 
noted above, with a clean characterization of optimal behavior and a 
set of instances that range from easy to very hard as we sweep over a
simple set of tunable parameters.
We focus in particular on one of the best-known games in
this genre, {\em Spencer's attacker-defender game} (also known as the
``tenure game''; \citeauthor{spencer1994game}, \citeyear{spencer1994game}), in which --- roughly speaking ---
an {\em attacker} advances a set of pieces up the levels of a board,
while a {\em defender} destroys subsets of these pieces to try prevent any
of them from reaching the final level (\autoref{fig-schematic}).
An instance of the game can be parametrized by two key quantities.
The first is the number of levels $K$, which determines both the size
of the state space and the approximate length of the game; the latter
is directly related to the sparsity of win/loss signals as rewards.
The second quantity is a {\em potential function} $\phi$, whose
magnitude characterizes whether the instance favors the defender or attacker,
and how much ``margin of error'' there is in optimal play.

The environment therefore allows us to study learning by the defender and 
attacker, separately or concurrently in multiagent and self-play.
In the process, we are able to develop insights about the robustness
of solutions to changes in the environment. 
These types of analyses have been long-standing goals, but they have
generally been approached much more abstractly, given the difficulty
in characterizing step-by-step optimality in non-trivial environments
such as this one.
Because we have a move-by-move characterization of optimal play, we can
go beyond simple measures of reward based purely on win/loss outcomes
and use supervised learning techniques to pinpoint the exact location
of the errors in a trajectory of play.

The main contributions of this work are thus the following:
\begin{enumerate}
\item We develop these combinatorial
games as environments for studying the behavior of reinforcement learning
algorithms in a setting where it is possible to characterize optimal play and to tune the underlying difficulty using natural parameters.
\item We show how reinforcement learning algorithms in this domain are able to learn generalizable policies in addition to simply achieving high performance, and through new combinatorial results about the domain, we are able to develop strong methods for {\em multiagent} play that enhance generalization.
\item Through an extension of our combinatorial results, we show how this domain lends itself  to a subtle self-play algorithm, which achieves a significant improvement in performance.
\item We can characterize optimal play at a move-by-move level and thus compare the performance of a deep RL agent to one trained using supervised learning
on move-by-move decisions. By doing so, we discover
an intriguing phenomenon: while
the supervised learning agent is more accurate
on individual move decisions than the RL agent, the RL agent
is better at playing the game! We further interpret this result by defining a notion of \textit{fatal mistakes}, and showing that while the deep RL agent makes more mistakes overall, it makes fewer fatal mistakes.
\end{enumerate}

In summary, we present learning and
generalization experiments for a variety of commonly used model
architectures and learning algorithms. We show that despite the superficially
simple structure of the game, it provides both significant challenges
for standard reinforcement learning approaches and a number of tools
for precisely understanding those challenges.

\section{Erdos-Selfridge-Spencer Attacker Defender Game}
\label{sec-ess-defn}
We first introduce the family of 
Attacker-Defender Games \citep{spencer1994game},
a set of games with two properties that yield a particularly attractive
testbed for deep reinforcement learning: the ability to continuously vary the
difficulty of the environment through two parameters, and the
existence of a closed form solution that is expressible as a
\textit{linear model}.  

\begin{figure}
\centering
   \includegraphics[scale=0.2]{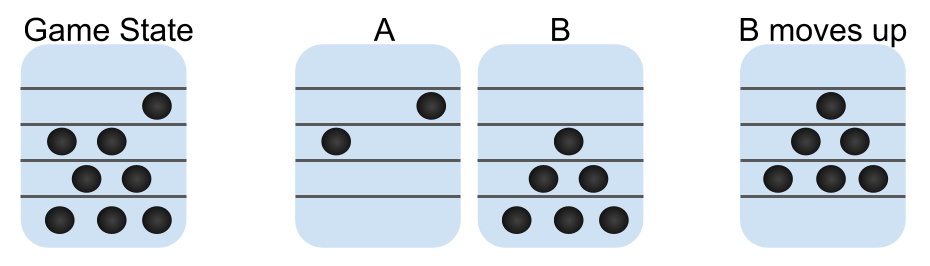}
  \caption{ \small One turn in an ESS Attacker-Defender game. The attacker proposes a partition $A, B$ of the current game state, and the defender chooses one set to destroy (in this case $A$). Pieces in the remaining set ($B$) then move up a level to form the next game state.}
  \label{fig-schematic}
\end{figure}

An Attacker-Defender game involves two players: an attacker who moves
pieces, and a defender who destroys pieces. An instance of the game
has a set of {\em levels} numbered from $0$ to $K$, and $N$ pieces that are
initialized across these levels. The attacker's goal is to get at
least one of their pieces to level $K$, and the defender's goal is to
destroy all $N$ pieces before this can happen. In each turn, the attacker
proposes a partition $A, B$ of the pieces still in play. The defender
then chooses one of the sets to destroy and remove from play. All
pieces in the other set are moved up a level. The game ends when
either one or more pieces reach level $K$, or when all pieces are
destroyed. Figure \ref{fig-schematic} shows one turn of play.

With this setup, varying the number of levels $K$ or the number of
pieces $N$ changes the difficulty for the attacker and the defender.
One of the most striking aspects of the Attacker-Defender game is that
it is possible to make this trade-off precise, and en route to doing so, also
identify a \textit{linear optimal policy.} 

We start with a simple special case --- rather than initializing
the board with pieces placed arbitrarily, we require the pieces to all
start at level 0. In this special case, we can directly think of 
the game's difficulty in terms of the number of levels $K$ and 
the number of pieces $N$.

\begin{theorem}
\label{thm-gamev1} Consider an instance of the Attacker-Defender game with $K$ levels and $N$ pieces, with all $N$ pieces starting at level $0$. Then if $N < 2^K$, the defender can \emph{always} win.
\end{theorem} 

There is a simple proof of this fact: the defender simply always destroys
the larger one of the sets $A$ or $B$.
In this way, the number of pieces is reduced by at least a factor of two
in each step; since a piece must travel $K$ steps in order to reach level $K$,
and $N < 2^K$, no piece will reach level $K$.

When we move to the more general case in which the board is initialized
at the start of the game with pieces placed at arbitrary levels,
it will be less immediately clear how to define the ``larger''
one of the sets $A$ or $B$.  
We therefore describe a second proof of Theorem \ref{thm-gamev1}
that will be useful in these more general settings.
This second 
proof, due to \citet{spencer1994game}, uses Erdos's probabilistic method 
and proceeds as follows.

For any attacker strategy, assume the defender plays randomly.
Let $T$ be a random variable for the number of pieces that reach level $K$. Then $T = \sum T_i$ where $T_i$ is the indicator that piece $i$ reaches level $K$. 
But then $E\left[T\right] = \sum E\left[T_i\right] = \sum_i 2^{-K}$: as the defender is playing randomly, any piece has probability $1/2$ of advancing a level and $1/2$ of being destroyed. As all the pieces start at level $0$, they must advance $K$ levels to reach the top, which happens with probability $2^{-K}$. But now, by choice of $N$, we have that  $\sum_i 2^{-K}=  N2^{-K} < 1$. Since $T$ is an integer random variable, $E\left[T\right] < 1$ implies that the distribution of $T$ has nonzero mass at $0$ - in other words there is some set of choices for the defender that guarantees destroying all pieces. 
This means that the attacker does not have a strategy that
wins with probability 1 against random play by the defender; since
the game has the property that one player or the other must be able
to force a win, it follows that the defender can force a win.
This completes the proof.

Now consider the general form of the game, in which the initial
configuration can have pieces at arbitrary levels.
Thus, at any point in time, the state of the game can be described by
a $K$-dimensional vector $S = (n_0, n_1,...,n_K)$, with $n_i$ the number of
pieces at level $i$.

Extending the argument used in the second proof above,
we note that a piece at level $l$ has a $2^{-(K-l)}$ 
chance of survival under random play. 
This motivates the following {\em potential function} on states:
\begin{definition}
\label{def-potfn}
\emph{Potential Function}: Given a game state $S = (n_0, n_1,...,n_K)$, 
we define the \emph{potential} of the 
state as $ \phi(S) = \sum_{i=0}^K n_i 2^{-(K-i)}$. 
\end{definition} 

Note that this is a \textit{linear} function on the input state, expressible as $\phi(S) = w^T \cdot S$ for $w$ a vector with $w_l = 2^{-(K-l)}$. 
We can now state the following generalization of Theorem \ref{thm-gamev1}, again due to \citet{spencer1994game}.
\begin{theorem}[\cite{spencer1994game}]
\label{thm-gamev2}
Consider an instance of the Attacker-Defender game that has $K$ levels and $N$ pieces, with pieces placed anywhere on the board, and let the initial state be $S_0$. Then
\begin{enumerate}
\item[(a)] If $\phi(S_0) < 1$, the defender can always win
\item[(b)] If $\phi(S_0) \geq 1$, the attacker can always win.
\end{enumerate}
\end{theorem}
One way to prove part (a) of this theorem is by directly extending the
proof of Theorem \ref{thm-gamev1}, with $E\left[T\right] = \sum
E\left[T_i\right] = \sum_i 2^{-(K - l_i)}$ where $l_i$ is the level of
piece $i$. After noting that $\sum_i 2^{-(K - l_i)} = \phi(S_0) < 1$
by our definition of the potential function and choice of $S_0$, we
finish off as in Theorem \ref{thm-gamev1}.

This definition of the potential function gives a natural, concrete
strategy for the defender: the defender simply destroys whichever of
$A$ or $B$ has higher potential. 
We claim that if $\phi(S_0) < 1$, then this strategy
guarantees that any subsequent state $S$ will also have $\phi(S) < 1$.
Indeed, suppose (renaming the sets if necessary) that $A$ has a
potential at least as high as $B$'s, 
and that $A$ is the set destroyed by the defender.
Since $\phi(B) \leq \phi(A)$ and $\phi(A) + \phi(B) = \phi(S) < 1$, the next state has potential $2\phi(B)$ (double the potential of $B$ as all pieces move up a level) which is also less than $1$.
In order to win, the attacker would need to place a piece on level $K$,
which would produce a set of potential at least $1$.
Since all sets under the defender's strategy have potential strictly less
than 1, it follows that no piece ever reaches level $K$.

For $\phi(S_0) \geq 1$, \citet{spencer1994game} proves part (b) of the theorem by defining an optimal strategy for the attacker, using a greedy algorithm to pick two sets $A, B$ each with potential $\geq 0.5$. For our purposes, the proof from \citet{spencer1994game} results in an intractably large action space for the attacker; we therefore (in Theorem \ref{thm-prefix-attacker} later in the paper) define a new kind of attacker --- the \textit{prefix-attacker} --- and we prove its optimality. These new combinatorial insights about the game enable us to later perform multiagent play, and subsequently self-play.

\section{Related Work}

The Atari benchmark \citep{mnih2015dqn} is a well known set of tasks, ranging from easy to solve (Breakout, Pong) to very difficult (Montezuma's Revenge). \citet{duan2016benchmarking} proposed a set of continuous environments, implemented in the MuJoCo simulator \cite{todorov2012mujoco}. An advantage of physics based environments is that they can be varied continuously by changing physics parameters \citep{rajeswaran2016epopt}, or by randomizing rendering \citep{tobin2017domain}. DeepMind Lab \citep{deepmindlab} is a set of 3D navigation based environments. OpenAI Gym \citep{openaigym} contains both the Atari and MuJoCo benchmarks, as well as classic control environments like Cartpole \citep{stephenson1909cartpole} and algorithmic tasks like copying an input sequence. The difficulty of algorithmic tasks can be easily increased by increasing the length of the input. Automated game playing in algorithmic settings has also been explored outside of RL \citep{Bouzy2010MultiagentLE, Zinkevich2011Lemonade, Bowling2017poker, Littman1994markovgames}.  Our proposed benchmark merges properties of both the algorithmic tasks and physics-based tasks, letting us increase difficulty by discrete changes in length or continuous changes in potential.

\section{Deep Reinforcement Learning on the Attacker-Defender Game}
\label{sec-deeprl-ess}
From Section \ref{sec-ess-defn}, we see that the Attacker-Defender games are a family of environments with a difficulty knob that can be continuously adjusted through the start state potential $\phi(S_0)$ and the number of levels $K$. 
In this section, 
we describe a set of baseline results on Attacker-Defender games that motivate the exploration
in the remainder of this paper. We set up the Attacker-Defender environment as follows: the game state is represented by a $K+1$ dimensional vector for levels $0$ to $K$, with coordinate $l$ representing the number of pieces at level $l$. For the defender agent, the input is the concatenation of the partition $A, B$, giving a $2(K+1)$ dimensional vector. The start state $S_0$ is initialized randomly from a distribution over start states of a certain potential.

\begin{figure}
  \centering
  \begin{tabular}{cc}
  \vspace*{-5mm}
   \hspace*{-20mm}\includegraphics[width=0.5\columnwidth]{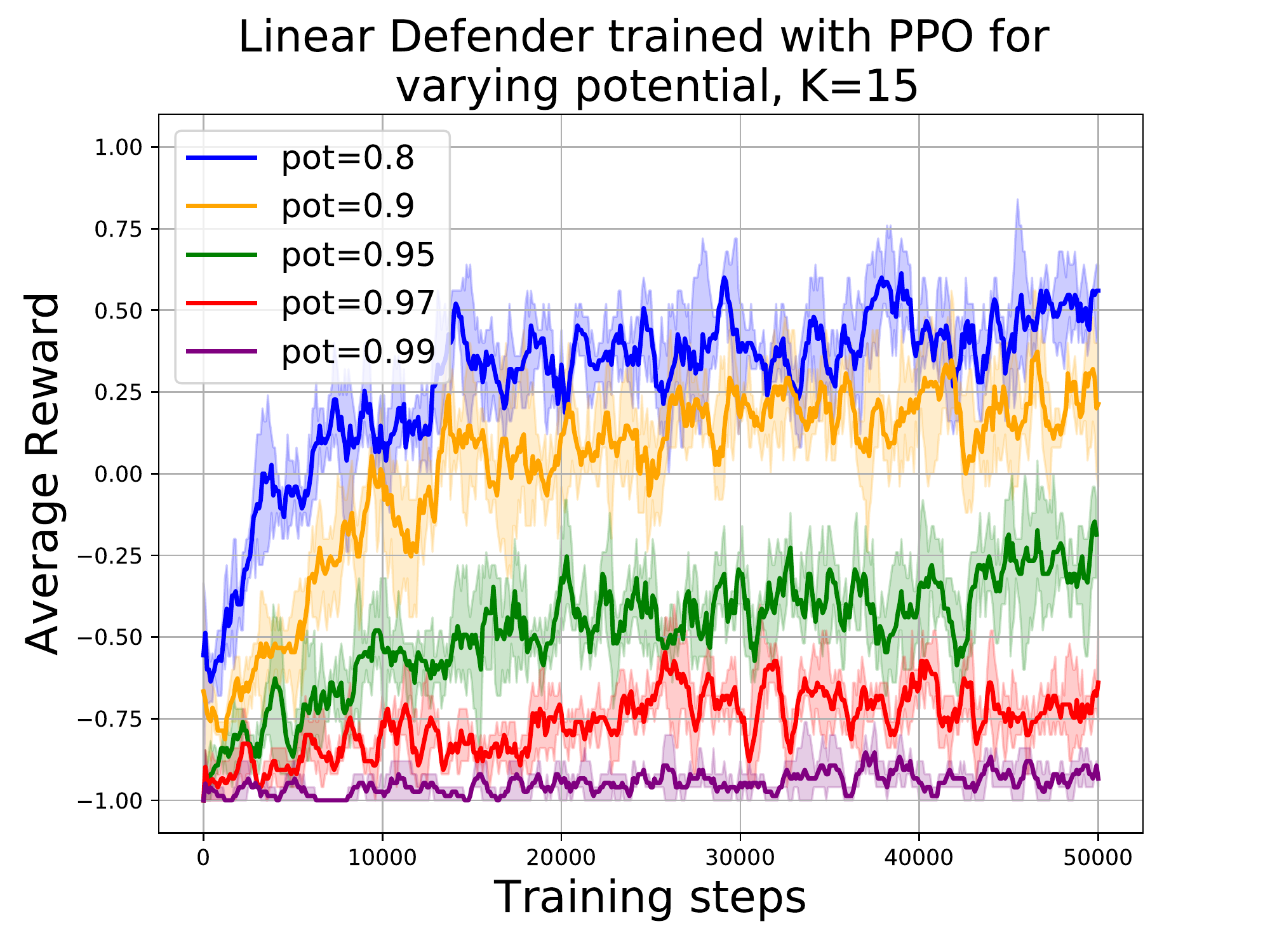}\hspace*{-20mm}
  &
  \hspace*{-15mm}\includegraphics[width=0.5\columnwidth]{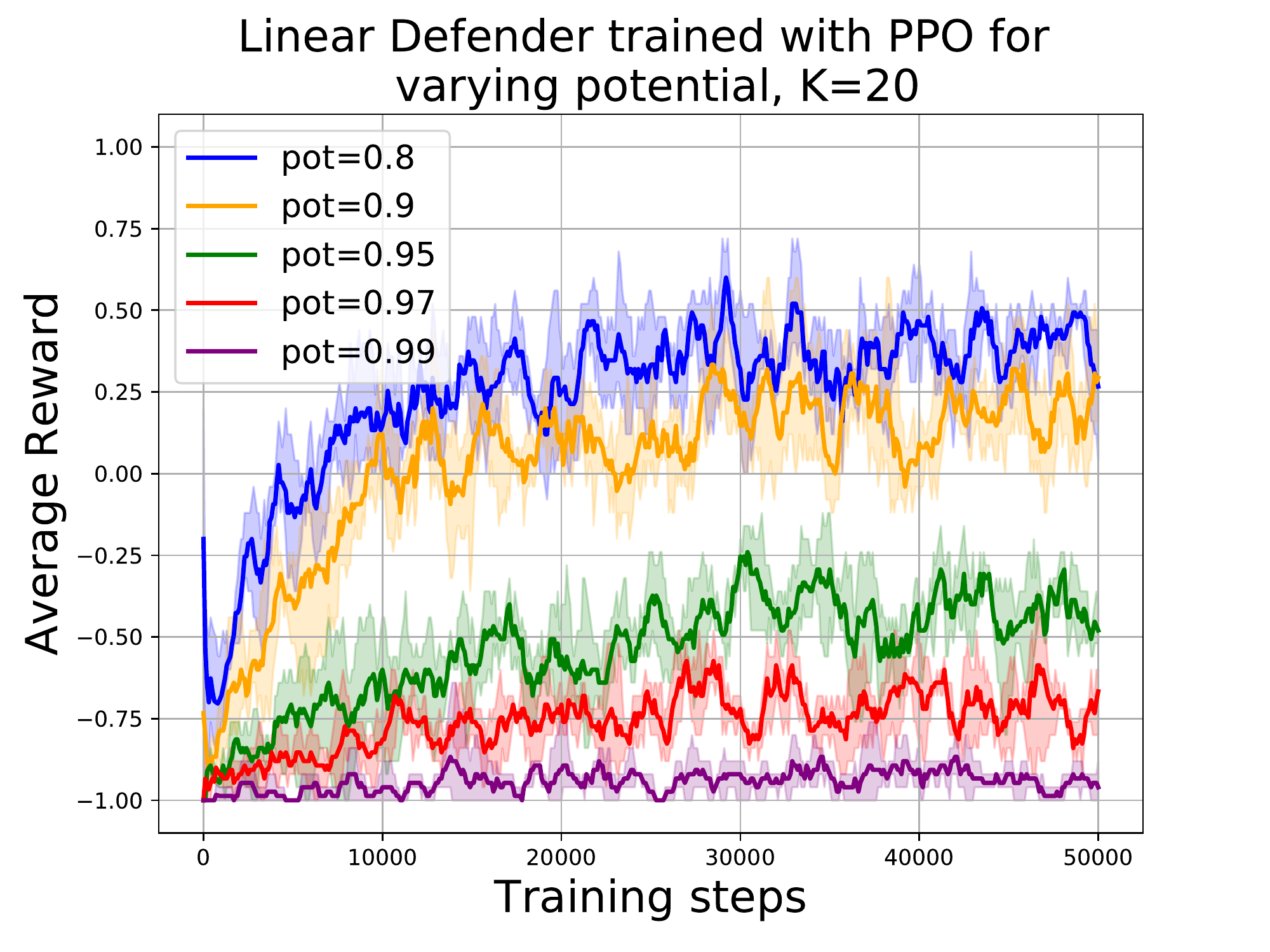}\hspace*{-9mm} 
  \\
  \hspace*{-20mm}
  \includegraphics[width=0.5\columnwidth]{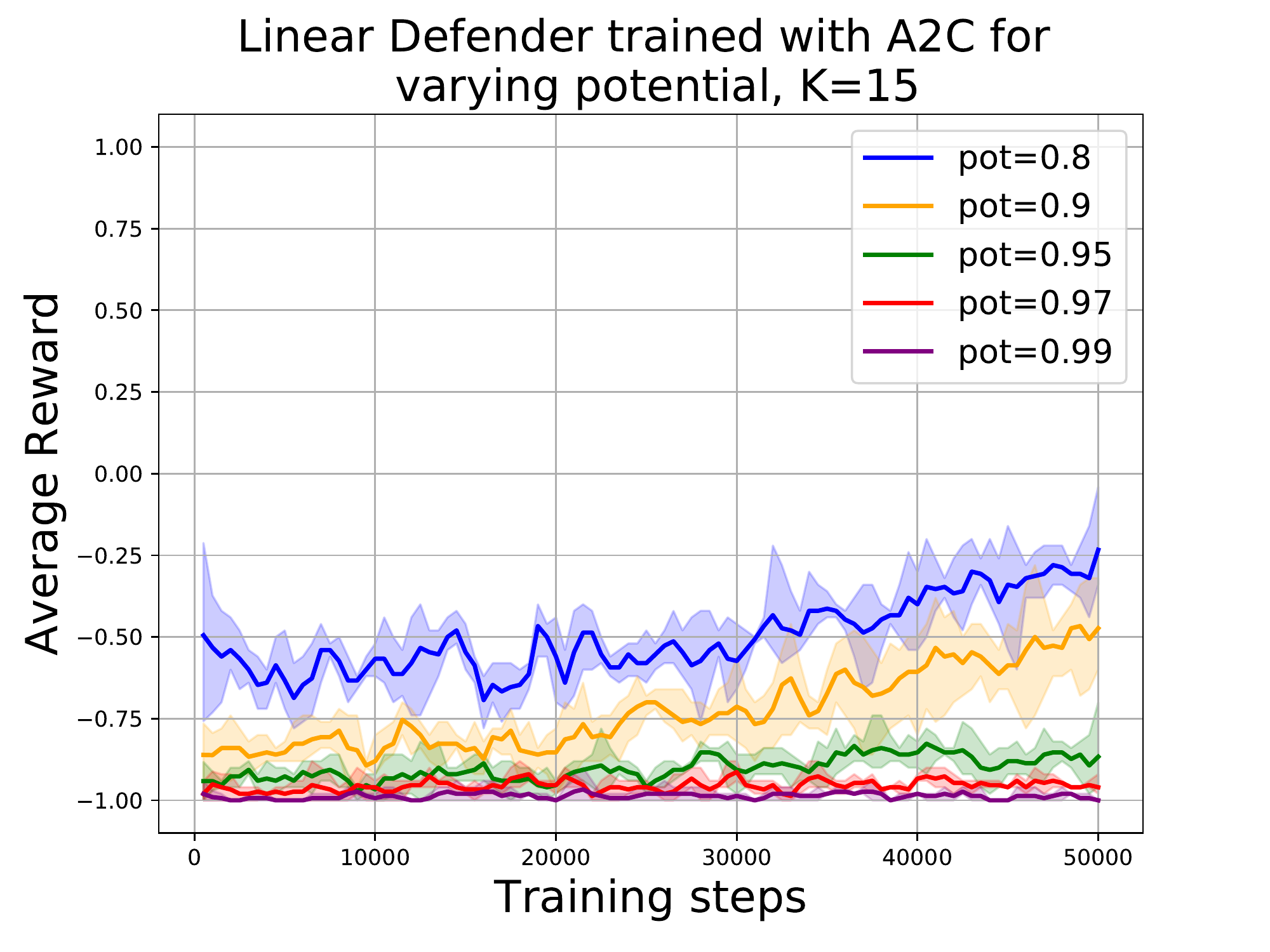}\hspace*{-20mm}
  &
    \hspace*{-15mm}
  \includegraphics[width=0.5\columnwidth]{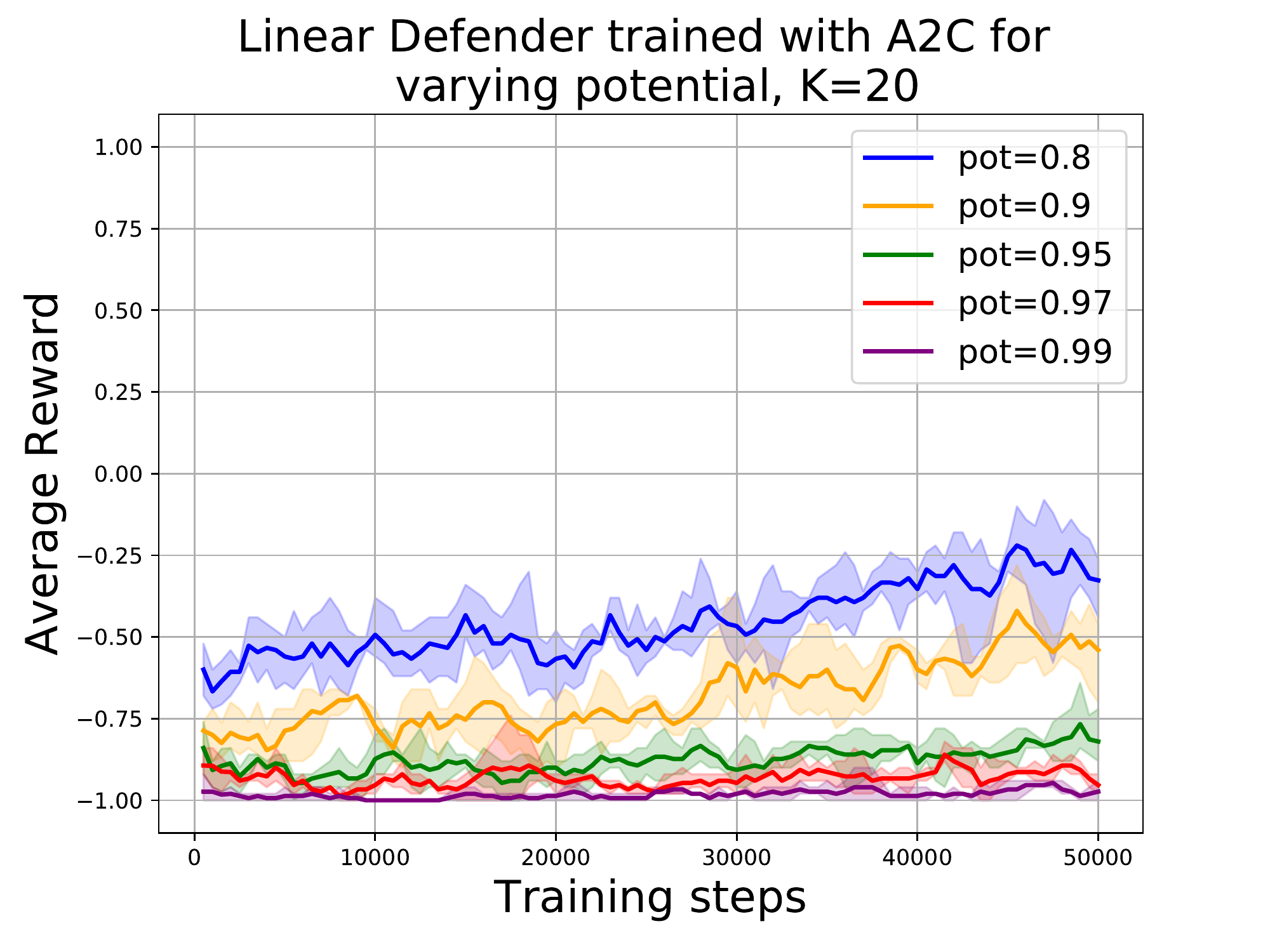}\hspace*{-9mm} 
  \\
  \includegraphics[width=0.5\columnwidth]{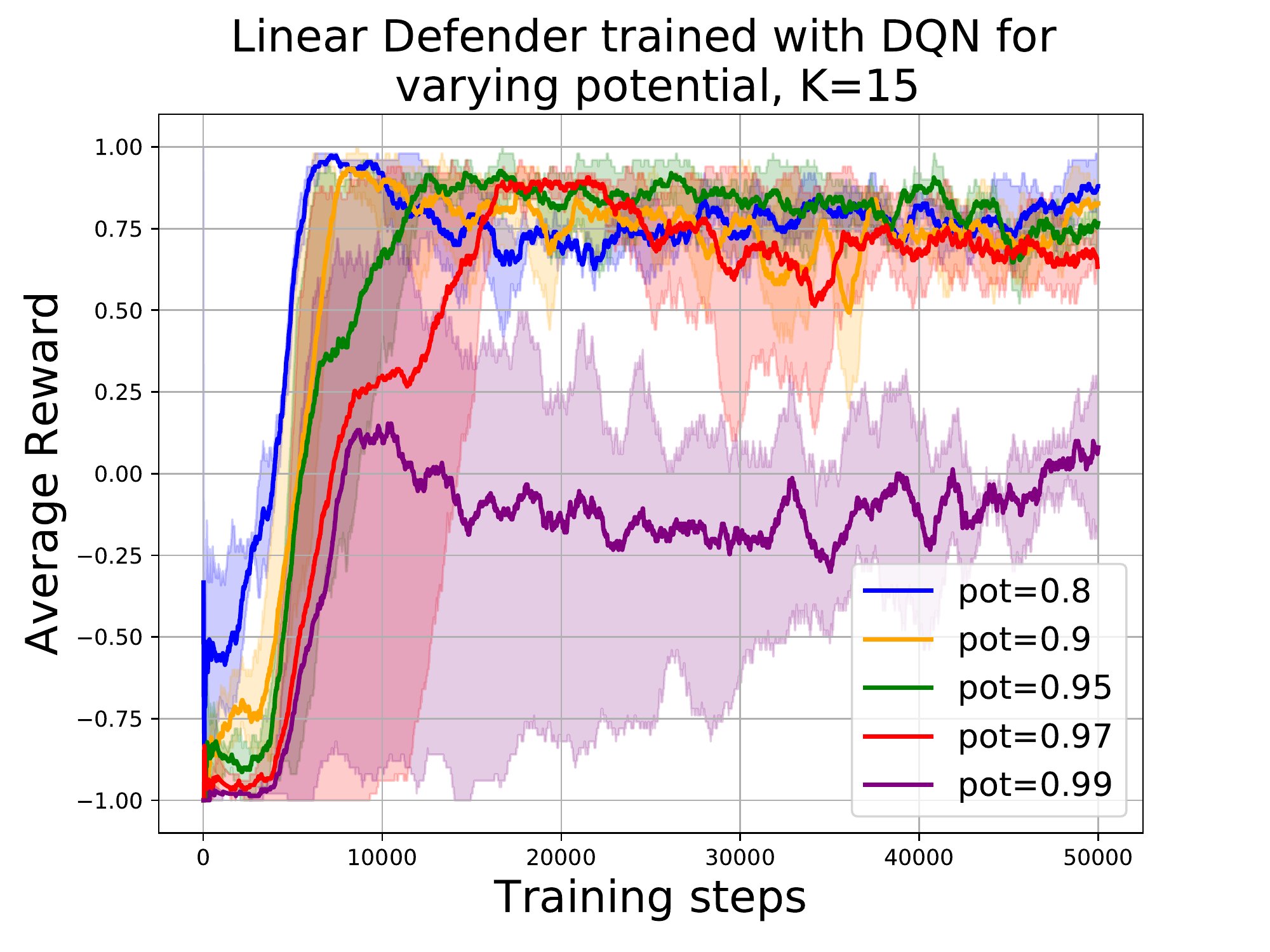}
  &
      \hspace*{-7mm}
  \includegraphics[width=0.5\columnwidth]{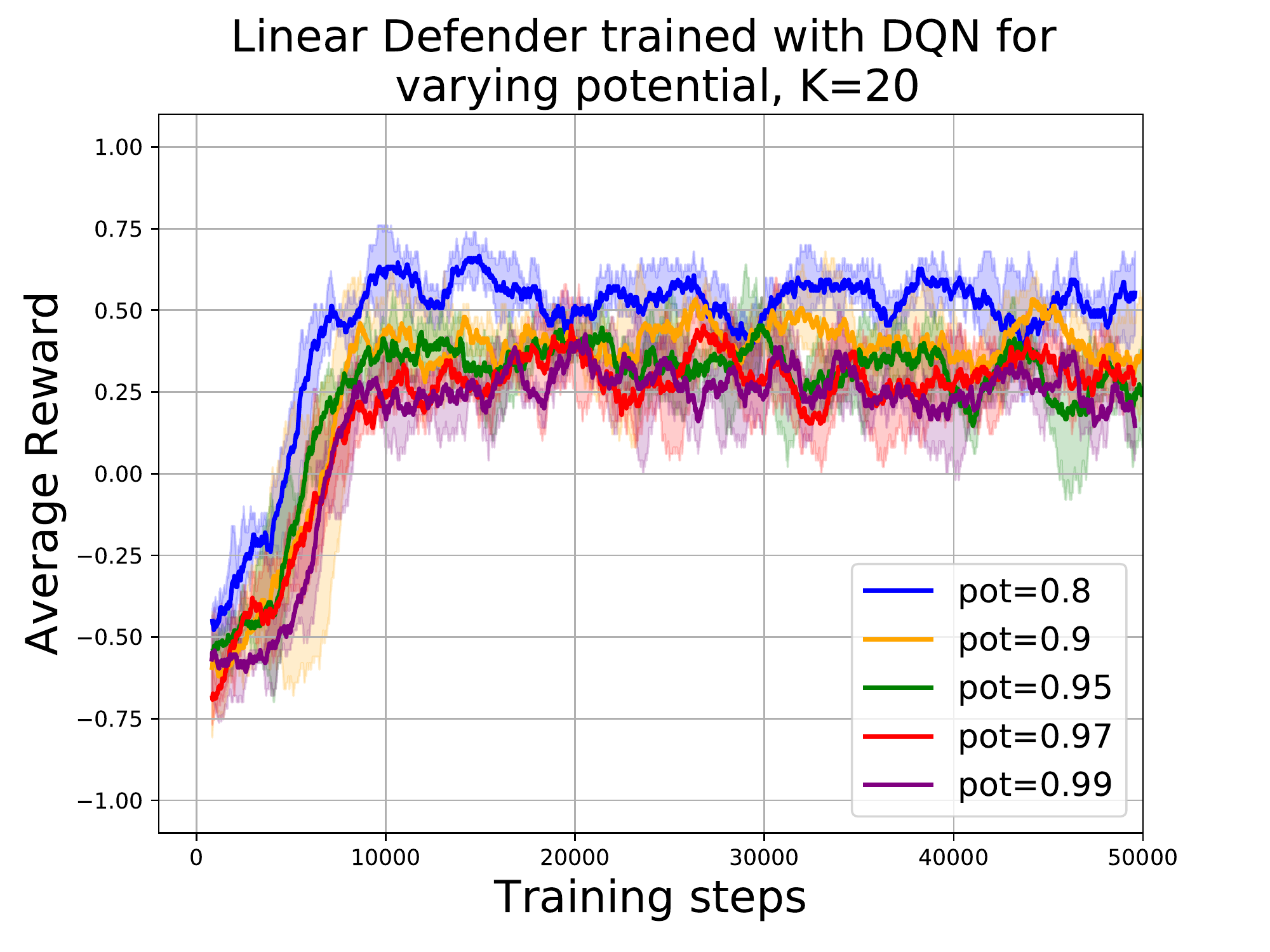}\\
  \end{tabular}
  \caption{\small Training a linear network to play as the defender agent with PPO, A2C and DQN. A linear model is theoretically expressive enough to learn the optimal policy for the defender agent. In practice, we see that for many difficulty settings and algorithms, RL struggles to learn the optimal policy and performs more poorly than when using deeper models (compare to Figure \ref{fig-defender-potential}). An exception to this is DQN which performs relatively well on all difficulty settings.}
  \label{fig-defender-linear}
\end{figure}

\subsection{Training a Defender Agent on Varying Environment Difficulties}
We first look at training a defender agent against an attacker that randomly chooses between (mostly) playing optimally, and (occasionally) playing suboptimally for exploration (with the \textit{Disjoint Support Strategy}, described in the Appendix.) Recall from the specification of the potential function, in Definition \ref{def-potfn} and Theorem \ref{thm-gamev2}, that the defender has a \textit{linear optimal policy}: given an input partition $A, B$, the defender simply computes $\phi(A) - \phi(B)$, with $\phi(A) = \sum_{i=0}^K a_i w_i$, where $a_i$ is the number of pieces in $A$ at level $i$;
$\phi(B) = \sum_{i=0}^K b_i w_i$, where $b_i$ is the number of pieces in $B$ at level $i$;
and $w_i = 2^{-(K-i)}$ is the weighting defining the potential function.

\begin{figure}
  \centering
  \begin{tabular}{cc}
  \vspace*{-5mm}
   \hspace*{-20mm}\includegraphics[width=0.5\columnwidth]{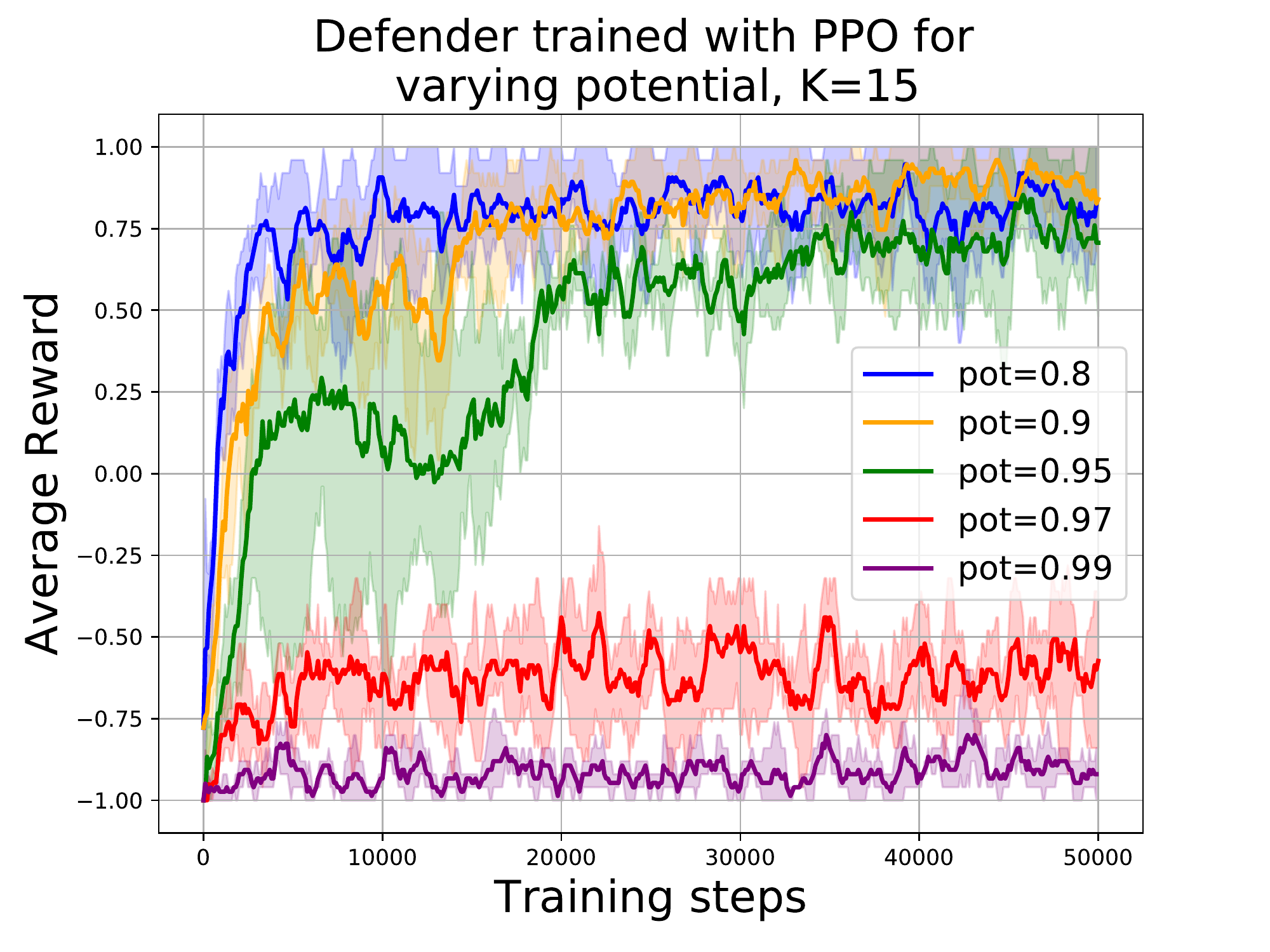}\hspace*{-20mm}
  &
  \hspace*{-15mm}\includegraphics[width=0.5\columnwidth]{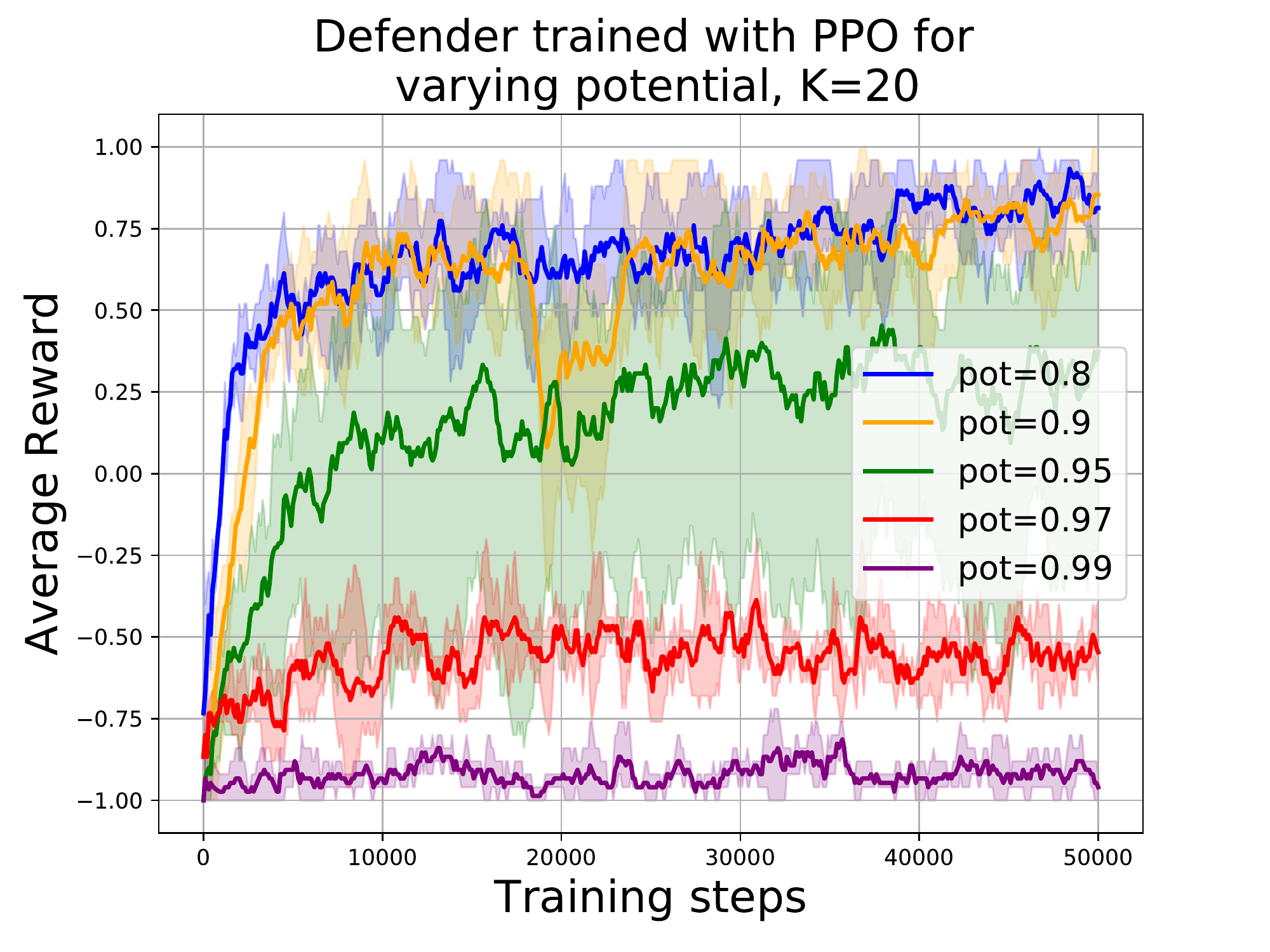}\hspace*{-9mm} 
  \\
  \hspace*{-20mm}
  \includegraphics[width=0.5\columnwidth]{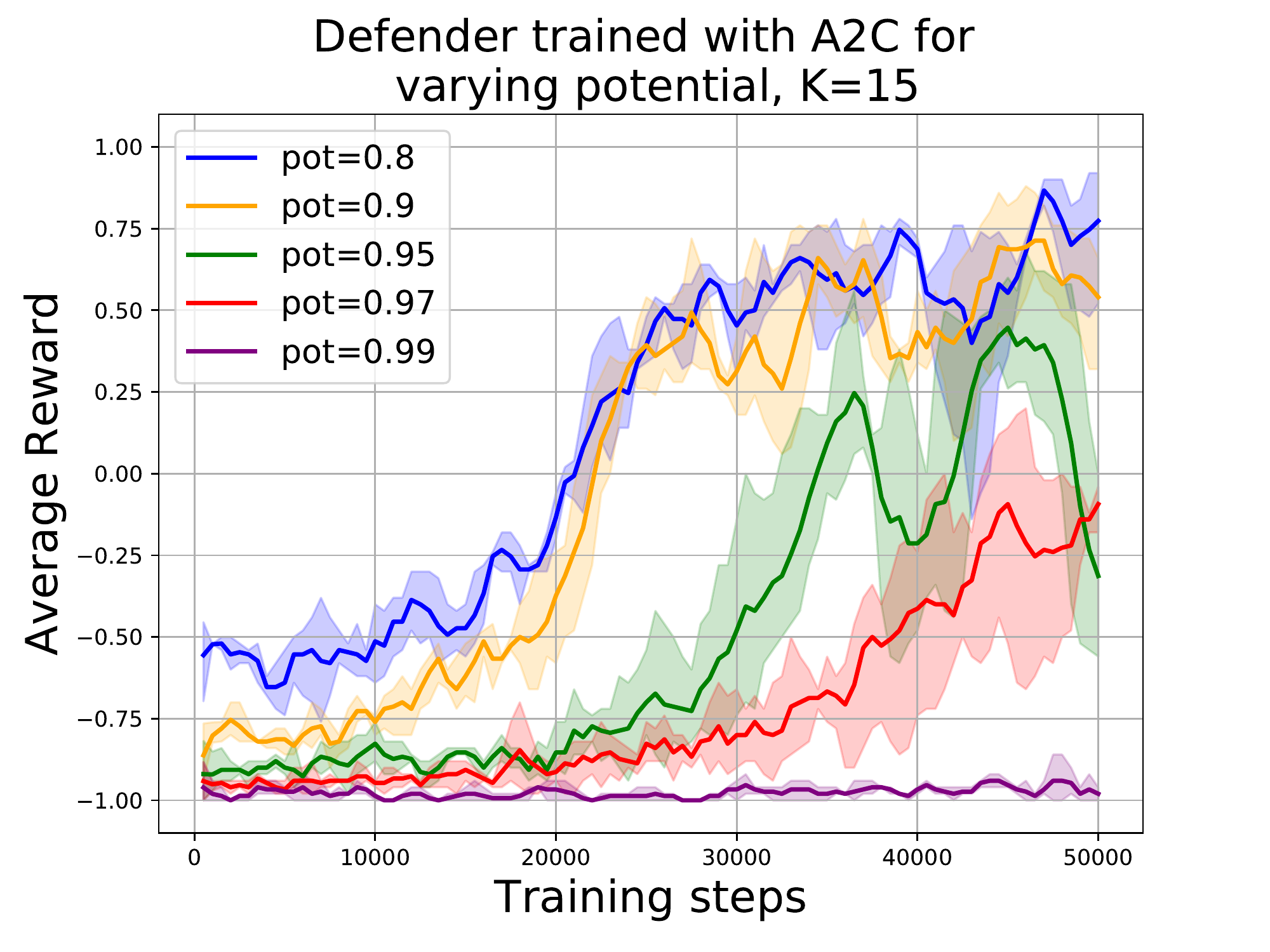}\hspace*{-20mm}
  &
    \hspace*{-15mm}
  \includegraphics[width=0.5\columnwidth]{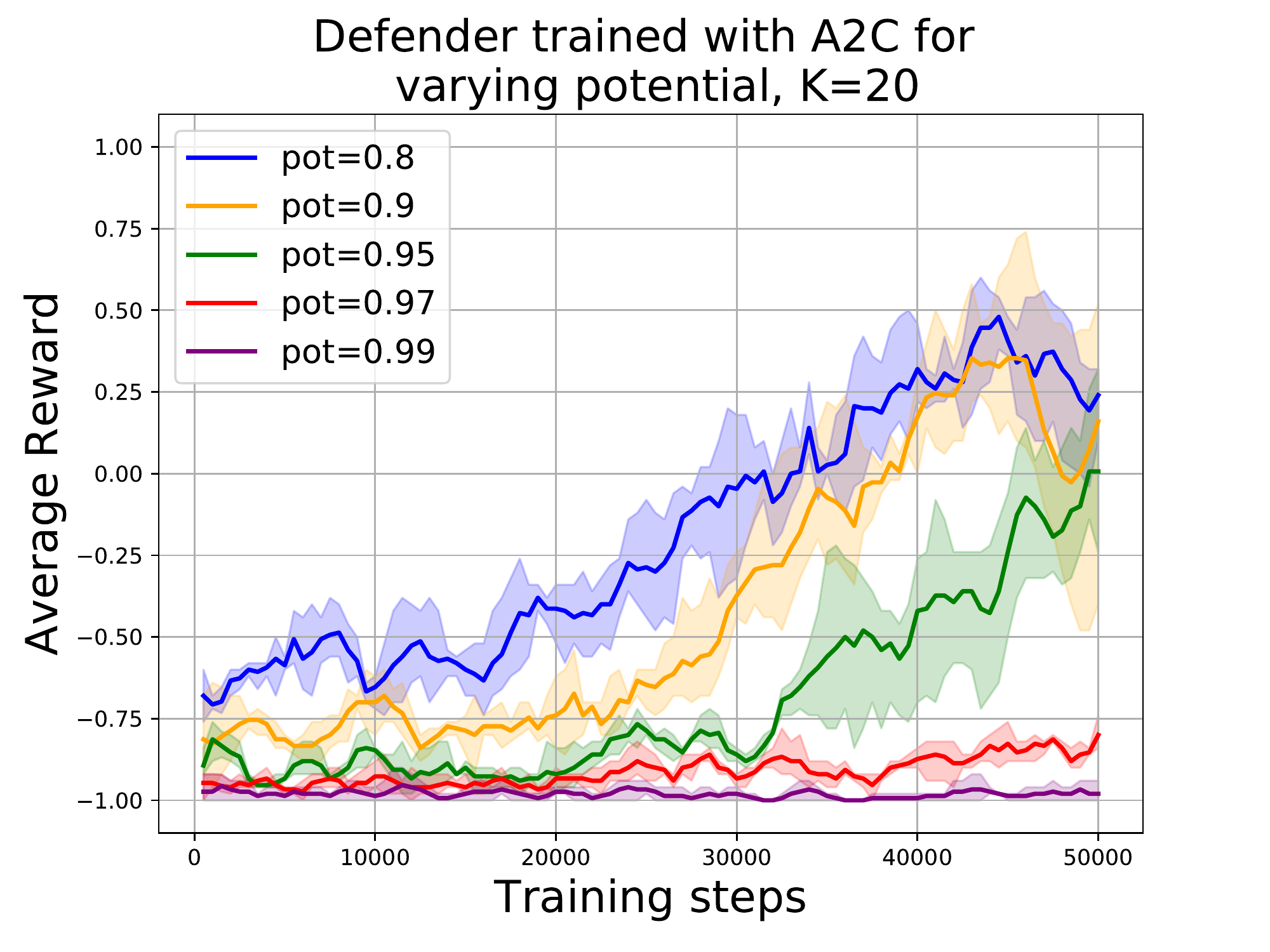}\hspace*{-9mm} 
  \\
  \includegraphics[width=0.5\columnwidth]{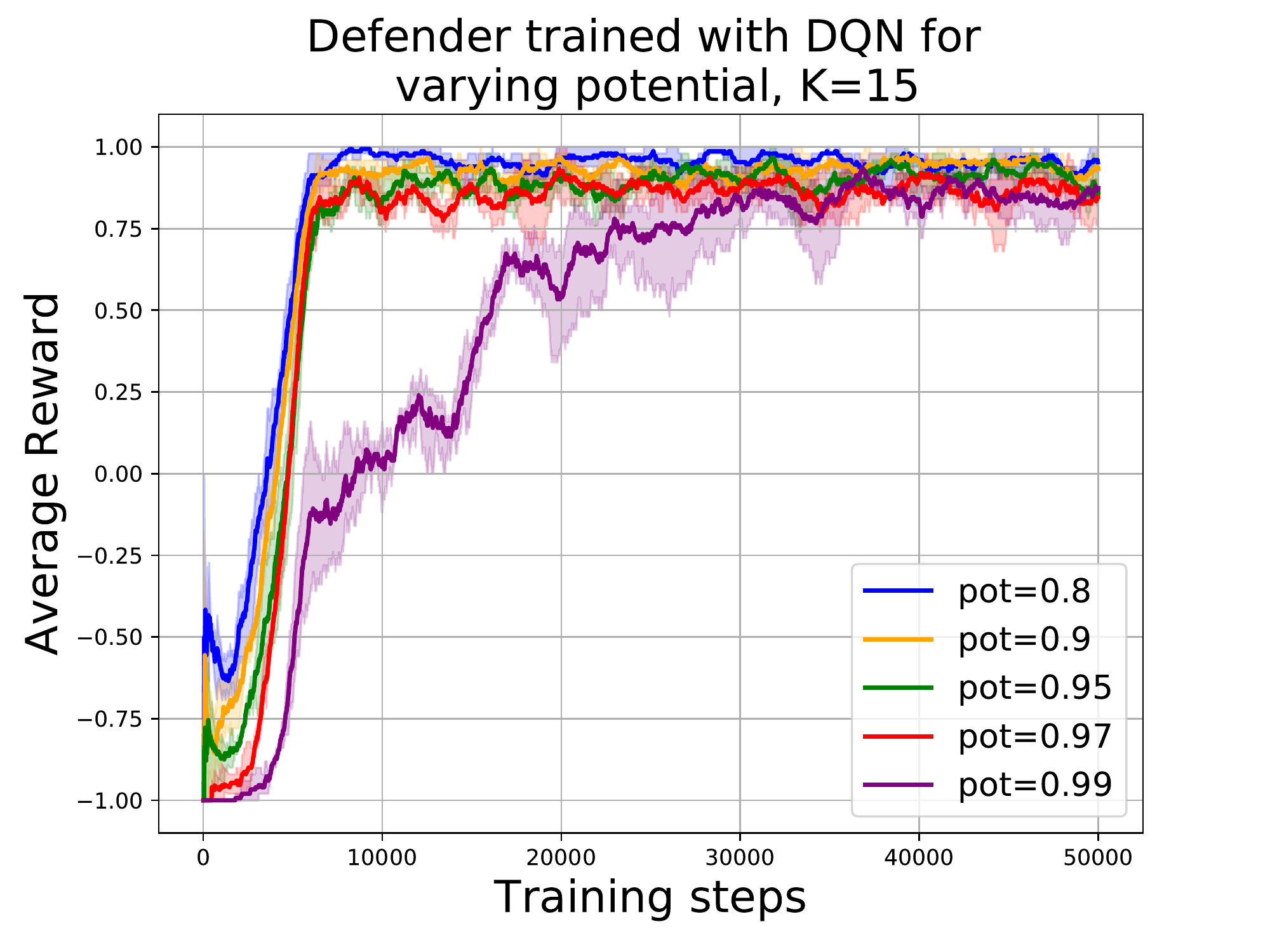}
  &
      \hspace*{-7mm}
  \includegraphics[width=0.5\columnwidth]{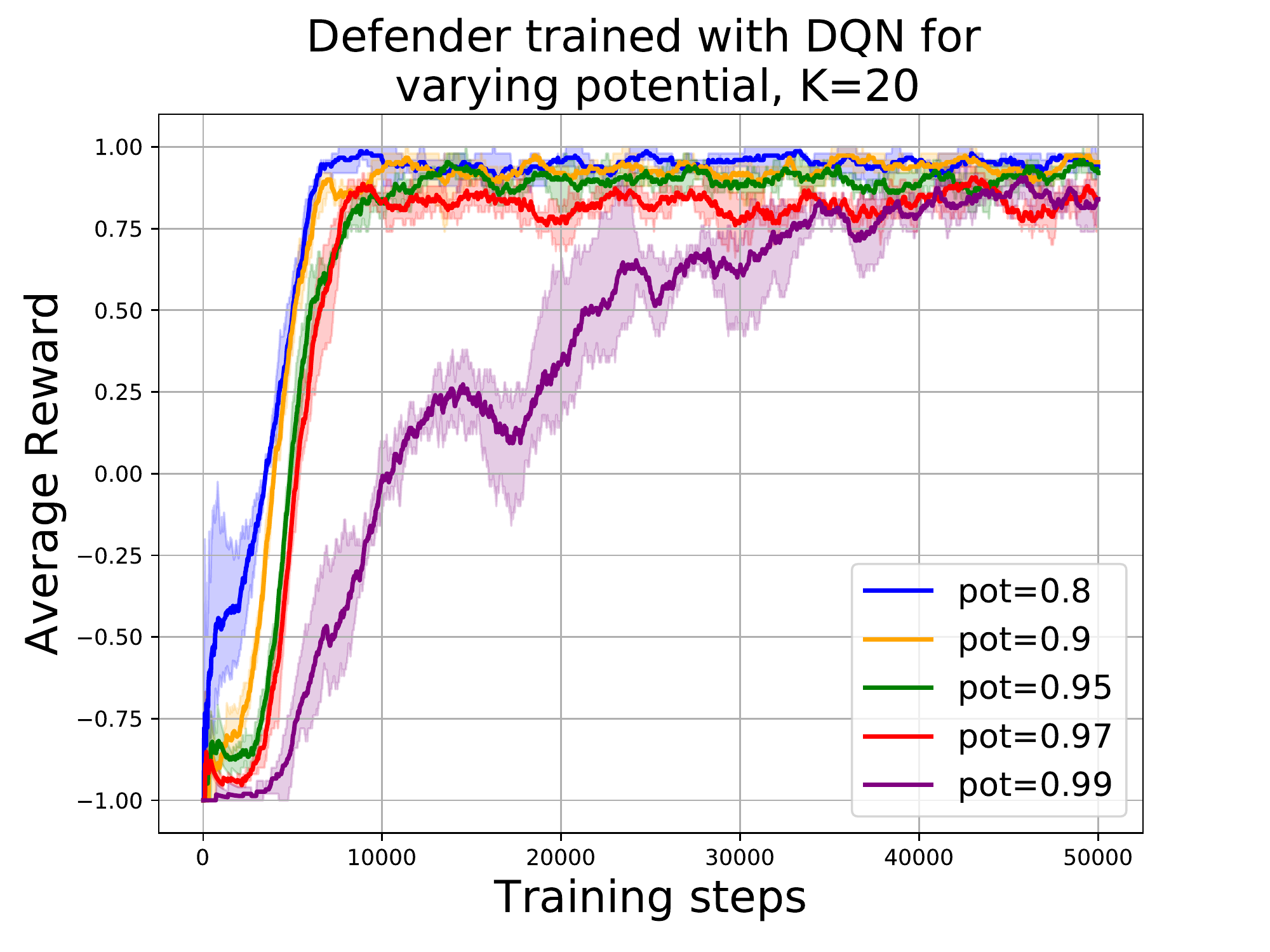}\\
  \end{tabular}
  \caption{\small Training defender agent with PPO, A2C and DQN for varying values of potentials and two different choices of $K$ with a deep network. Overall, we see significant improvements over using a linear model. DQN performs the most stably, while A2C tends to fare worse than both PPO and DQN.}
  \label{fig-defender-potential}
\end{figure}

When training the defender agent with RL, we have two choices of difficulty parameters. The potential of the start state, $\phi(S_0)$, changes how close to optimality the defender has to play, with values close to $1$ giving much less leeway for mistakes in valuing the two sets. Changing $K$, the number of levels, directly affects the sparsity of the reward, with higher $K$ resulting in longer games and less feedback. Additionally, $K$ also greatly increases the number of possible states and game trajectories (see Theorem \ref{thm-states-growth} in the Appendix).

\subsubsection{Evaluating Deep RL}
As the optimal policy can be expressed as a linear network, we first try training a linear model for the defender agent. We evaluate Proximal Policy Optimization (PPO) \citep{schulman2017ppo}, Advantage Actor Critic (A2C) \citep{mnih2016a2c}, and Deep Q-Networks (DQN) \citep{mnih2015dqn}, using the OpenAI Baselines implementations \citep{baselines}. Both PPO and A2C find it challenging to learn the harder difficulty settings of the game, and perform better with deeper networks  (Figure \ref{fig-defender-linear}). DQN performs surprisingly well, but we see some improvement in performance variance with a deeper model. In summary, while the policy can theoretically be expressed with a linear model, empirically we see gains in performance and a reduction in variance when using deeper networks (c.f. Figures \ref{fig-defender-potential}, \ref{fig-defender-K}.)

Having evaluated the performance of linear models, we turn to using deeper neural networks for our policy net.  (A discussion of the hyperparameters used is provided in the appendix.) Identically to above, we evaluate PPO, A2C and DQN on varying start state potentials and $K$. Each algorithm is run with 3 random seeds, and in all plots we show minimum, mean, and maximum performance. Results are shown in Figures \ref{fig-defender-potential}, \ref{fig-defender-K}. Note that all algorithms show variation in performance across different settings of potentials and $K$, and show noticeable drops in performance with harder difficulty settings. When varying potential in Figure \ref{fig-defender-potential}, both PPO and A2C show larger variation than DQN. A2C shows the greatest variation and worst performance out of all three methods.

\begin{figure}
\centering
  \begin{tabular}{cc}
  \vspace*{-5mm}
   \hspace*{-20mm}\includegraphics[width=0.5\columnwidth]{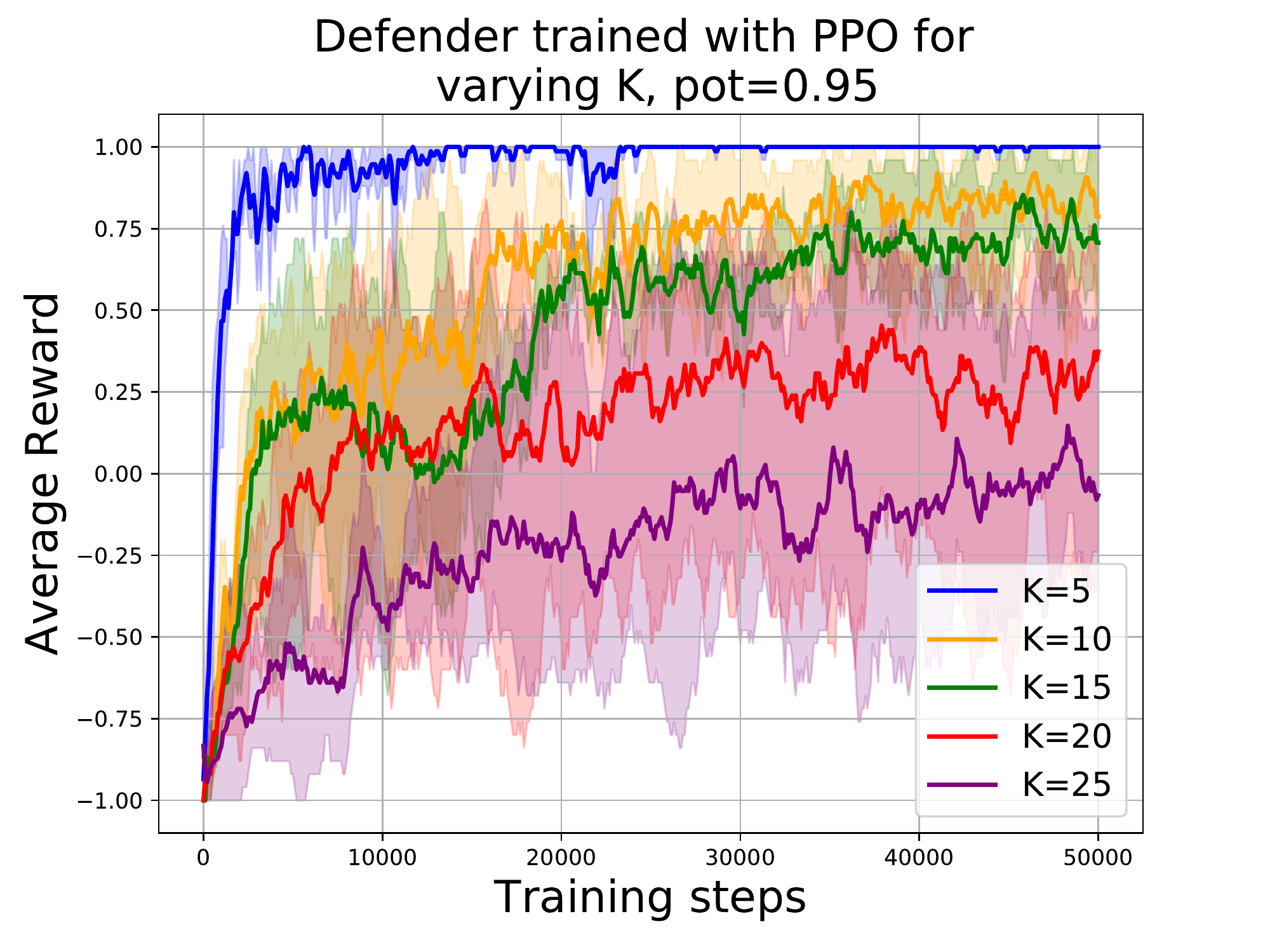}\hspace*{-20mm}
  &
  \hspace*{-15mm}\includegraphics[width=0.5\columnwidth]{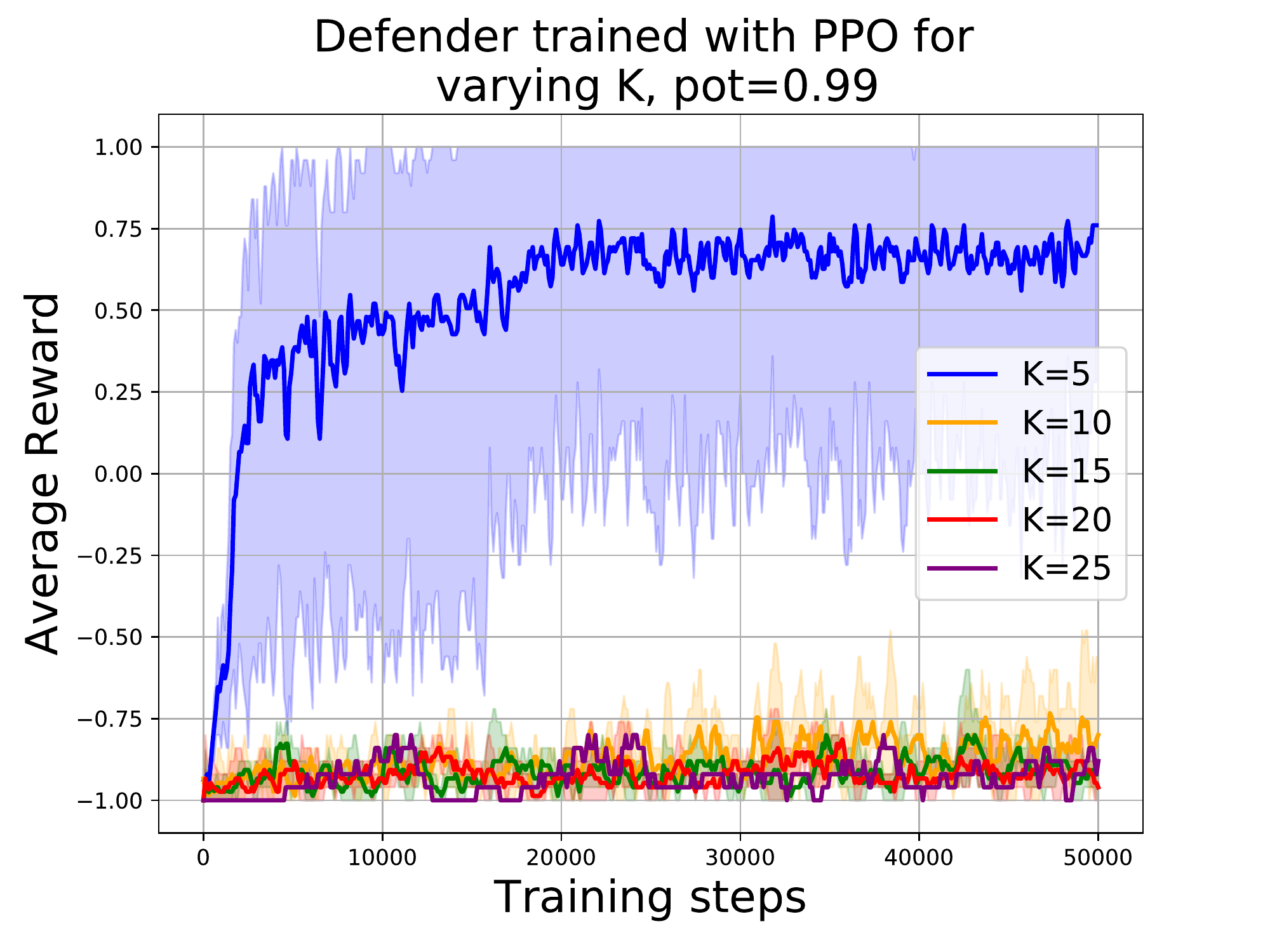}\hspace*{-9mm} 
  \\
  \hspace*{-20mm}
  \includegraphics[width=0.5\columnwidth]{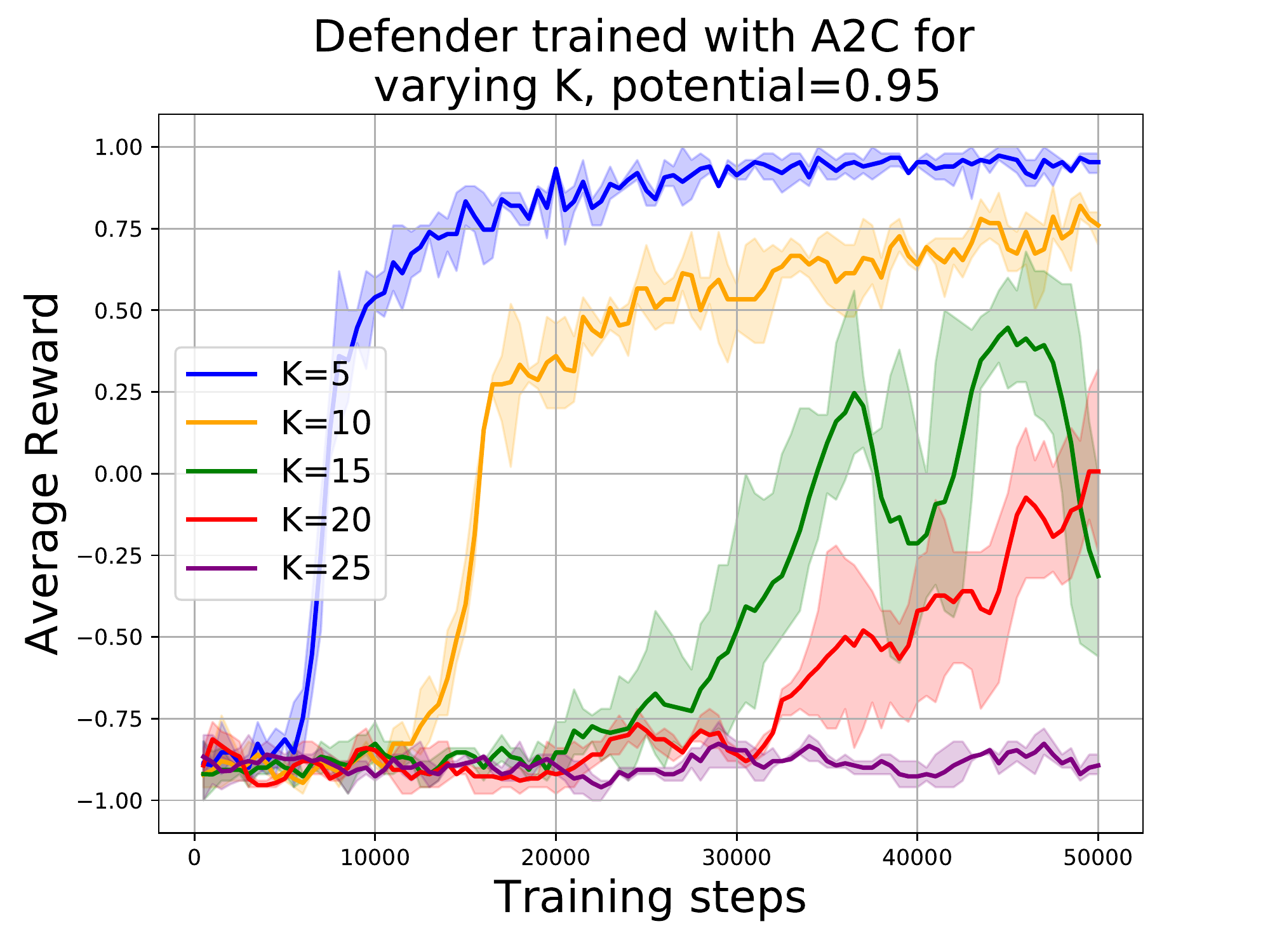}\hspace*{-20mm}
  &
    \hspace*{-15mm}
  \includegraphics[width=0.5\columnwidth]{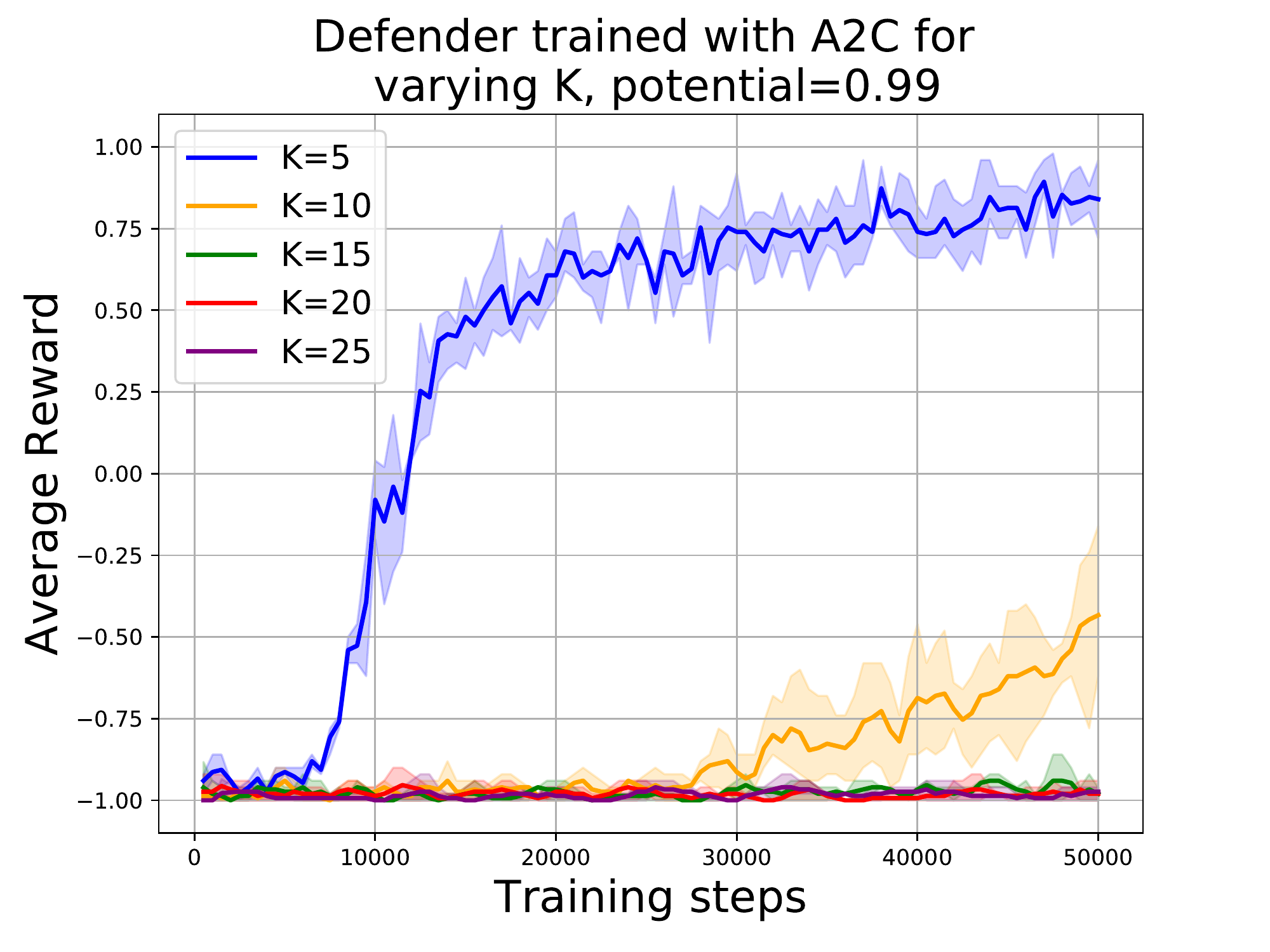}\hspace*{-9mm} 
  \\
  \includegraphics[width=0.5\columnwidth]{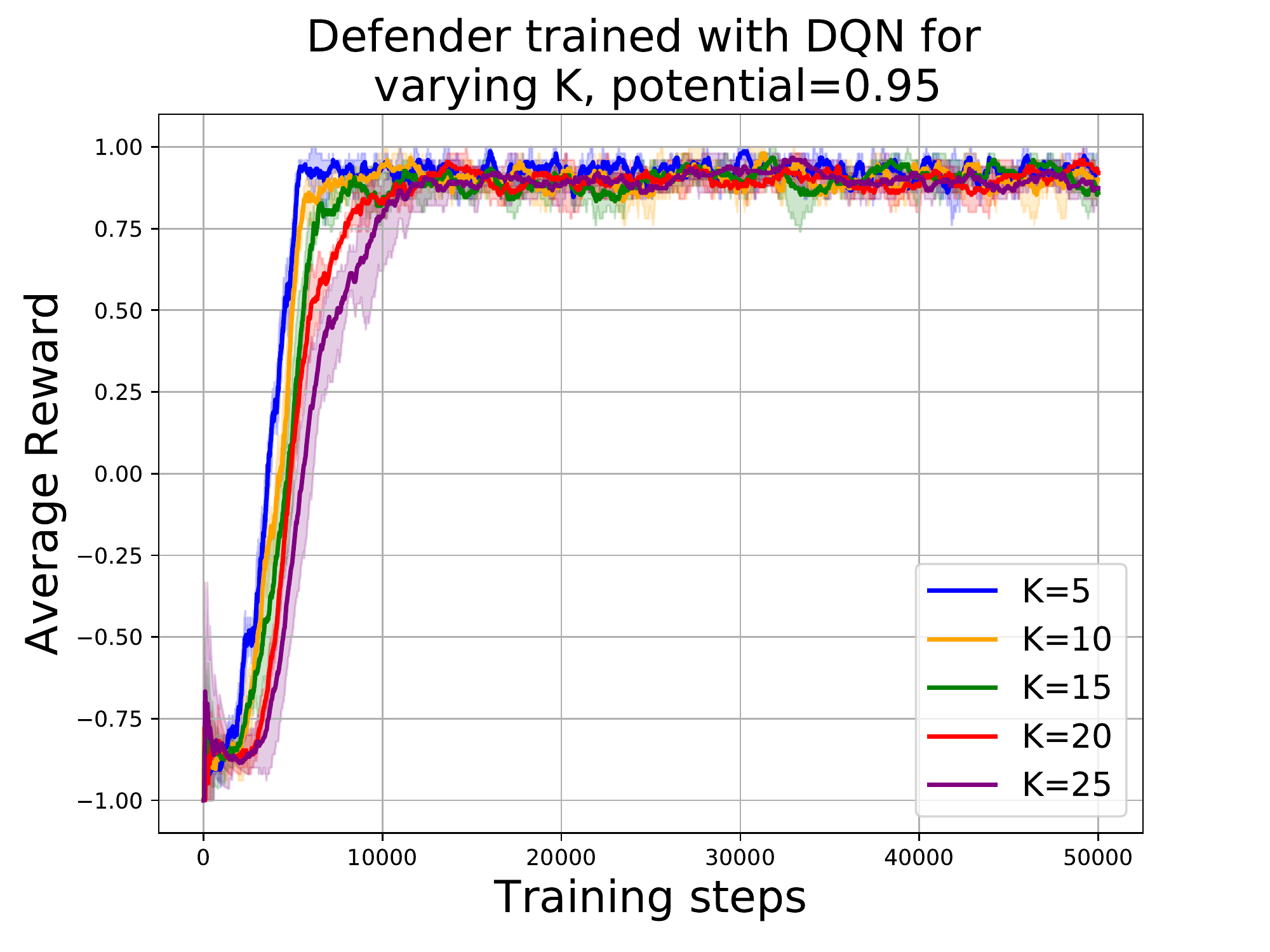}
  &
      \hspace*{-7mm}
  \includegraphics[width=0.5\columnwidth]{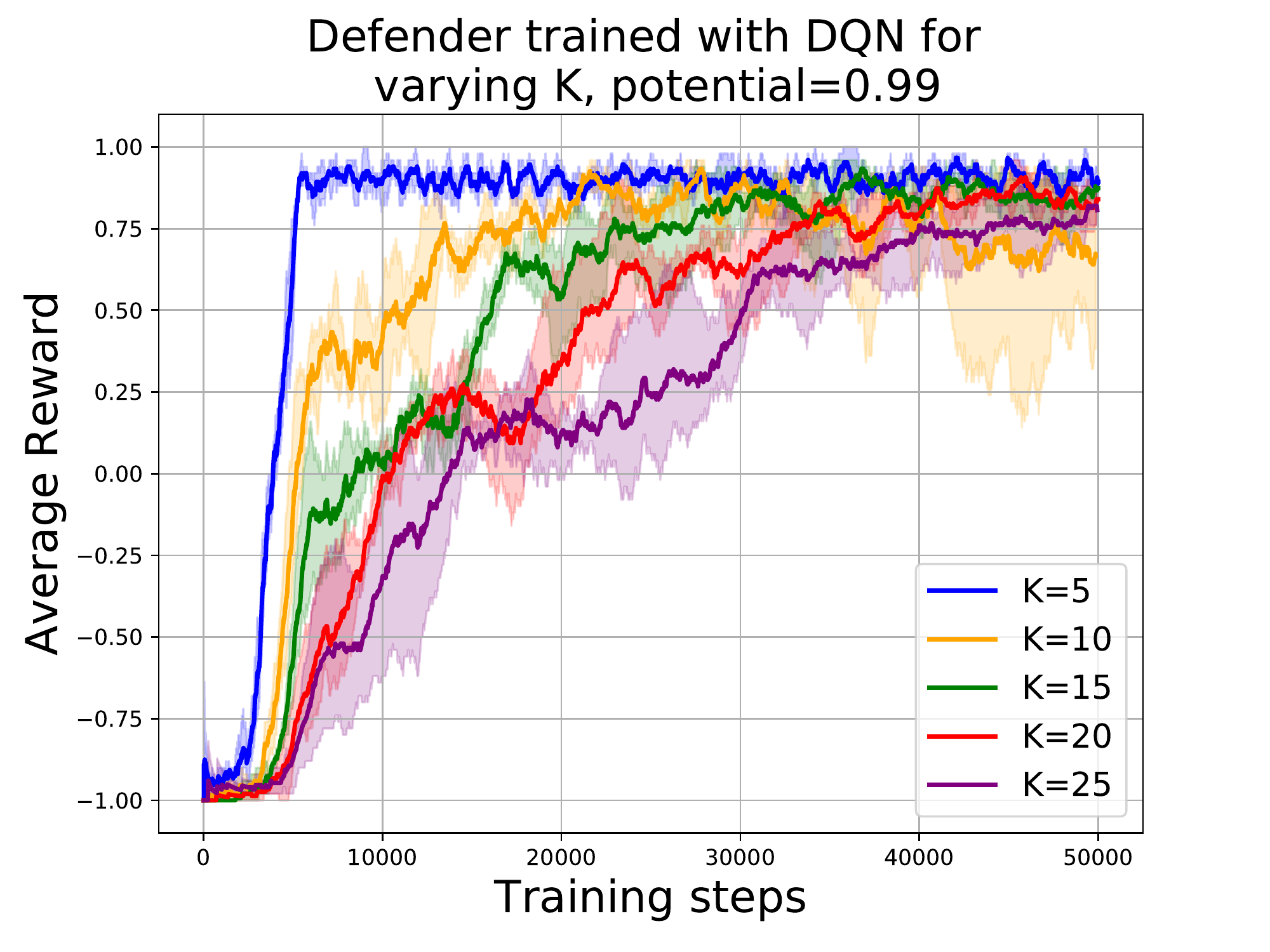}\\
  \end{tabular}
  \caption{\small Training defender agent with PPO, A2C and DQN for varying values of $K$ and two different choices of potential (left and right column) with a deep network. All three algorithms show a noticeable variation in performance over different difficulty settings. Again, we notice that DQN performs the best, with PPO doing reasonably for lower potential, but not for higher potentials. A2C tends to fare worse than both PPO and DQN.}
  \label{fig-defender-K}
\end{figure}


\begin{figure}
\centering
\begin{tabular}{cc}
   \hspace*{-6mm} \includegraphics[width=0.55\columnwidth]{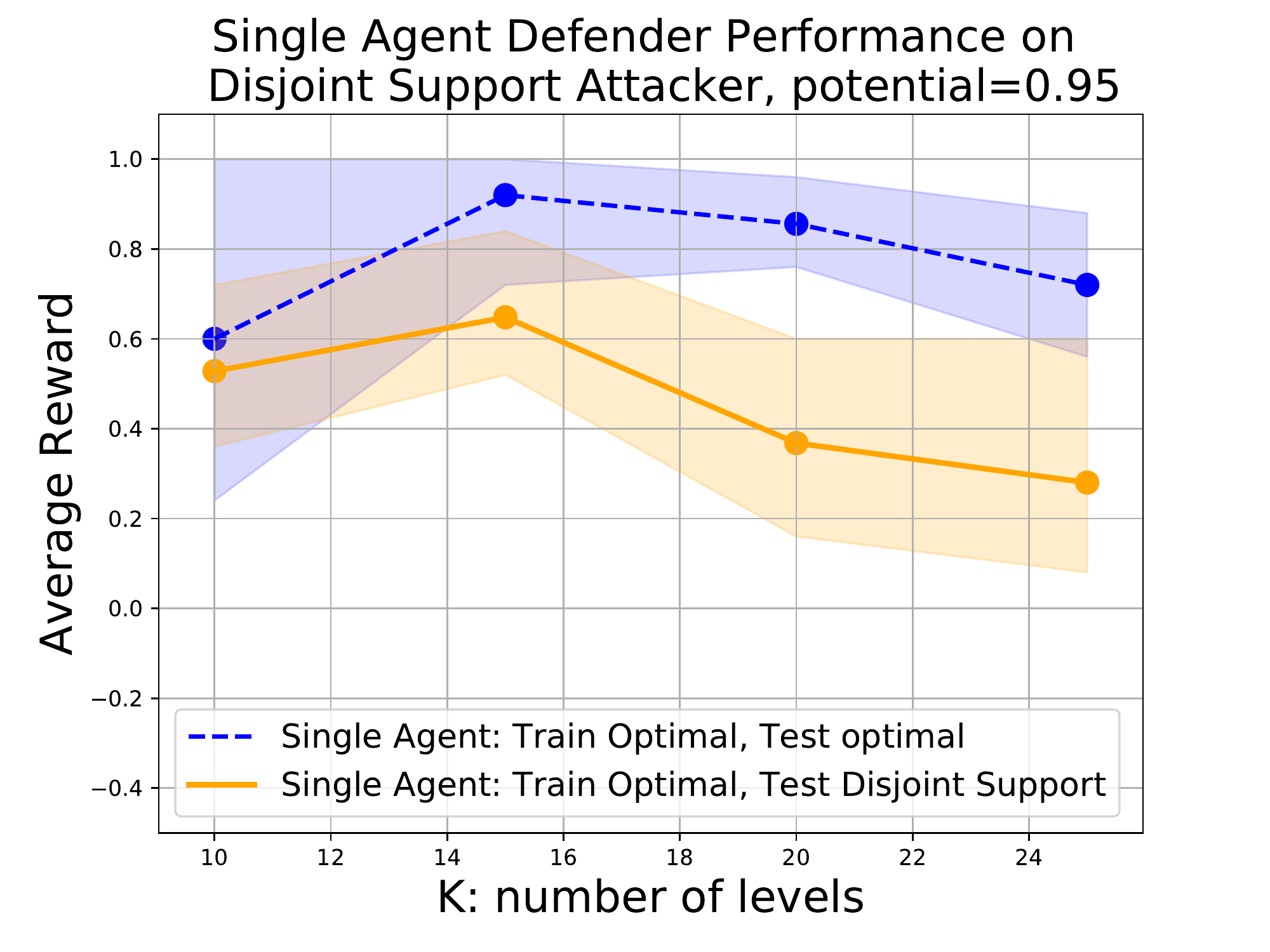}
   \hspace*{-6mm}
   &
   \includegraphics[width=0.55\columnwidth]{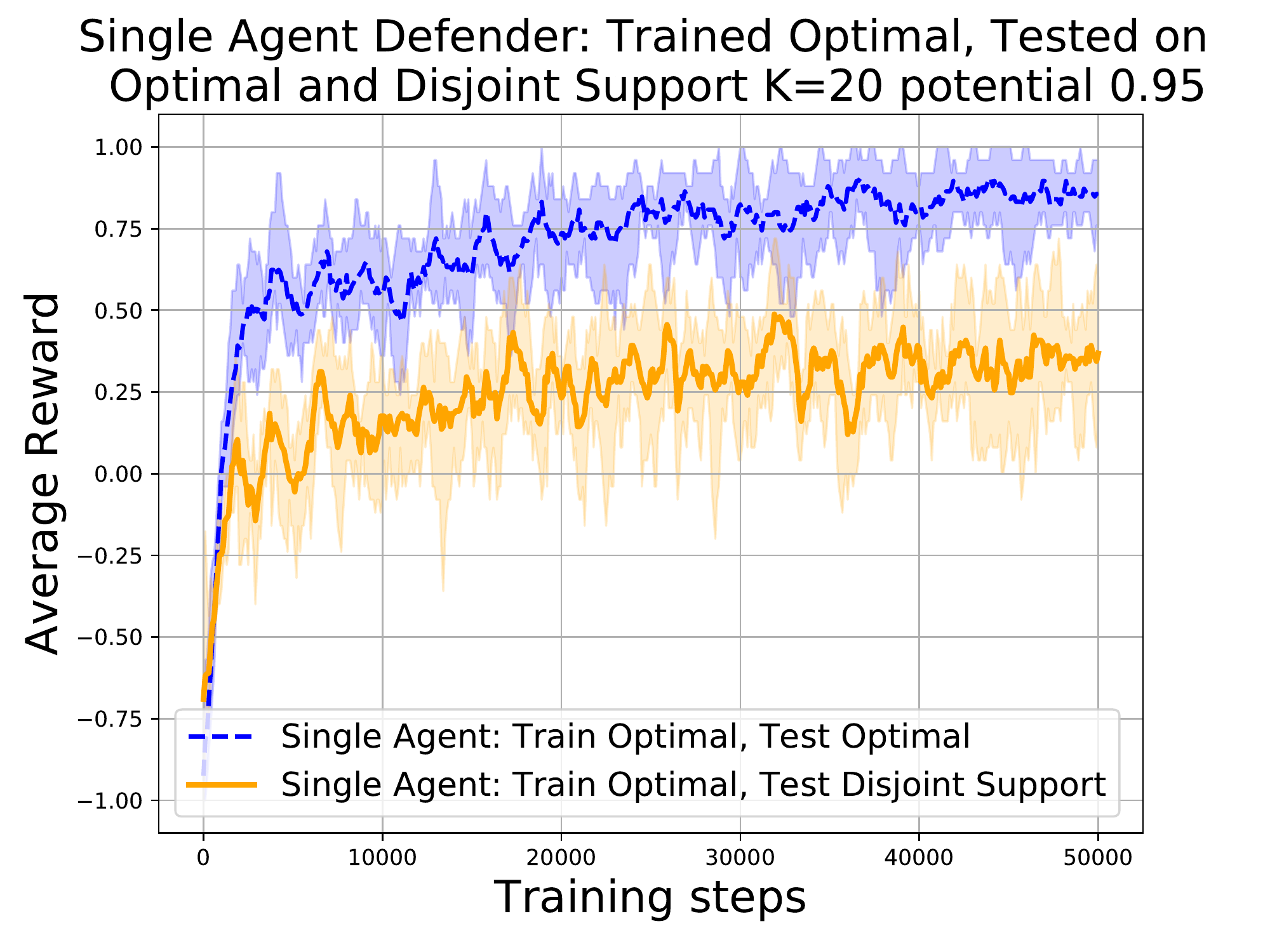}
  \end{tabular}
  \caption{\small Plot showing overfitting to opponent strategies. A defender agent is trained on the optimal attacker, and then tested on (a) another optimal attacker environment (b) the disjoint support attacker environment. The left pane shows the resulting performance drop when switching to testing on the same opponent strategy as in training to a different opponent strategy. The right pane shows the result of testing on an optimal attacker vs a disjoint support attacker during training. We see that performance on the disjoint support attacker converges to a significantly lower level than the optimal attacker.}
  \label{fig-single-agent-overfit}
\end{figure}

\section{Generalization in RL and Multiagent Learning}
\label{sec-generalization}
In the previous section we trained defender agents using popular RL algorithms, and then evaluated the performance of the trained agents on the environment. However, noting that the Attacker-Defender game has a known optimal policy that works perfectly in any game setting, we can evaluate our RL algorithms on \textit{generalization}, not just performance. 

We take a setting of parameters, start state potential $0.95$ and $K = 5, 10, 15, 20, 25$ where we have seen the RL agent perform well, and change the procedural attacker policy between train and test. In detail, we train RL agents for the defender against an attacker playing optimally, and test these agents for the defender on the disjoint support attacker. The results are shown in Figure \ref{fig-single-agent-overfit} where we can see that while performance is high, the RL agents fail to generalize against an opponent that is theoretically easier.  This leads to the natural question of how we might achieve stronger generalization in our environment.

\begin{figure}
  \centering
  \begin{tabular}{cc}
  \hspace*{-2mm} \includegraphics[width=0.5\columnwidth]{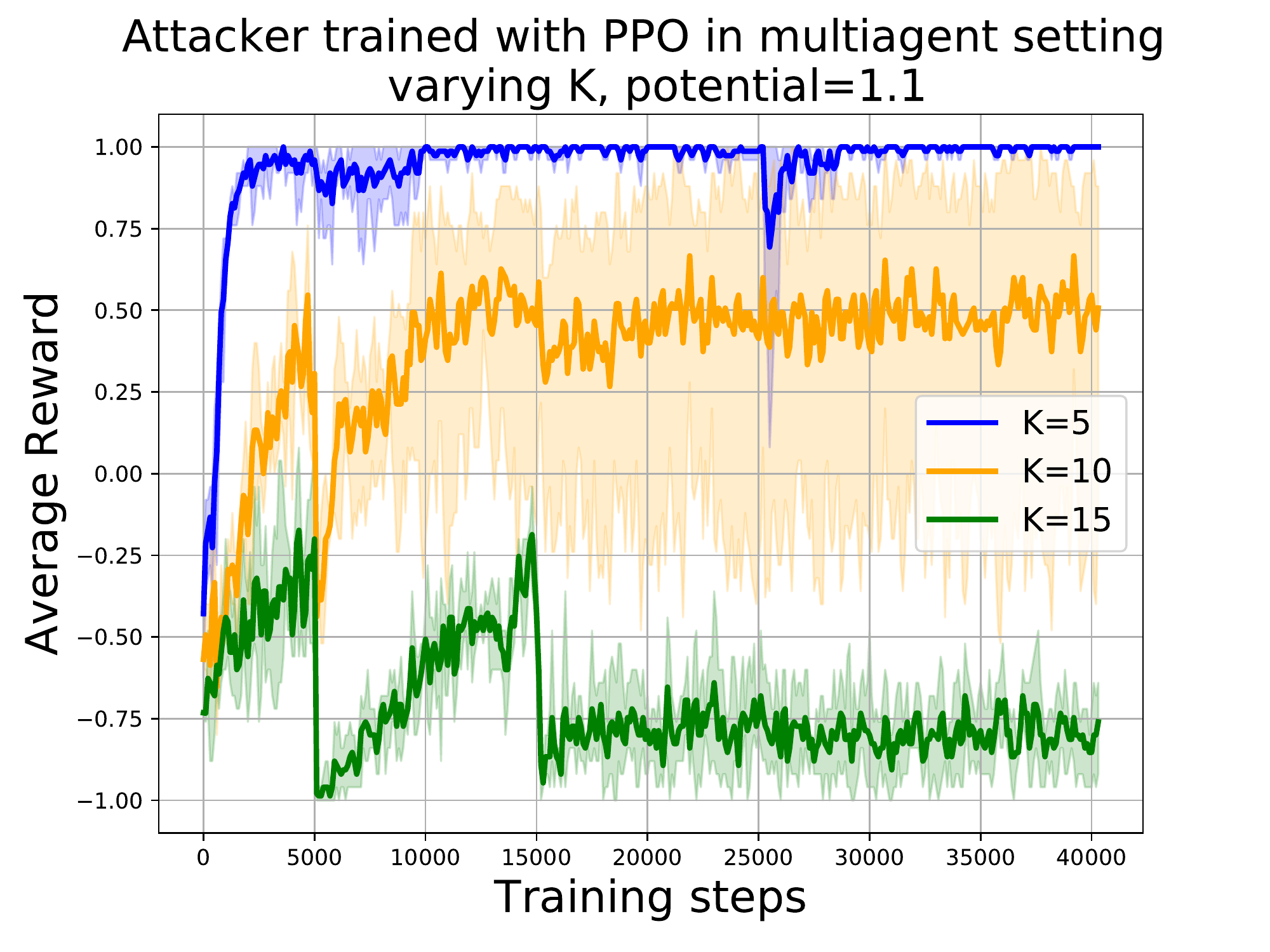} \hspace*{-5mm}
  &
  \includegraphics[width=0.5\columnwidth]{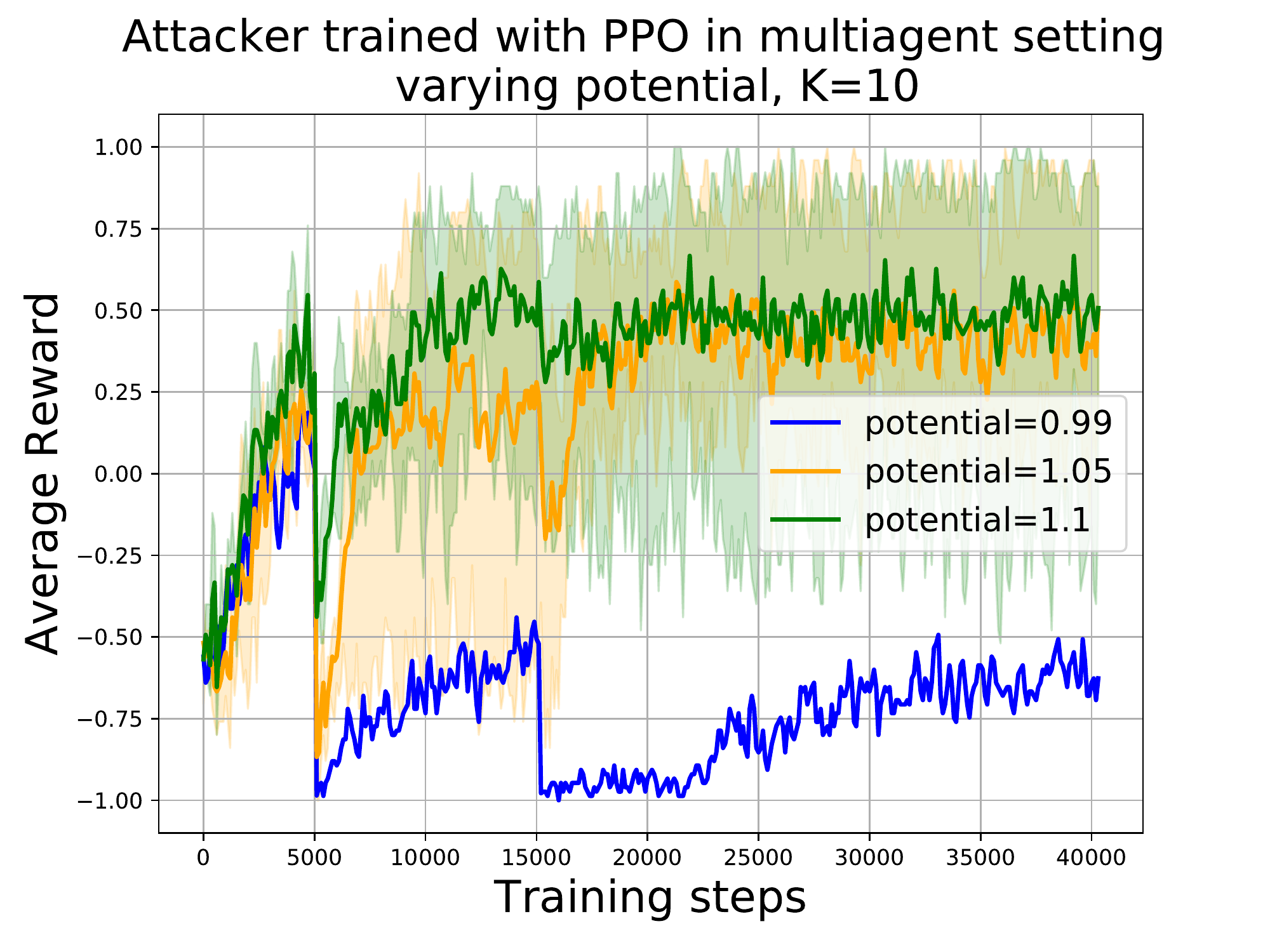}\\
  \hspace*{-2mm} \includegraphics[width=0.5\columnwidth]{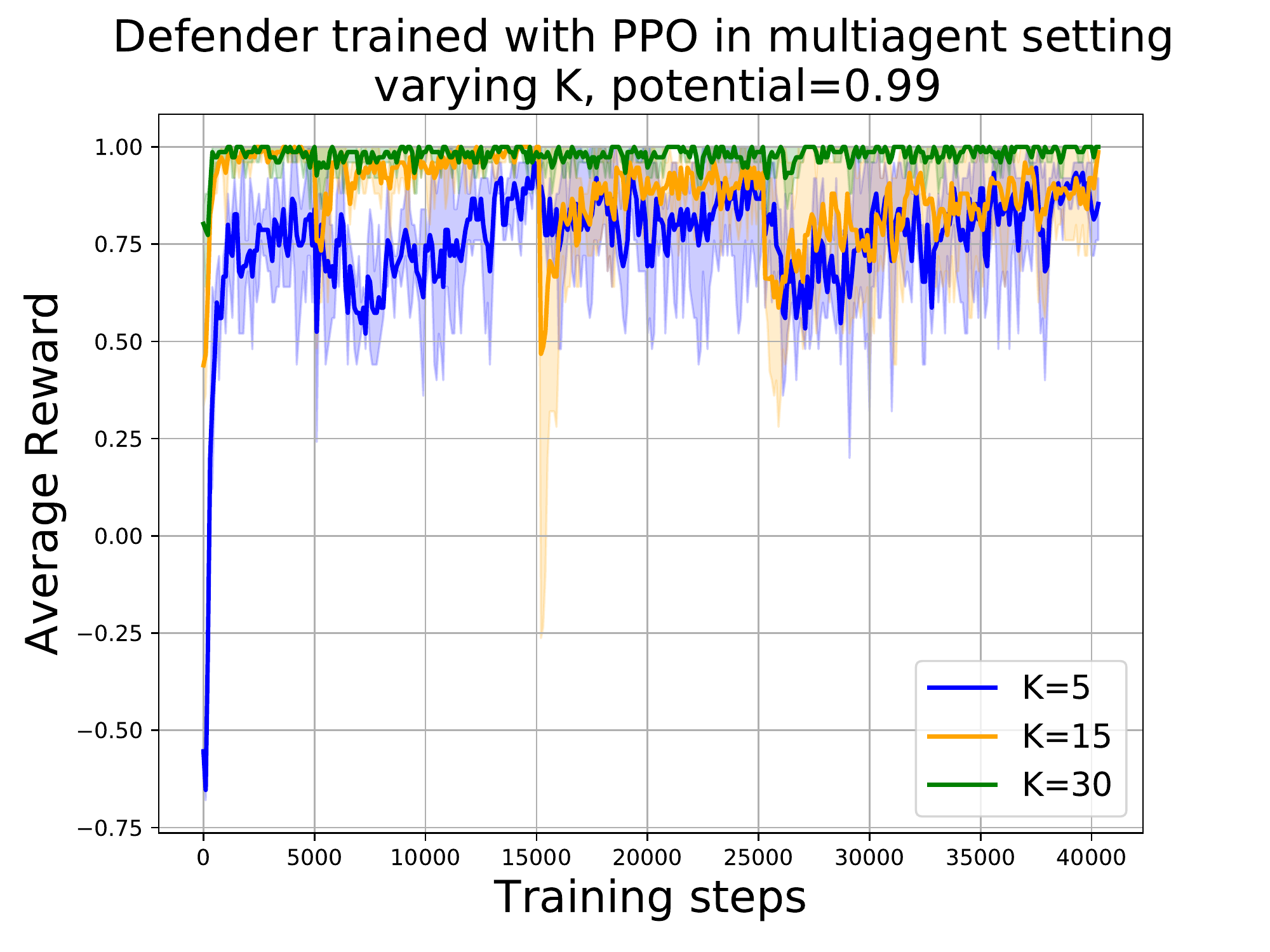} 
  \hspace*{-5mm}
  &
  \includegraphics[width=0.5\columnwidth]{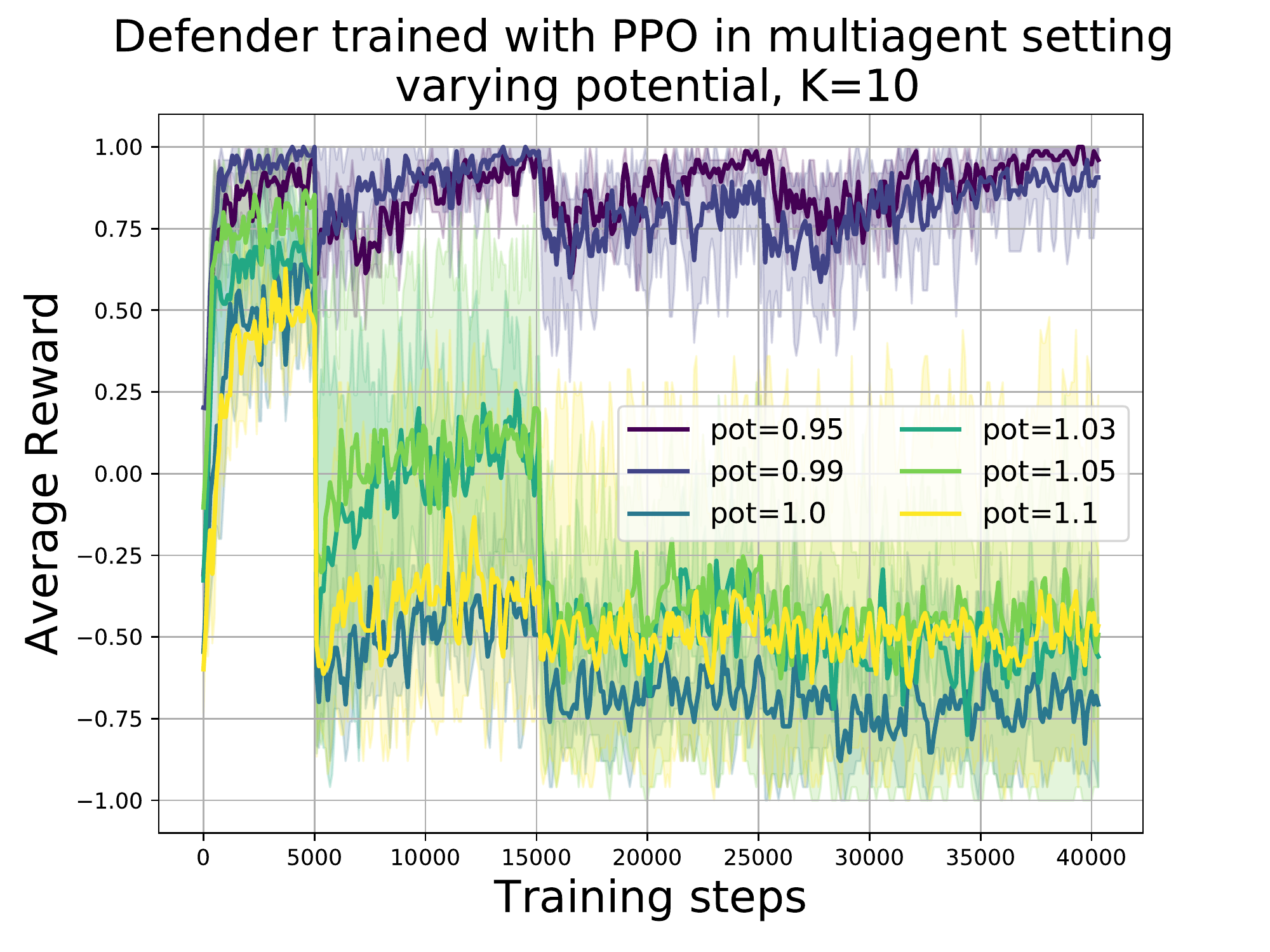} \\
  \end{tabular}
  \caption{\small Performance of attacker and defender agents when learning in a multiagent setting. In the top panes, solid lines denote attacker performance. In the bottom panes, solid lines are defender performance. The sharp changes in performance correspond
  to the times we switch which agent is training.
  We note that the defender performs much better in the multiagent setting.
  Furthermore, the attacker loses to the defender for potential $1.1$ at $K=15$, despite winning against the optimal defender in Figure \ref{fig-attacker-K} in the Appendix. We also see (right panes) that the attacker has higher variance and sharper changes in its performance even under conditions when it is guaranteed to win.
  }
  \vspace*{-1.4em}
  \label{fig-ppo-multiagent}
\end{figure}

\subsection{Training an Attacker Agent}

One way to mitigate this overfitting issue is to set up a method of also training the attacker, so that we can train the defender against a learned attacker, or in a multiagent setting. However, determining the correct setup to train the attacker agent first requires devising a tractable parametrization of the action space.  A naive implementation of the attacker would be to have the policy output how many pieces should be allocated to $A$ for each of the $K+1$ levels (as in the construction from \citet{spencer1994game}). This can grow very rapidly in $K$, which is clearly impractical. 
To address this, we first prove a new theorem that enables us to parametrize an optimal attacker with a much smaller action space.
\begin{theorem}
\label{thm-prefix-attacker}
For any Attacker-Defender game with $K$ levels, start state $S_0$ and $\phi(S_0) \geq 1$, there exists a partition $A, B$ such that $\phi(A) \ge 0.5$, $\phi(B) \ge 0.5$, and for some $l$, $A$ contains pieces of level $i > l$, and $B$ contains all pieces of level $i < l$.
\end{theorem}
\begin{proof}
For each $l \in \{0, 1, \ldots, K \}$, let $A_l$ be the set of all pieces from levels $K$ down to and excluding level $l$, with $A_K = \emptyset$. We have $\phi(A_{i+1}) \le \phi(A_i)$, $\phi(A_K) = 0$ and $\phi(A_0) = \phi(S_0) \geq 1$. Thus, there exists an $l$ such that $\phi(A_{l}) < 0.5$ and $\phi(A_{l-1}) \geq 0.5$. If $\phi(A_{l-1}) = 0.5$, we set $A_{l-1} = A$ and $B$ the complement, and are done. So assume $\phi(A_{l}) < 0.5$ and $\phi(A_{l-1}) > 0.5$

Since $A_{l-1}$ only contains pieces from levels $K$ to $l$, potentials $\phi(A_{l})$ and $\phi(A_{l-1})$ are both integer multiples of $2^{-(K-l)}$, the value of a piece in level $l$. Letting $\phi(A_{l}) = n \cdot 2^{-(K-l)}$ and $\phi(A_{l-1}) = m \cdot 2^{-(K-l)}$, we are guaranteed that level $l$ has $m - n$ pieces, and that we can move $k < m -n$ pieces from $A_{l-1}$ to $A_{l}$ such that the potential of the new set equals $0.5$.
\end{proof}

This theorem gives a different way of constructing and parametrizing an optimal attacker. The attacker outputs a level $l$. The environment assigns all pieces before level $l$ to $A$, all pieces after level $l$ to $B$, and splits level $l$ among $A$ and $B$ to keep the potentials of $A$ and $B$ as close as possible. Theorem \ref{thm-prefix-attacker} guarantees the optimal policy is representable, and the action space is linear instead of exponential in $K$.

With this setup, we train an attacker agent against the optimal defender with PPO, A2C, and DQN. The DQN results were very poor, and so we show results for just PPO and A2C. In both algorithms we found there was a large variation in performance when changing $K$, which now affects both reward sparsity and action space size. We observe less outright performance variability with changes in potential for small $K$ but see an increase in the variance (\autoref{fig-attacker-K} in Appendix). 

\subsection{Learning through Multiagent Play}

With this attacker training, we can now look at learning in a multiagent setting. We first explore the effects of varying the potential and $K$ as shown in Figure \ref{fig-ppo-multiagent}. Overall, we find that the attacker fares worse in multiagent play than in the single agent setting. In particular, note that in the top left pane of Figure \ref{fig-ppo-multiagent}, we see that the attacker loses to the defender even with $\phi(S_0) = 1.1$ for $K=15$. We can compare this to Figure \ref{fig-attacker-K} in the Appendix where with PPO, we see that with $K=15$, and potential $1.1$, the single agent attacker succeeds in winning against the optimal defender.

\subsection{Single Agent and Multiagent Generalization Across Opponent Strategies}
\begin{figure}
  \centering
  \includegraphics[scale=0.4]{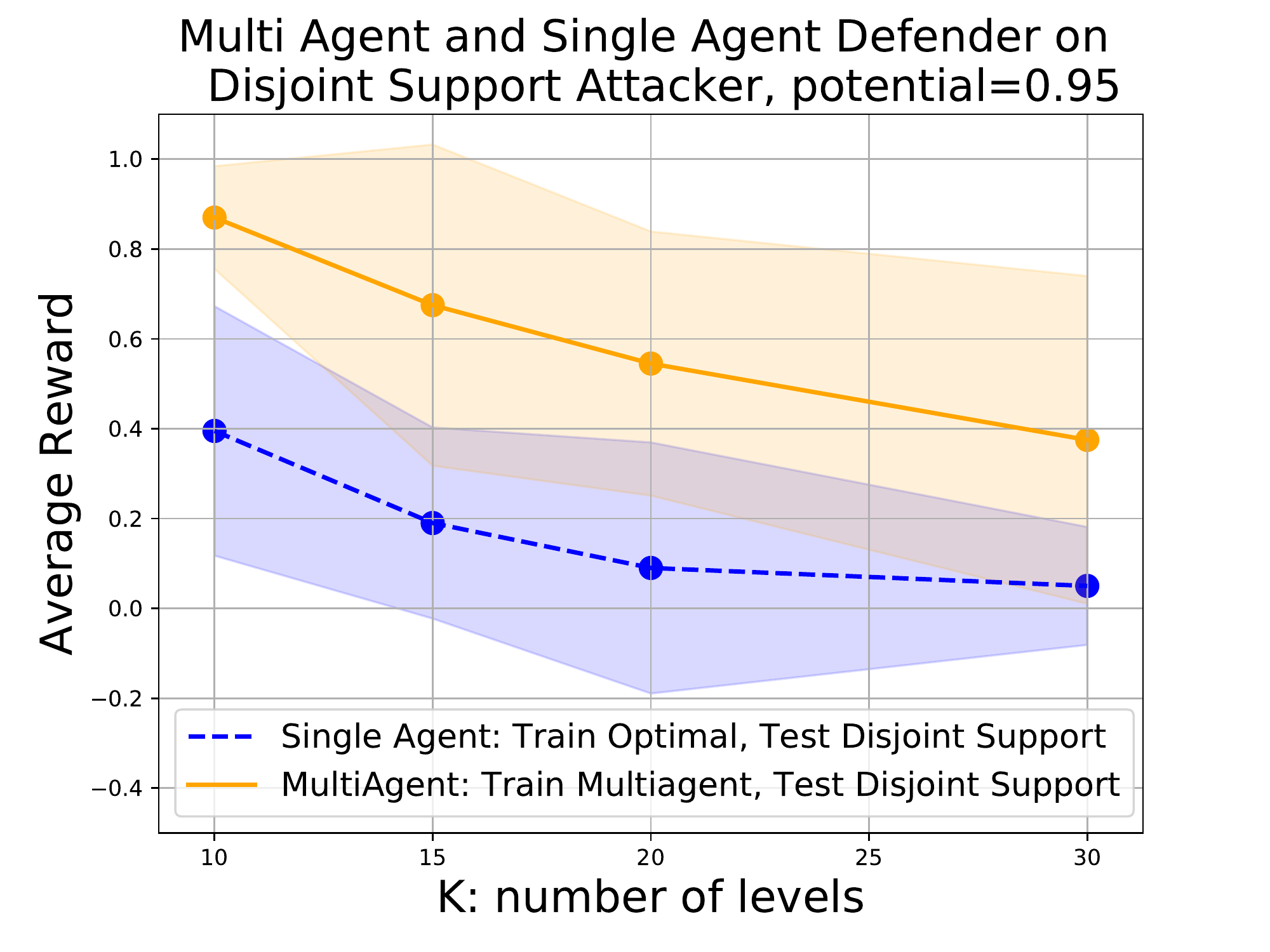}
  \caption{\small Results for generalizing to different attacker strategies with single agent defender and multiagent defender. The figure single agent defender trained on the optimal attacker and then tested on the disjoint support attacker and a multiagent defender also tested on the disjoint support attacker for different values of $K$. We see that multiagent defender generalizes better to this unseen strategy than the single agent defender.}
  \label{fig-multi-agent-generalize}
\end{figure}

Finally, we return again to our defender agent, and test generalization between the single and multiagent settings. We train a defender agent in the single agent setting against the optimal attacker, and test on a an attacker that only uses the Disjoint Support strategy. We also test a defender trained in the multiagent setting (which has never seen any hardcoded strategy of this form) on the Disjoint Support attacker. The results are shown in Figure \ref{fig-multi-agent-generalize}. We find that the defender trained as part of a multiagent setting generalizes noticeably better than the single agent defender. We show the results over $8$ random seeds and plot the mean (solid line) and shade in the standard deviation.

\section{Training with Self Play}
In the previous section, we showed that with a new theoretical insight into a more efficient attacker action space paramterization (Theorem \ref{thm-prefix-attacker}), it is possible to train an attacker agent. The attacker agent was parametrized by a neural network different from the one implementing the defender, and it was trained in a multiagent fashion. In this section, we present additional theoretical insights that enables training by \textit{self-play}: using a single neural network to parametrize both the attacker and defender. 

The key insight is the following: both the defender's optimal strategy and the construction of the optimal attacker in Theorem \ref{thm-prefix-attacker} depend on a primitive operation that takes a partition of the pieces into sets $A, B$ and determines which of $A$ or $B$ is ``larger'' (in the sense that it has higher potential).  For the defender, this leads directly to a strategy that destroys the set of higher potential.  For the attacker, this primitive can be used in a binary search procedure to find the desired partition in Theorem \ref{thm-prefix-attacker}: given an initial partition $A, B$ into a prefix and a suffix of the pieces sorted by level, we determine which set has higher potential, and then recursively find a more balanced split point inside the larger of the two sets. This process is summarized in Algorithm \ref{alg_self_play}.


\begin{algorithm}[tb]
   \caption{Self Play with Binary Search}
   \label{alg_self_play}
\begin{algorithmic}
   \STATE initialize game
   \REPEAT
   \STATE Partition game pieces at center into $A, B$
   \REPEAT
    \STATE Input partition $A, B$ into neural network
    \STATE Output of neural network determines next binary search split
    \STATE Create new partition $A, B$ from this split
   \UNTIL{\textit{Binary Search Converges}}
   \STATE Input final partition $A, B$ to network
   \STATE Destroy larger set according to network output
   \UNTIL{\textit{Game Over:} use reward to update network parameters with RL algorithm}
\end{algorithmic}
\end{algorithm}

Thus, by training a single neural network designed to implement this primitive operation --- determining which of two sets has higher potential --- we can simultaneously train both an attacker and a defender that invoke this neural network. We use DQN for this purpose because we found empirically that it is the quickest among our alternatives at converging to consistent estimates of relative potentials on sets. We train both agents in this way, and test the defender agent on the same attacker used in Figures \ref{fig-defender-linear}, \ref{fig-defender-potential} and \ref{fig-defender-K}. The results in Figure \ref{fig-self-play} show that a defender trained through self play significantly outperforms defenders trained against a procedural attacker. 

\begin{figure}
\centering
\begin{tabular}{cc}
   \hspace*{-5mm} \includegraphics[width=0.55\columnwidth]{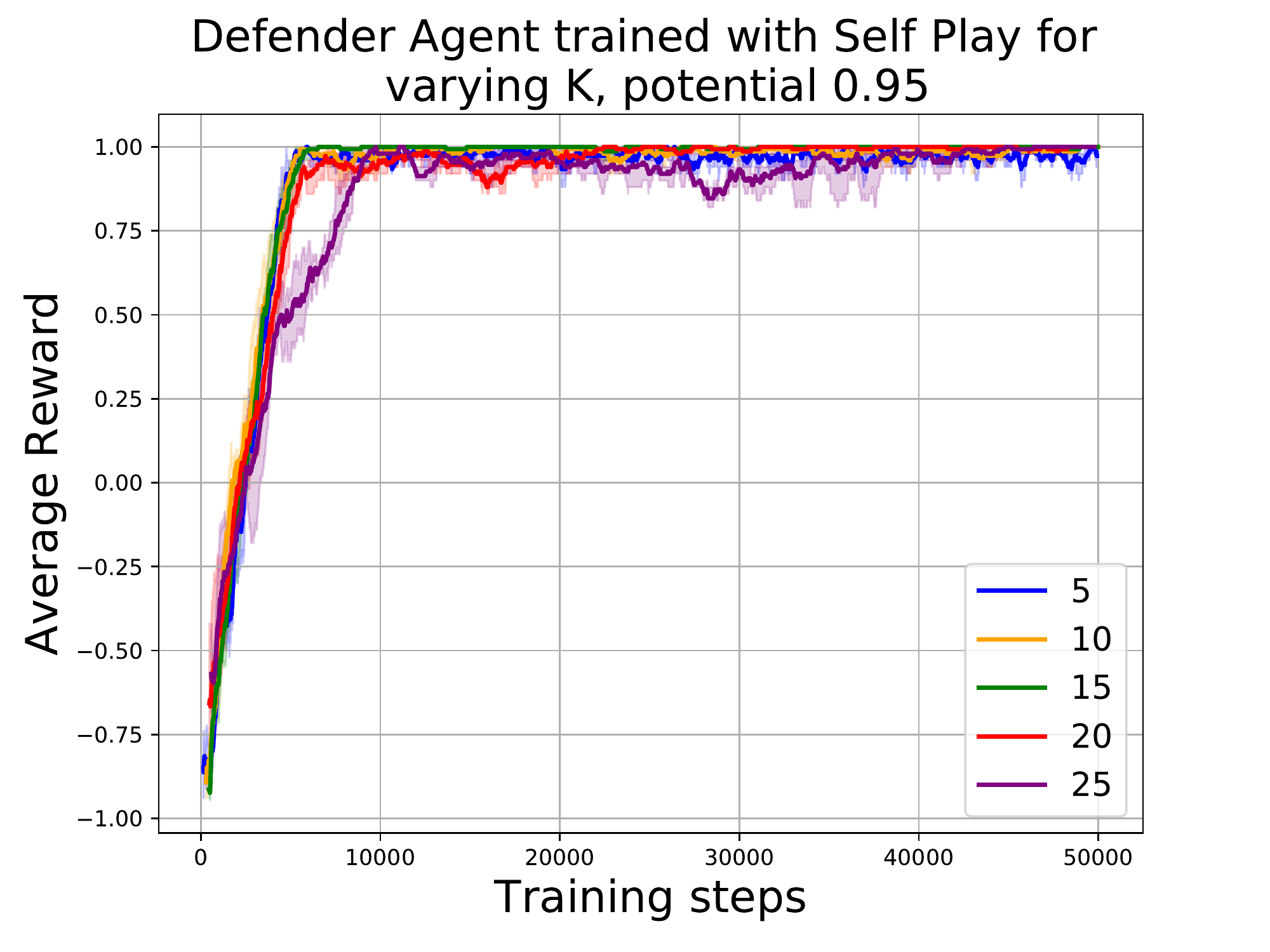}
   \hspace*{-6mm}
   &
   \includegraphics[width=0.55\columnwidth]{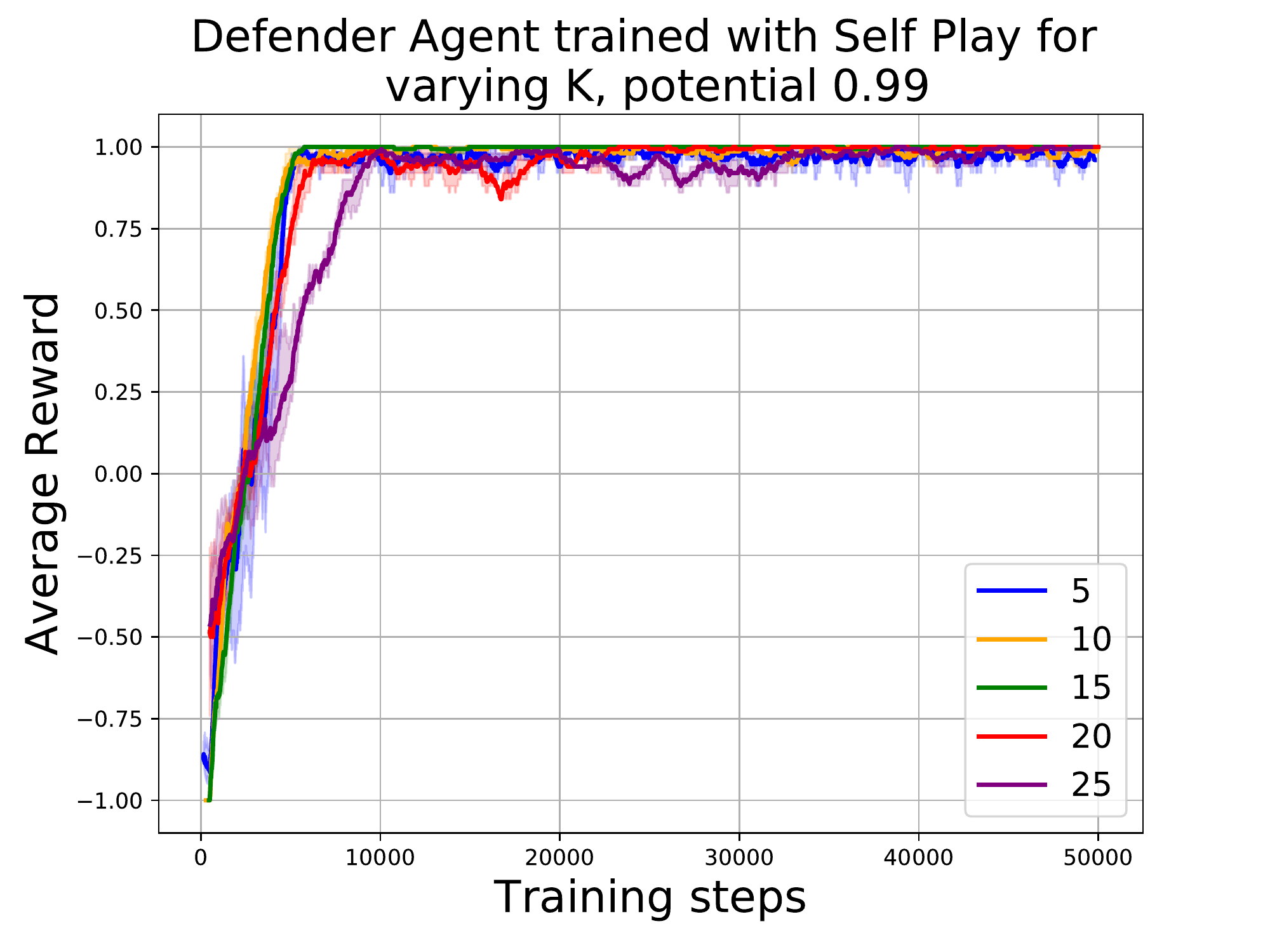}
  \end{tabular}
  \caption{\small We train an attacker and defender via self play using a DQN. The defender is implemented as in Figures \ref{fig-defender-K}, \ref{fig-defender-potential}, and the same neural network is used to train an attacker agent, performing binary search according to the Q values on partitions of the input space $A, B$. We then test the defender agent on the same procedural attacker used in Figures \ref{fig-defender-K}, \ref{fig-defender-potential}, and find that self play shows markedly improved performance.}
  \label{fig-self-play}
\end{figure}

\begin{figure}[t]
  \centering
  \begin{tabular}{cc}
   \hspace*{-5mm} \includegraphics[width=0.55\columnwidth]{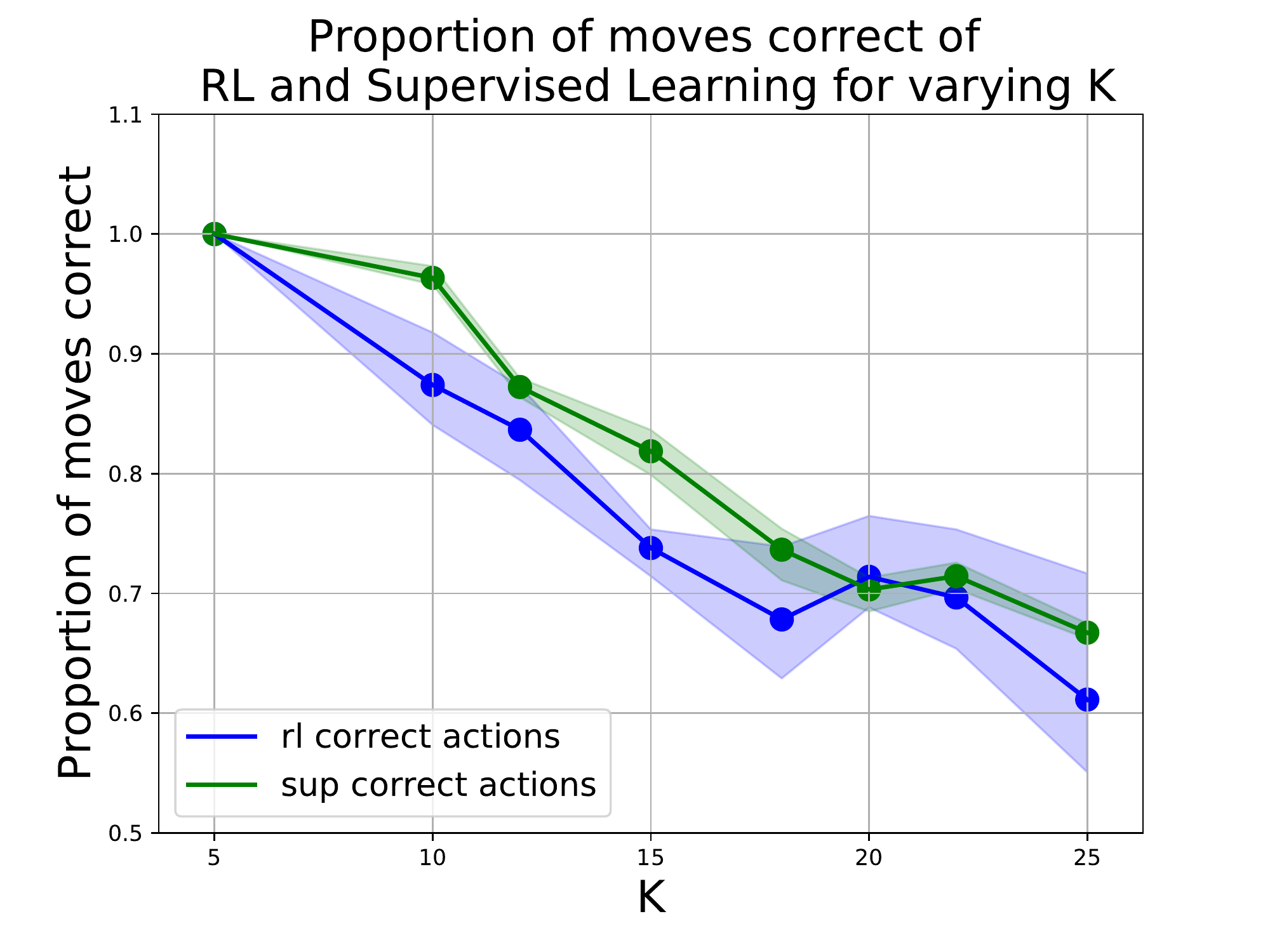}
   \hspace*{-6mm}
   &
   \includegraphics[width=0.55\columnwidth]{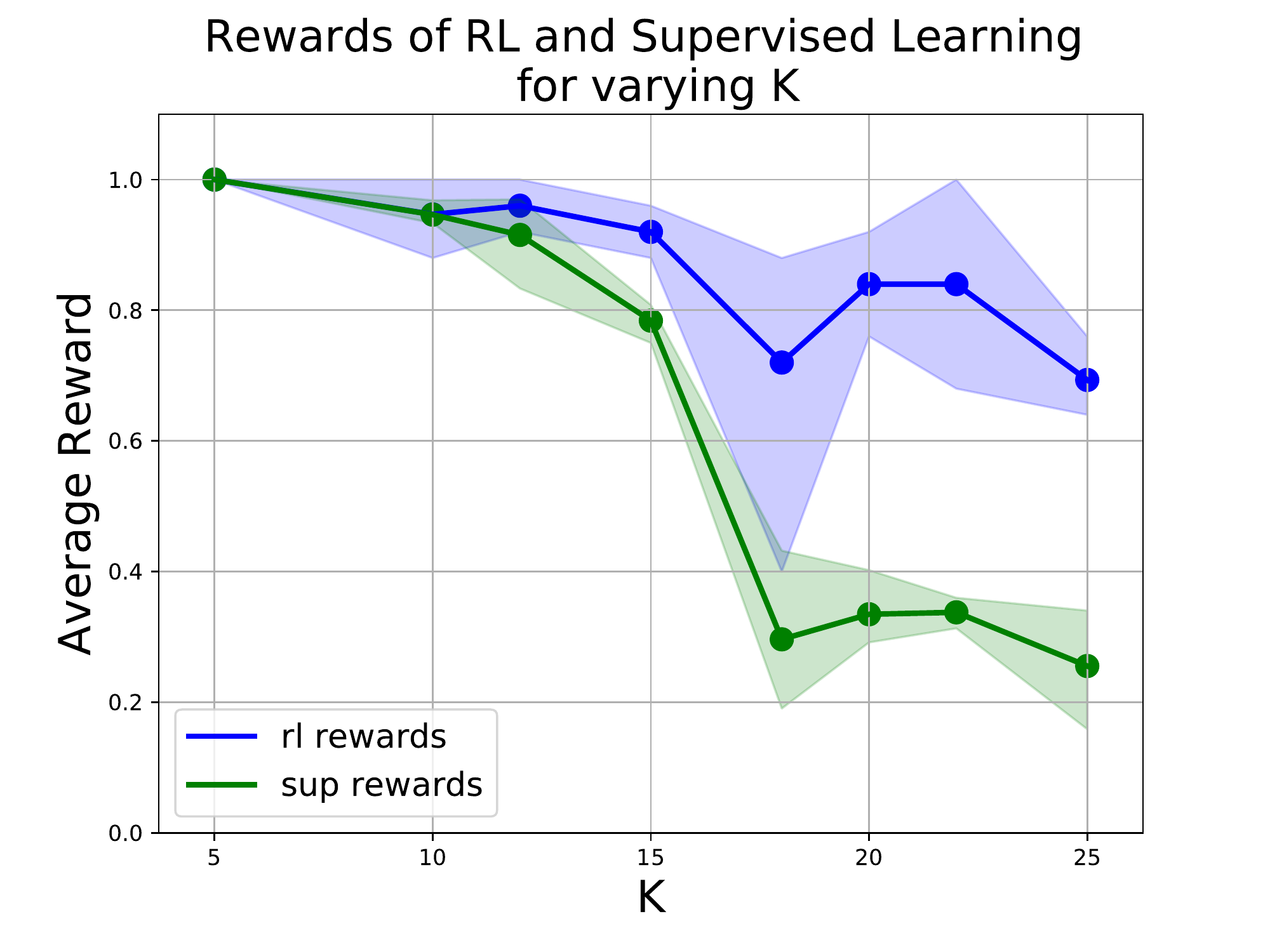}
  \end{tabular}
  \caption{\small Plots comparing the performance of Supervised Learning and RL on the Attacker Defender Game for different choices of K. The left pane shows the proportion of moves correct for supervised learning and RL (according to the ground truth). Unsurprisingly, we see that supervised learning is better on average at getting the ground truth correct move. However, RL is better at playing the game: a policy trained through RL significantly outperforms a policy trained through supervised learning (right pane), with the difference growing for larger $K$.}
  \vspace*{-1.4em}
  \label{fig-supervised}
\end{figure}

\section{Supervised Learning vs RL}
Aside from testing the generalization of RL, the Attacker-Defender game also enables us to make a comparison with Supervised Learning. The closed-form optimal policy enables an evaluation of the ground truth on a \textit{per move} basis. We can thus compare RL to a Supervised Learning setup, where we classify the correct action on a large set of sampled states. To carry out this test in practice, we first train a defender policy with reinforcement learning, saving all observations seen to a dataset. We then train a supervised network (with the same architecture as the defender policy) on this dataset to classify the optimal action. We then test the supervised network to determine how well it can play. 

We find an interesting dichotomy between the proportion of correct moves and the average reward. Unsurprisingly, supervised learning boasts a higher proportion of correct moves: if we keep count of the ground truth correct move for each turn in the game, RL has a lower proportion of correct moves compared to supervised learning (Figure \ref{fig-supervised} left pane). However, in the right pane of Figure \ref{fig-supervised}, we see that RL is better at playing the game, achieving higher average reward for all difficulty settings, and significantly beating supervised learning as $K$ grows. 

This contrast forms an interesting counterpart to recent findings of \cite{silver2017alphagozero}, who in the context of Go also compared reinforcement learning to supervised approaches.  A key distinction is that their supervised work was relative to a heuristic objective, whereas in our domain we are able to compare to provably optimal play.  This both makes it possible to rigorously define the notion of a mistake, and also to perform more fine-grained analysis as we do in the remainder of this section. 

Specifically, how is it that the RL agent is achieving higher reward in the game, if it is making more mistakes at a per-move level?  To gain further insight into this, we categorize the per-move mistakes into different types, and study them separately.  In particular, suppose the defender is presented with a partition of the pieces into two sets, where as before we assume without loss of generality that $\phi(A) \geq \phi(B)$.  We say that the defender makes a {\em terminal mistake} if $\phi(A) \geq 0.5$ while $\phi(B) < 0.5$, but the defender chooses to destroy $B$.  Note that this means the defender now faces a forced loss against optimal play, whereas it could have forced a win with optimal play had it destroyed $A$.  We also define a subset of the family of terminal mistakes as follows: we say that the defender makes a {\em fatal mistake} if $\phi(A) + \phi(B) < 1$, but $\phi(A) \geq 0.5$ and the defender chooses to destroy $B$.  Note that a fatal mistake is one that converts a position where the defender had a forced win to one where it has a forced loss.

\begin{figure}[t]
\centering
\begin{tabular}{cc}
   \hspace*{-5mm} \includegraphics[width=0.55\columnwidth]{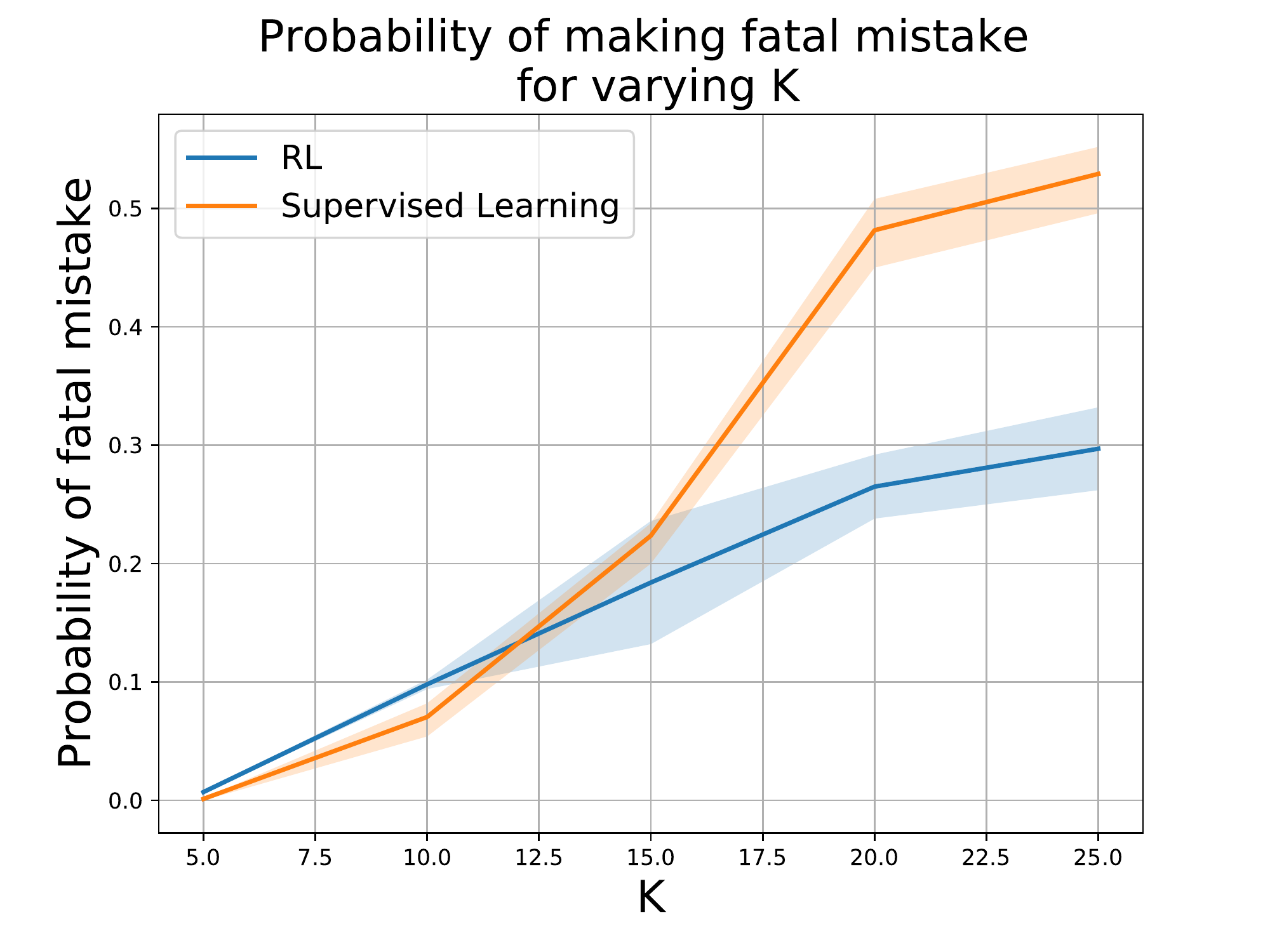}
   \hspace*{-6mm}
   &
   \includegraphics[width=0.55\columnwidth]{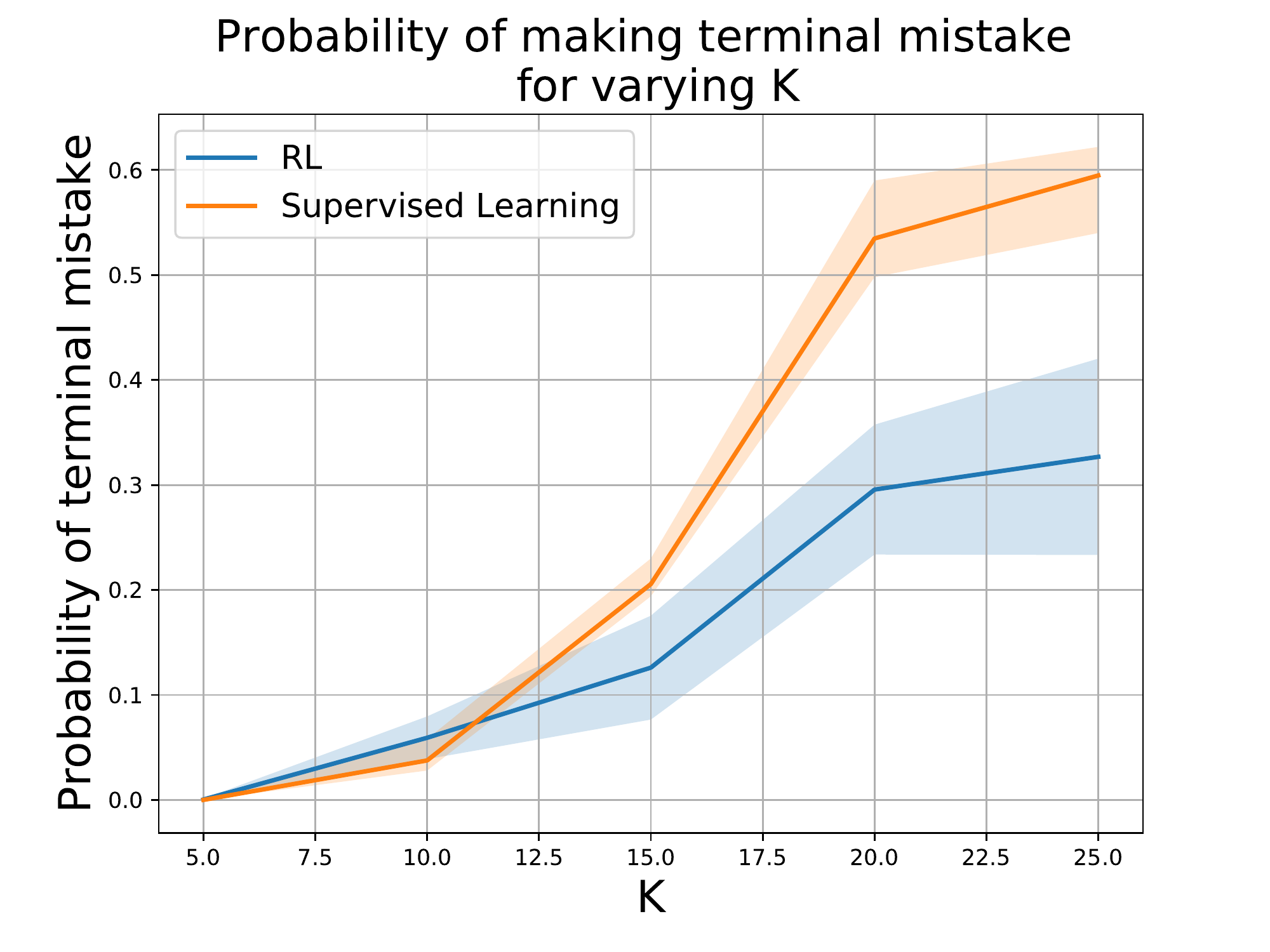}
   \end{tabular}
  \caption{\small Figure showing the frequencies of different kinds of mistakes made by supervised learning and RL that would cost the game. The left pane shows the frequencies of fatal mistakes: where the agent goes from a winning state (potential $<1$) to a losing state (potential $\geq 1$). A superset of this kind of mistake, terminal mistakes, look at where the agent makes the wrong choice (irrespective of state potential), destroying (wlog) set $A$ with $\phi(A) < 0.5$, instead of $B$, with $\phi(B) \geq 0.5$. In both cases we see that RL makes significantly fewer mistakes than supervised learning, particularly as difficulty increases.}
  \label{fig-end-game}
\end{figure}

In Figure \ref{fig-end-game}, we see that especially as $K$ gets larger, reinforcement learning makes terminal and fatal mistakes at a much lower rate than supervised learning does. This suggests a basis for the different in performance: even if supervised learning is making fewer mistakes overall, it is making more mistakes in certain well-defined consequential situations.

\section{Conclusion}
In this paper, we have proposed Erdos-Selfridge-Spencer games as rich environments for investigating reinforcement learning, exhibiting continuously tunable difficulty and an exact combinatorial characterization of optimal behavior.  We have demonstrated that algorithms can exhibit wide variation in performance as we tune the game's difficulty, and we use the characterization of optimal behavior to evaluate generalization over raw performance. We provide theoretical insights that enable multiagent play and, through binary search, self play. Finally, we compare RL and Supervised Learning, highlighting interesting tradeoffs between per move optimality, average reward and fatal mistakes. We also develop further results in the Appendix, including an analysis of catastrophic forgetting,  generalization across different values of the game's parameters, and a method for investigating measures of the model's confidence.  We believe that this family of combinatorial games can be used as a rich environment for gaining further insights into deep reinforcement learning.

\clearpage

\bibliography{iclr2018_conference}

\begin{thebibliography}{20}
\providecommand{\natexlab}[1]{#1}
\providecommand{\url}[1]{\texttt{#1}}
\expandafter\ifx\csname urlstyle\endcsname\relax
  \providecommand{\doi}[1]{doi: #1}\else
  \providecommand{\doi}{doi: \begingroup \urlstyle{rm}\Url}\fi

\bibitem[Beattie et~al.(2016)Beattie, Leibo, Teplyashin, Ward, Wainwright,
  K{\"u}ttler, Lefrancq, Green, Vald{\'e}s, Sadik, et~al.]{deepmindlab}
Charles Beattie, Joel~Z Leibo, Denis Teplyashin, Tom Ward, Marcus Wainwright,
  Heinrich K{\"u}ttler, Andrew Lefrancq, Simon Green, V{\'\i}ctor Vald{\'e}s,
  Amir Sadik, et~al.
\newblock Deepmind lab.
\newblock \emph{arXiv preprint arXiv:1612.03801}, 2016.

\bibitem[Bouzy \& M{\'e}tivier(2010)Bouzy and
  M{\'e}tivier]{Bouzy2010MultiagentLE}
Bruno Bouzy and Marc M{\'e}tivier.
\newblock Multi-agent learning experiments on repeated matrix games.
\newblock In \emph{ICML}, 2010.

\bibitem[Bowling et~al.(2017)Bowling, Burch, Johanson, and
  Tammelin]{Bowling2017poker}
Michael Bowling, Neil Burch, Michael Johanson, and Oskari Tammelin.
\newblock Heads-up limit hold'em poker is solved.
\newblock \emph{Commun. ACM}, 60:\penalty0 81--88, 2017.

\bibitem[Brockman et~al.(2016)Brockman, Cheung, Pettersson, Schneider,
  Schulman, Tang, and Zaremba]{openaigym}
Greg Brockman, Vicki Cheung, Ludwig Pettersson, Jonas Schneider, John Schulman,
  Jie Tang, and Wojciech Zaremba.
\newblock Openai gym, 2016.

\bibitem[Duan et~al.(2016)Duan, Chen, Houthooft, Schulman, and
  Abbeel]{duan2016benchmarking}
Yan Duan, Xi~Chen, Rein Houthooft, John Schulman, and Pieter Abbeel.
\newblock Benchmarking deep reinforcement learning for continuous control.
\newblock In \emph{International Conference on Machine Learning}, pp.\
  1329--1338, 2016.

\bibitem[Erdos \& Selfridge(1973)Erdos and Selfridge]{erdos1973game}
Paul Erdos and John Selfridge.
\newblock On a combinatorial game.
\newblock \emph{Journal of Combinatorial Theory}, 14:\penalty0 298--301, 1973.

\bibitem[Henderson et~al.(2017)Henderson, Islam, Bachman, Pineau, Precup, and
  Meger]{hendersondeeprlmatters}
Peter Henderson, Riashat Islam, Philip Bachman, Joelle Pineau, Doina Precup,
  and David Meger.
\newblock Deep reinforcement learning that matters.
\newblock In \emph{AAAI 2018}, 2017.

\bibitem[Hesse et~al.(2017)Hesse, Plappert, Radford, Schulman, Sidor, and
  Wu]{baselines}
Christopher Hesse, Matthias Plappert, Alec Radford, John Schulman, Szymon
  Sidor, and Yuhuai Wu.
\newblock Openai baselines.
\newblock \url{https://github.com/openai/baselines}, 2017.

\bibitem[Littman(1994)]{Littman1994markovgames}
Michael~L. Littman.
\newblock Markov games as a framework for multi-agent reinforcement learning.
\newblock In \emph{Proceedings of the Eleventh International Conference on
  International Conference on Machine Learning}, pp.\  157--163, 1994.

\bibitem[Mania et~al.(2018)Mania, Guy, and Recht]{mania2018randomsearch}
Horia Mania, Aurelia Guy, and Benjamin Recht.
\newblock Simple random search provides a competitive approach to reinforcement
  learning.
\newblock abs/1803.07055, 2018.
\newblock URL \url{http://arxiv.org/abs/1803.07055}.

\bibitem[Mnih et~al.(2015)Mnih, Kavukcuoglu, Silver, Rusu, Veness, Bellemare,
  Graves, Riedmiller, Fidjeland, Ostrovski, Petersen, Beattie, Sadik,
  Antonoglou, King, Kumaran, Wierstra, Legg, and Hassabis]{mnih2015dqn}
Volodymyr Mnih, Koray Kavukcuoglu, David Silver, Andrei~A. Rusu, Joel Veness,
  Marc~G. Bellemare, Alex Graves, Martin Riedmiller, Andreas~K. Fidjeland,
  Georg Ostrovski, Stig Petersen, Charles Beattie, Amir Sadik, Ioannis
  Antonoglou, Helen King, Dharshan Kumaran, Daan Wierstra, Shane Legg, and
  Demis Hassabis.
\newblock Human-level control through deep reinforcement learning.
\newblock \emph{Nature}, 518\penalty0 (7540):\penalty0 529--533, Feb 2015.
\newblock ISSN 0028-0836.
\newblock URL \url{http://dx.doi.org/10.1038/nature14236}.

\bibitem[Mnih et~al.(2016)Mnih, Badia, Mirza, Graves, Lillicrap, Harley,
  Silver, and Kavukcuoglu]{mnih2016a2c}
Volodymyr Mnih, Adria~Puigdomenech Badia, Mehdi Mirza, Alex Graves, Timothy~P.
  Lillicrap, Tim Harley, David Silver, and Koray Kavukcuoglu.
\newblock Asynchronous methods for deep reinforcement learning.
\newblock \emph{arXiv preprint arxiv:1602.01783}, 2016.

\bibitem[Rajeswaran et~al.(2016)Rajeswaran, Ghotra, Levine, and
  Ravindran]{rajeswaran2016epopt}
Aravind Rajeswaran, Sarvjeet Ghotra, Sergey Levine, and Balaraman Ravindran.
\newblock Epopt: Learning robust neural network policies using model ensembles.
\newblock \emph{arXiv preprint arXiv:1610.01283}, 2016.

\bibitem[Schulman et~al.(2017)Schulman, Wolski, Dhariwal, Radford, and
  Klimov]{schulman2017ppo}
John Schulman, Filip Wolski, Prafulla Dhariwal, Alec Radford, and Oleg Klimov.
\newblock Proximal policy optimization algorithms.
\newblock \emph{arXiv preprint arxiv:1707.06347}, 2017.

\bibitem[Silver et~al.(2017)Silver, Schrittwieser, Simonyan, Antonoglou, Huang,
  Guez, Hubert, Baker, Lai, Bolton, Chen, Lillicrap, Hui, Sifre, van~den
  Driessche, Graepel, and Hassabis]{silver2017alphagozero}
David Silver, Julian Schrittwieser, Karen Simonyan, Ioannis Antonoglou, Aja
  Huang, Arthur Guez, Thomas Hubert, Lucas Baker, Matthew Lai, Adrian Bolton,
  Yutian Chen, Timothy Lillicrap, Fan Hui, Laurent Sifre, George van~den
  Driessche, Thore Graepel, and Demis Hassabis.
\newblock Mastering the game of go without human knowledge.
\newblock \emph{Nature}, 550\penalty0 (7676):\penalty0 354--359, Oct 2017.
\newblock ISSN 0028-0836.
\newblock URL \url{http://dx.doi.org/10.1038/nature24270}.

\bibitem[Spencer(1994)]{spencer1994game}
Joel Spencer.
\newblock Randomization, derandomization and antirandomization: Three games.
\newblock \emph{Theoretical Computer Science.}, 131:\penalty0 415--429, 09
  1994.

\bibitem[Stephenson(1909)]{stephenson1909cartpole}
Andrew Stephenson.
\newblock Lxxi. on induced stability.
\newblock \emph{The London, Edinburgh, and Dublin Philosophical Magazine and
  Journal of Science}, 17\penalty0 (101):\penalty0 765--766, 1909.

\bibitem[Tobin et~al.(2017)Tobin, Fong, Ray, Schneider, Zaremba, and
  Abbeel]{tobin2017domain}
Josh Tobin, Rachel Fong, Alex Ray, Jonas Schneider, Wojciech Zaremba, and
  Pieter Abbeel.
\newblock Domain randomization for transferring deep neural networks from
  simulation to the real world.
\newblock \emph{arXiv preprint arXiv:1703.06907}, 2017.

\bibitem[Todorov et~al.(2012)Todorov, Erez, and Tassa]{todorov2012mujoco}
Emanuel Todorov, Tom Erez, and Yuval Tassa.
\newblock Mujoco: A physics engine for model-based control.
\newblock In \emph{Intelligent Robots and Systems (IROS), 2012 IEEE/RSJ
  International Conference on}, pp.\  5026--5033. IEEE, 2012.

\bibitem[Zinkevich et~al.(2011)Zinkevich, Bowling, and
  Wunder]{Zinkevich2011Lemonade}
Martin~A. Zinkevich, Michael Bowling, and Michael Wunder.
\newblock The lemonade stand game competition: Solving unsolvable games.
\newblock \emph{SIGecom Exch.}, 10:\penalty0 35--38, 2011.

\end{thebibliography}
\bibliographystyle{iclr2018_conference}

\clearpage

\part*{Appendix}

\section*{Additional Definition and Results from Section 4}
Note that for evaluating the defender RL agent, we initially use a slightly suboptimal attacker, which randomizes over playing optimally and with a \textit{disjoint support strategy}. The disjoint support strategy is a suboptimal verison (for better exploration) of the prefix attacker described in Theorem \ref{thm-prefix-attacker}. Instead of finding the partition that results in sets $A, B$ as equal as possible, the disjoint support attacker greedily picks $A, B$ so that there is a potential difference between the two sets, with the fraction of total potential for the smaller set being uniformly sampled from $[0.1, 0.2, 0.3, 0.4]$ at each turn. This exposes the defender agent to sets of uneven potential, and helps it develop a more generalizable strategy.

To train our defender agent, we use a fully connected deep neural network with $2$ hidden layers of width $300$ to represent our policy. We decided on these hyperparameters after some experimentation with different depths and widths, where we found that network width did not have a significant effect on performance, but, as seen in Section 4, slightly deeper models ($1$ or $2$ hidden layers) performed
noticeably better than shallow networks.

\section*{Additional Results from Section \ref{sec-generalization}}
Here in Figure \ref{fig-attacker-K} we show results of training the attacker agent using the alternative parametrization given with Theorem \ref{thm-prefix-attacker} with PPO and A2C (DQN results were very poor and have been omitted.)  

\begin{figure}
  \centering
  \begin{tabular}{cc}
  \vspace*{-5mm}
   \hspace*{-20mm}\includegraphics[width=0.5\columnwidth]{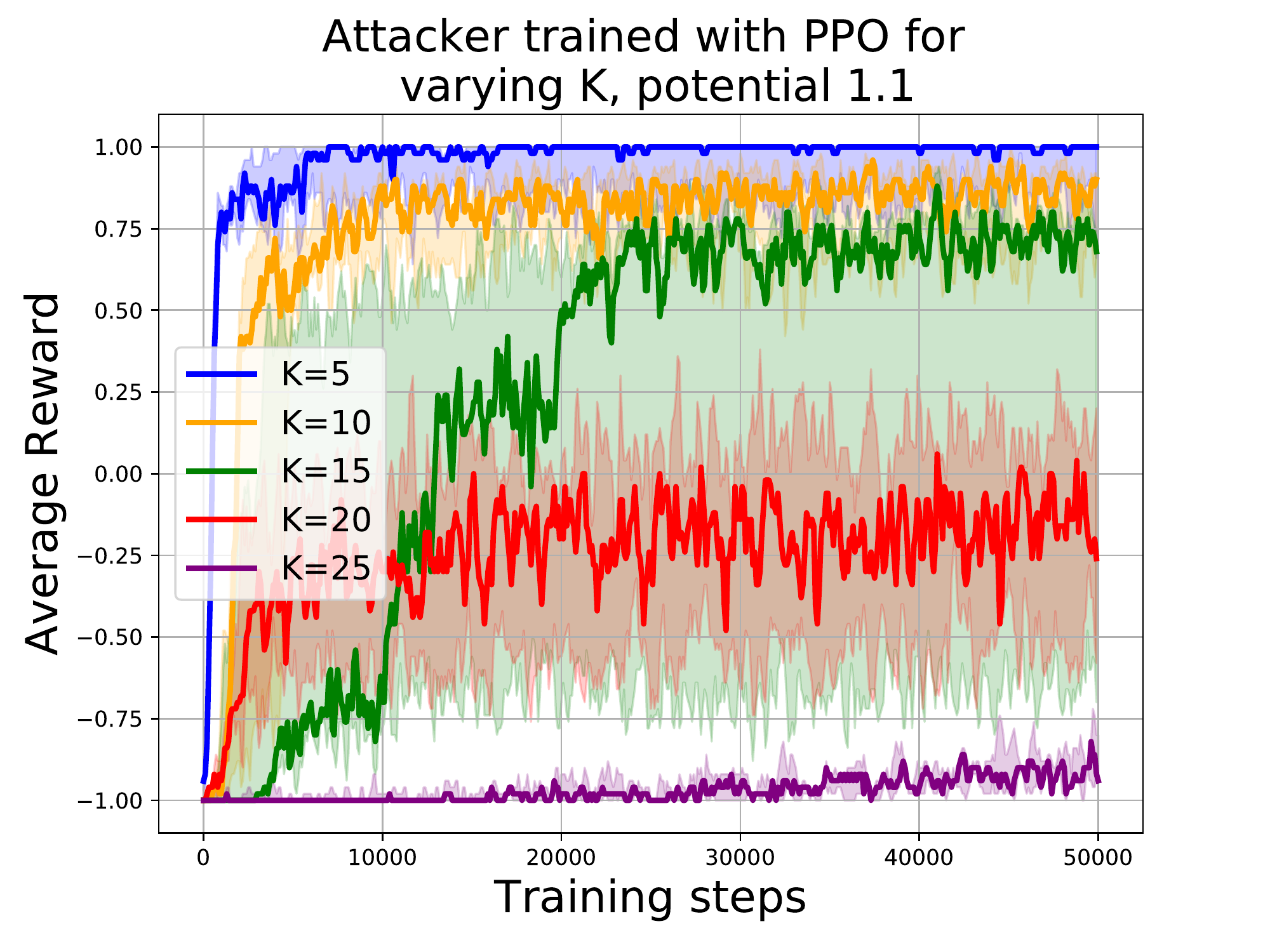}\hspace*{-20mm}
  &
  \hspace*{-15mm}\includegraphics[width=0.5\columnwidth]{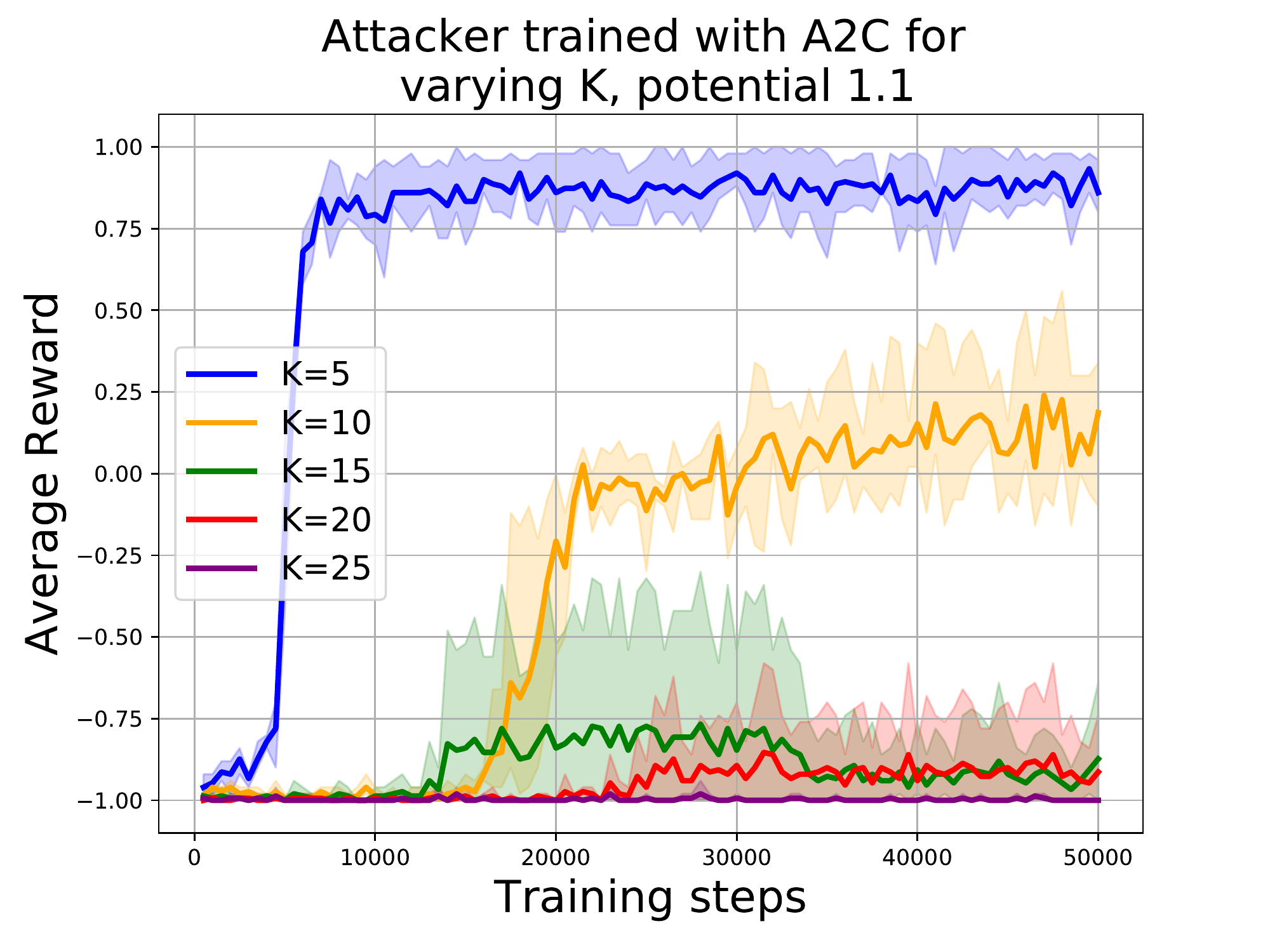}\hspace*{-9mm} 
  \\
  \hspace*{-20mm}
  \includegraphics[width=0.5\columnwidth]{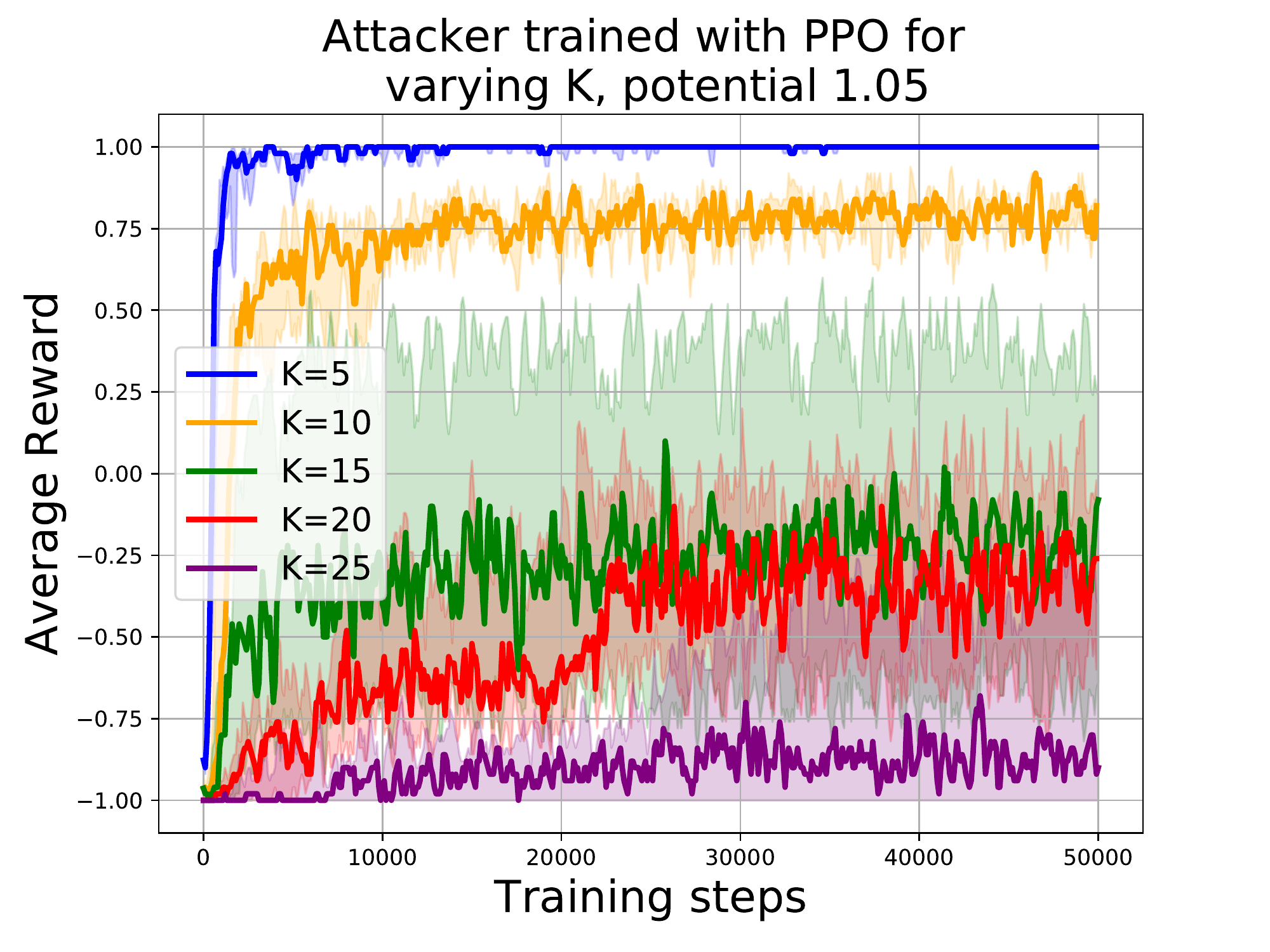}\hspace*{-20mm}
  &
    \hspace*{-15mm}
  \includegraphics[width=0.5\columnwidth]{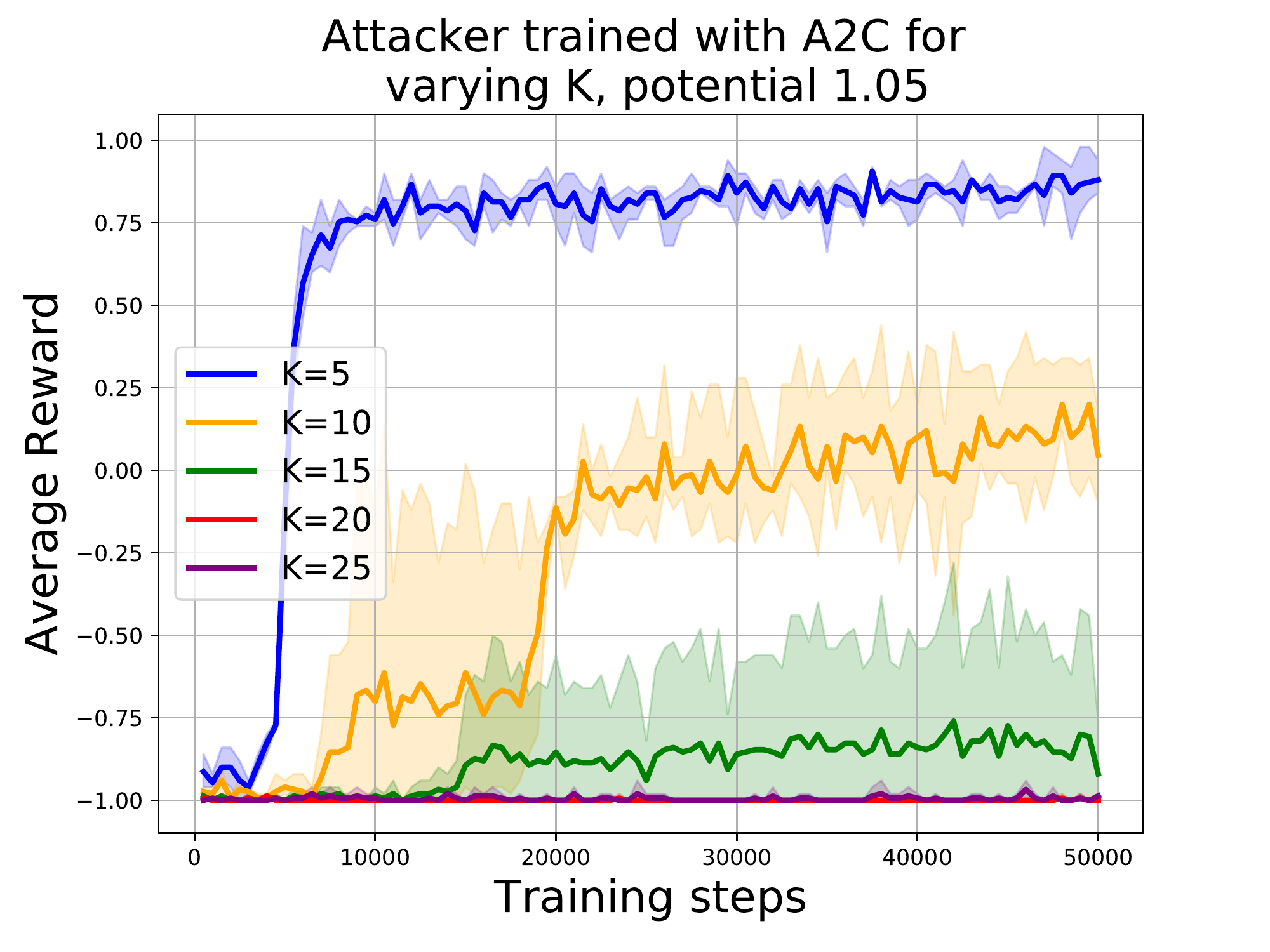}\hspace*{-9mm} 
  \\
  \includegraphics[width=0.5\columnwidth]{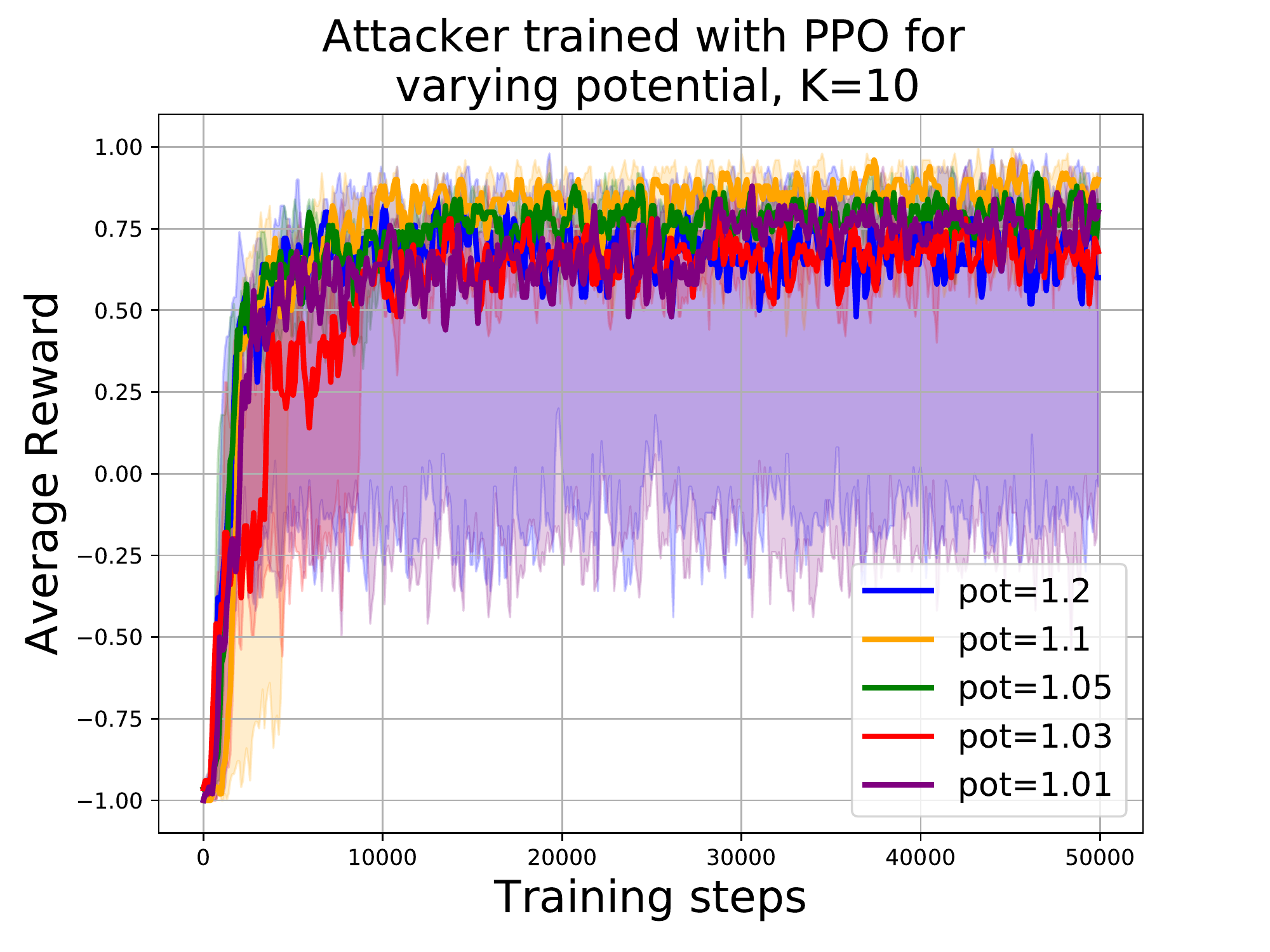}
  &
      \hspace*{-7mm}
  \includegraphics[width=0.5\columnwidth]{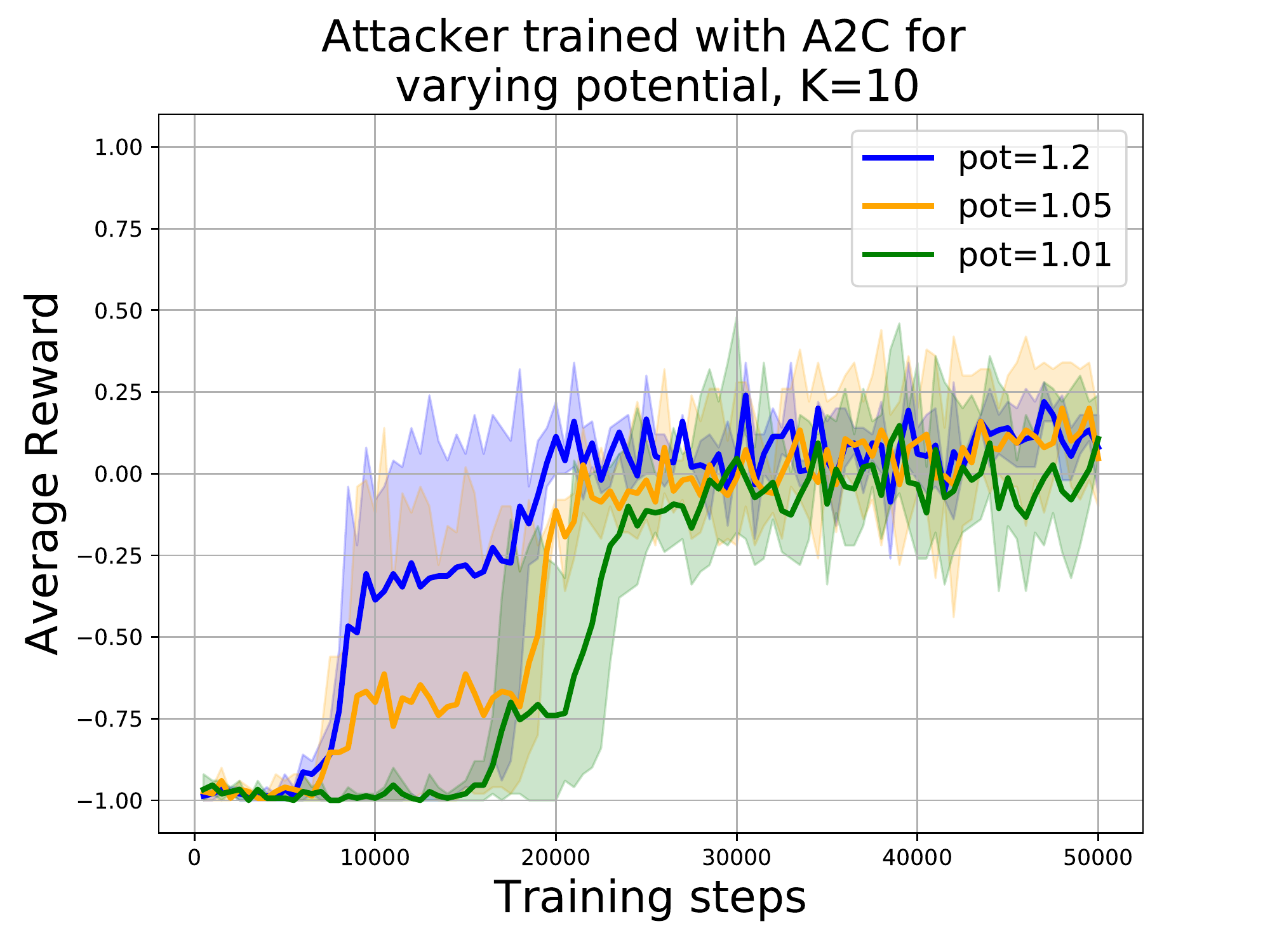}\\
  \end{tabular}
  
  \caption{\small Performance of PPO and A2C on training the attacker agent for different difficulty settings. DQN performance was very poor (reward < $-0.8$ at $K=5$ with best hyperparams). We see much greater variation of performance with changing $K$, which now affects the sparseness of the reward as well as the size of the action space. There is less variation with potential, but we see a very high performance variance with lower (harder) potentials.}
  \vspace*{-1.4em}
  \label{fig-attacker-K}
\end{figure}

\section*{Other Generalization Phenomena}

\subsection*{Generalizing Over Different Potentials and $K$}
\label{sec-vary-K}

In the main text we examined how our RL defender agent performance varies as we change the difficulty settings of the game, either the potential or $K$. Returning again to the fact that the Attacker-Defender game has an expressible optimal that generalizes across all difficulty settings, we might wonder how training on one difficulty setting and testing on a different setting perform. Testing on different potentials in this way is straightforwards, but testing on different $K$ requires a slight reformulation. our input size to the defender neural network policy is $2(K+1)$, and so naively changing to a different number of levels will not work. Furthermore, training on a smaller $K$ and testing on larger $K$ is not a fair test -- the model cannot be expected to learn how to weight the lower levels. However, testing the converse (training on larger $K$ and testing on smaller $K$) is both easily implementable and offers a legitimate test of generalization. We find (a subset of plots shown in Figure \ref{fig-vary-pot-K}) that when varying potential, training on harder games results in better generalization. When testing on a smaller $K$ than the one used in training, performance is inverse to the difference between train $K$ and test $K$. 

\begin{figure}
  \centering
  \begin{tabular}{cc}
   \hspace*{-5mm} \includegraphics[width=0.6\columnwidth]{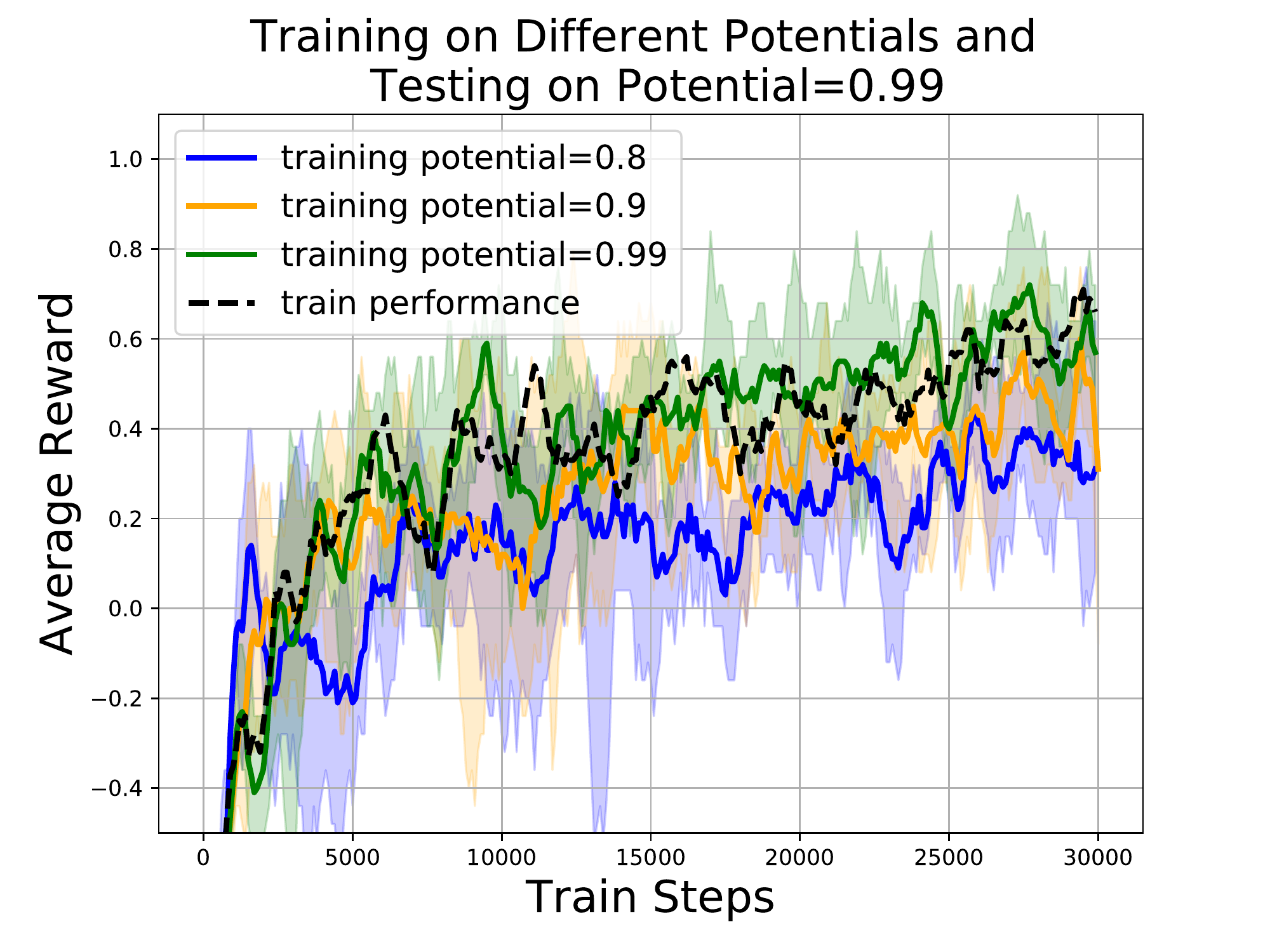}
   \hspace*{-9mm}
   &
   \includegraphics[width=0.6\columnwidth]{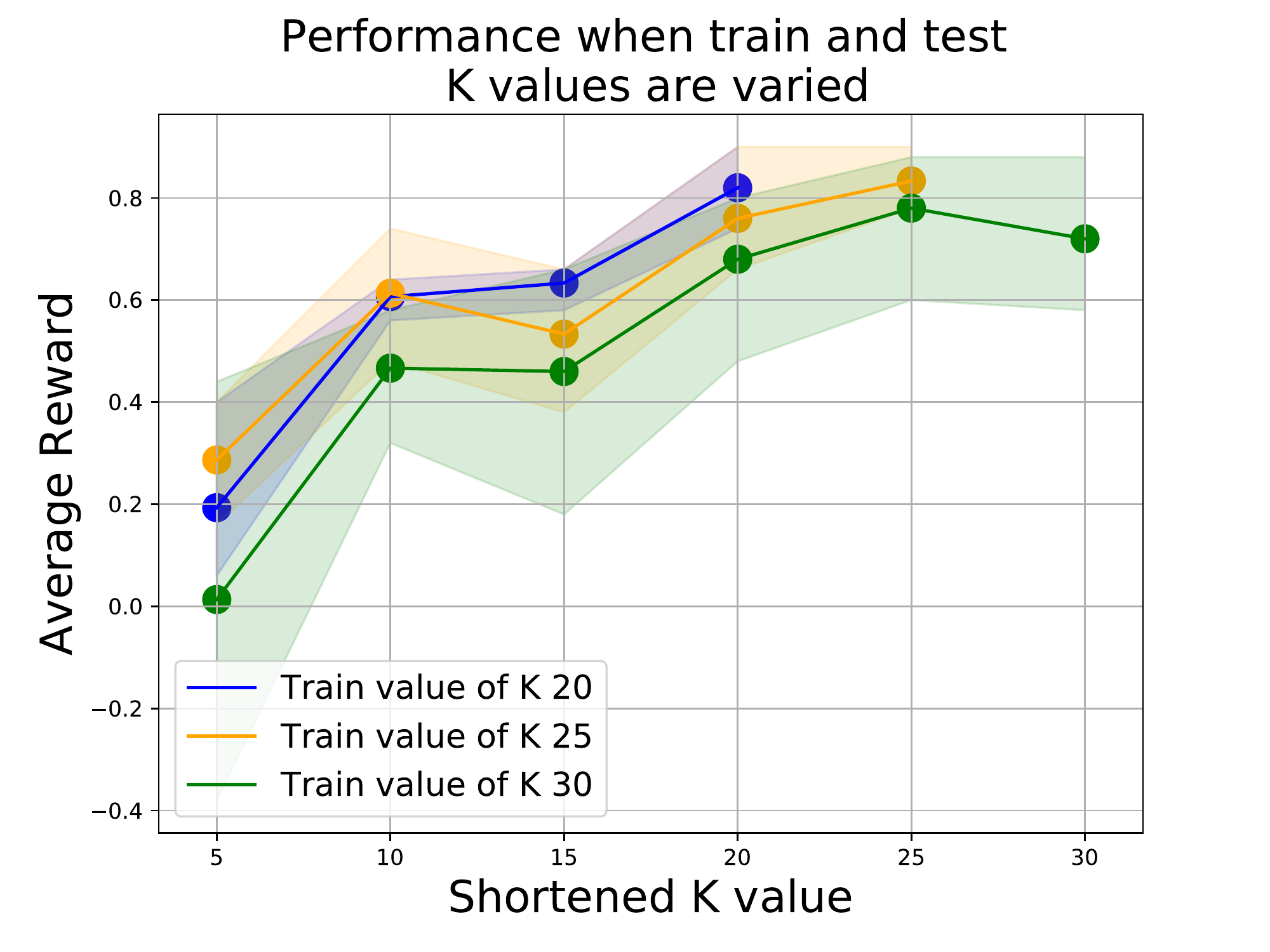}
  \end{tabular}
  \caption{On the left we train on different potentials and test on potential $0.99$. We find that training on harder games leads to better performance, with the agent trained on the easiest potential generalizing worst and the agent trained on a harder potential generalizing best. This result is consistent across different choices of test potentials. The right pane shows the effect of training on a larger $K$ and testing on smaller $K$. We see that performance appears to be inversely proportional to the difference between the train $K$ and test $K$.}
  \label{fig-vary-pot-K}
\end{figure}

\subsection*{Catastrophic Forgetting and Curriculum Learning}
Recently, several papers have identified the issue of catastrophic forgetting in Deep Reinforcement Learning, where switching between different tasks results in destructive interference and lower performance instead of positive transfer. We witness effects of this form in the Attacker-Defender games. As in Section \ref{sec-vary-K}, our two environments differ in the $K$ that we use -- we first try training on a small $K$, and then train on larger $K$. For lower difficulty (potential) settings, we see that this curriculum learning improves play, but for higher potential settings, the learning interferes catastrophically, Figure \ref{fig-catastrophic-forgetting-defender}

\begin{figure*}
  \centering
   \hspace*{-5mm} \includegraphics[scale=0.2]{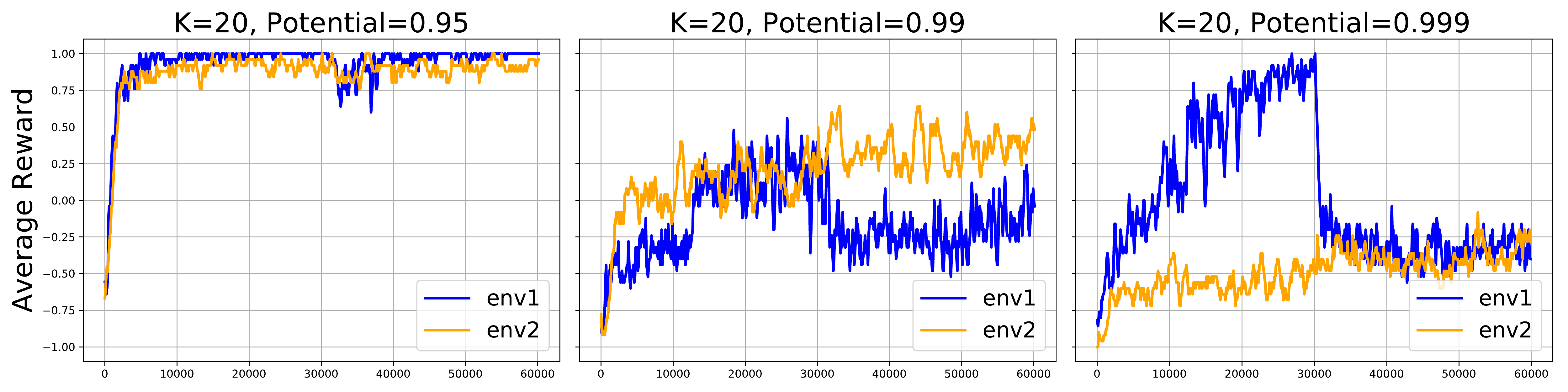}
  \caption{Defender agent demonstrating catastrophic forgetting when trained on environment $1$ with smaller $K$ and environment $2$ with larger $K$.}
  \vspace*{-1.4em}
  \label{fig-catastrophic-forgetting-defender}
\end{figure*}

\section*{Understanding Model Failures}

\subsection*{Value of the Null Set}
The significant performance drop we see in Figure \ref{fig-single-agent-overfit} motivates investigating whether there are simple rules of thumb that the model has successfully learned. Perhaps the simplest rule of thumb is learning the value of the null set: if one of $A, B$ (say $A$) consists of only zeros and the other ($B$) has some pieces, the defender agent should reliably choose to destroy $B$. Surprisingly, even this simple rule of thumb is violated, and even more frequently for larger $K$, Figure \ref{fig-zerostate-single}.

\begin{figure}
\centering
   \includegraphics[scale=0.4]{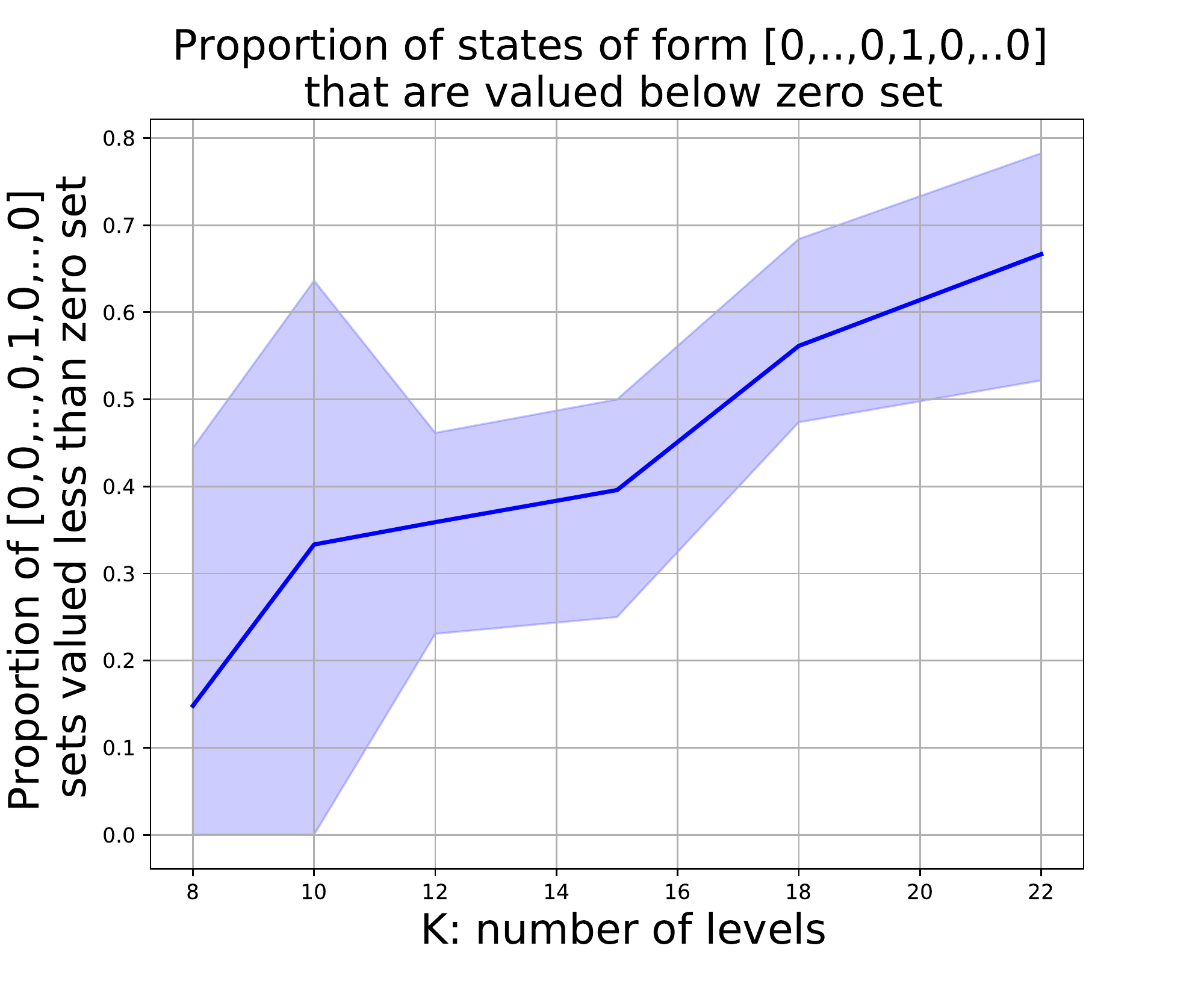}
  \caption{Figure showing proportion of sets of form $[0,...,0,1,0,...,0]$ that are valued less than the null set. Out of the $K+1$ possible one hot sets, we determine the proportion that are not picked when paired with the null (zero) set, and plot this value for different $K$.}
  \label{fig-zerostate-single}
\end{figure}

\subsection*{Model Confidence}
We can also test to see if the model outputs are well \textit{calibrated} to the potential values: is the model more confident in cases where there is a large discrepancy between potential values, and fifty-fifty where the potential is evenly split? The results are shown in Figure \ref{fig:-model-confidence}.

\begin{figure}
    \centering
    \includegraphics[width=0.45\textwidth]{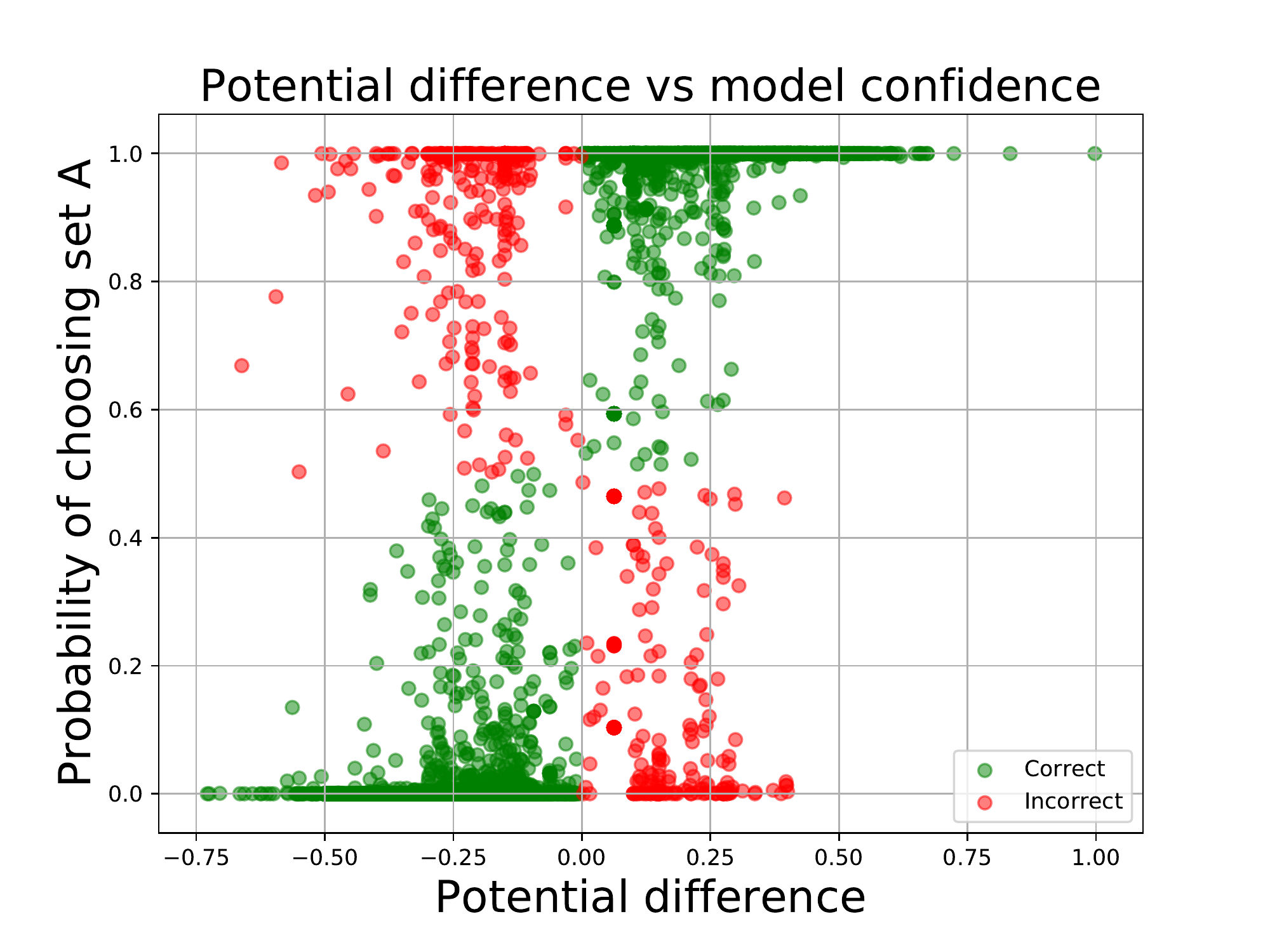}
    \includegraphics[width=0.45\textwidth]{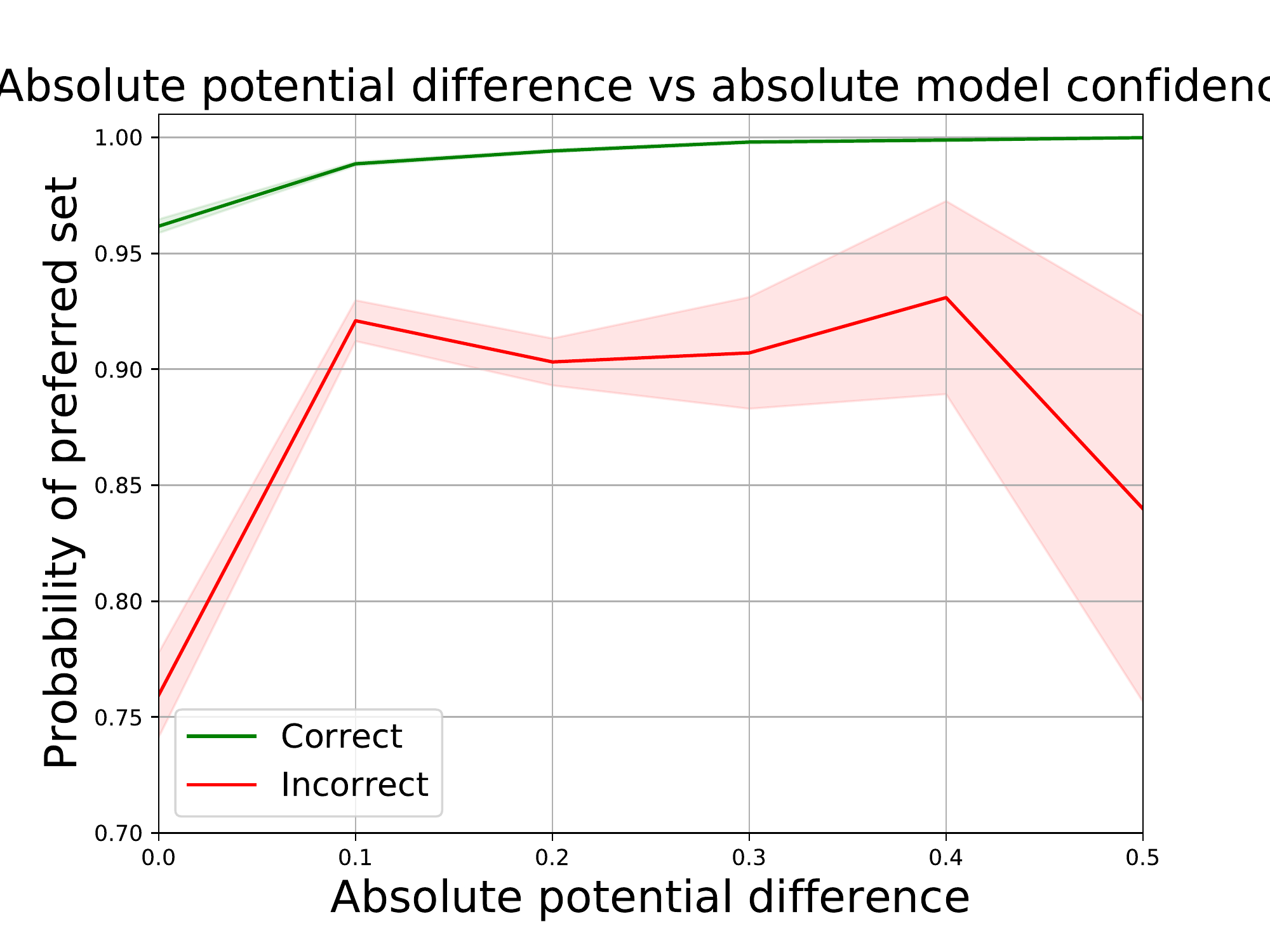}
    \caption{Confidence as a function of potential difference between states. The top figure shows true potential differences and model confidences; green dots are moves where the model prefers to make the right prediction, while red moves are moves where it prefers to make the wrong prediction. The right shows the same data, plotting the \emph{absolute} potential difference and
    absolute model confidence in its preferred move. Remarkably, an increase in the potential difference associated with an increase in model confidence over a wide range, even when the model is wrong.}
    \label{fig:-model-confidence}
\end{figure}

\subsection{Generalizing across Start States and Opponent Strategies}
In the main paper, we mixed between different start state distributions to ensure a wide variety of states seen. This begets the natural question of how well we can generalize across start state distribution if we train on purely one distribution. The results in Figure \ref{fig-start-state} show that training naively on an `easy' start state distribution (one where most of the states seen are very similar to one another) results in a significant performance drop when switching distribution.

\begin{figure}
  \centering
  \begin{tabular}{cc}
   \hspace*{-15mm} \includegraphics[width=0.6\columnwidth]{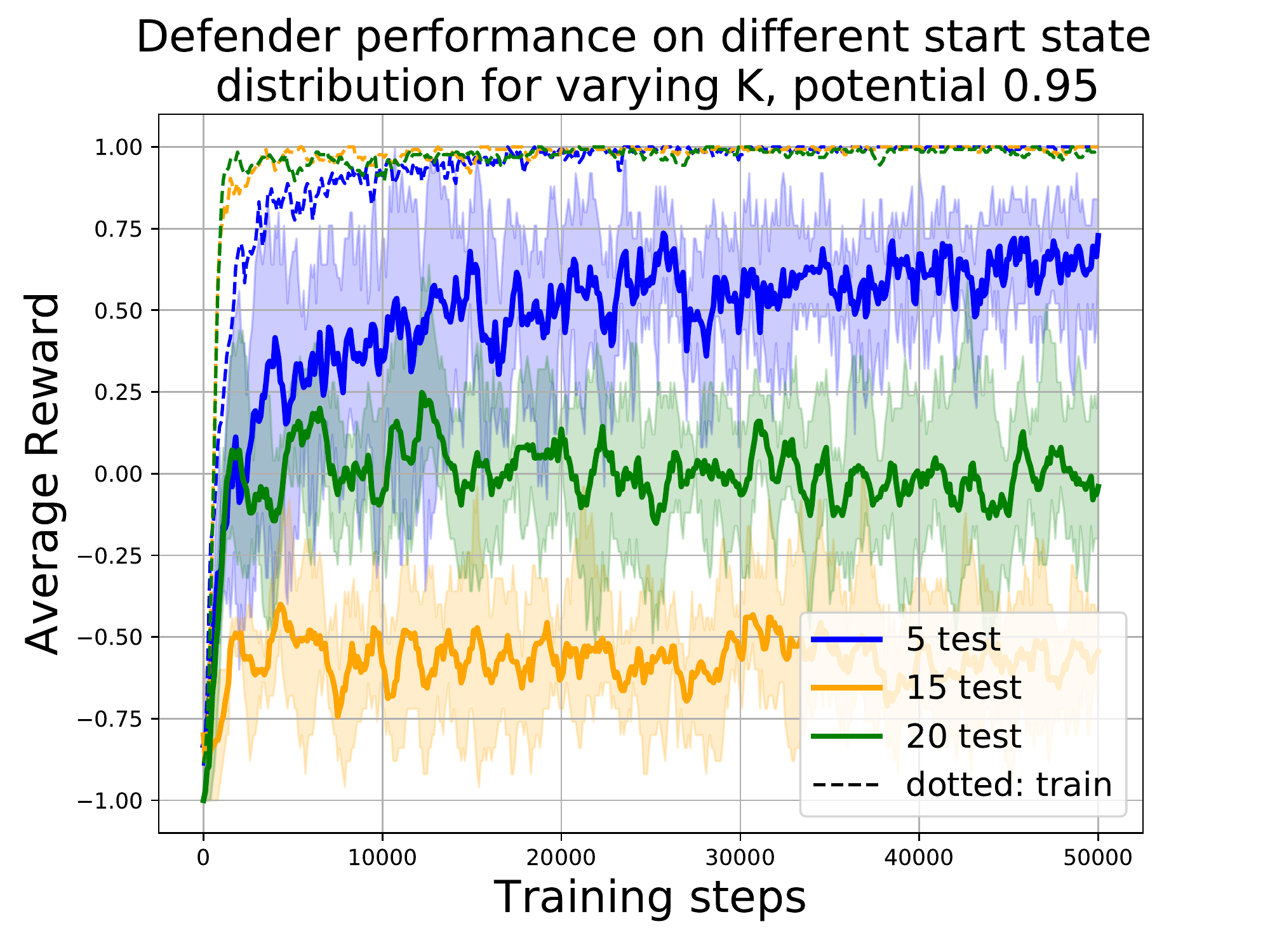}
   \hspace*{-9mm}
   &
   \includegraphics[width=0.6\columnwidth]{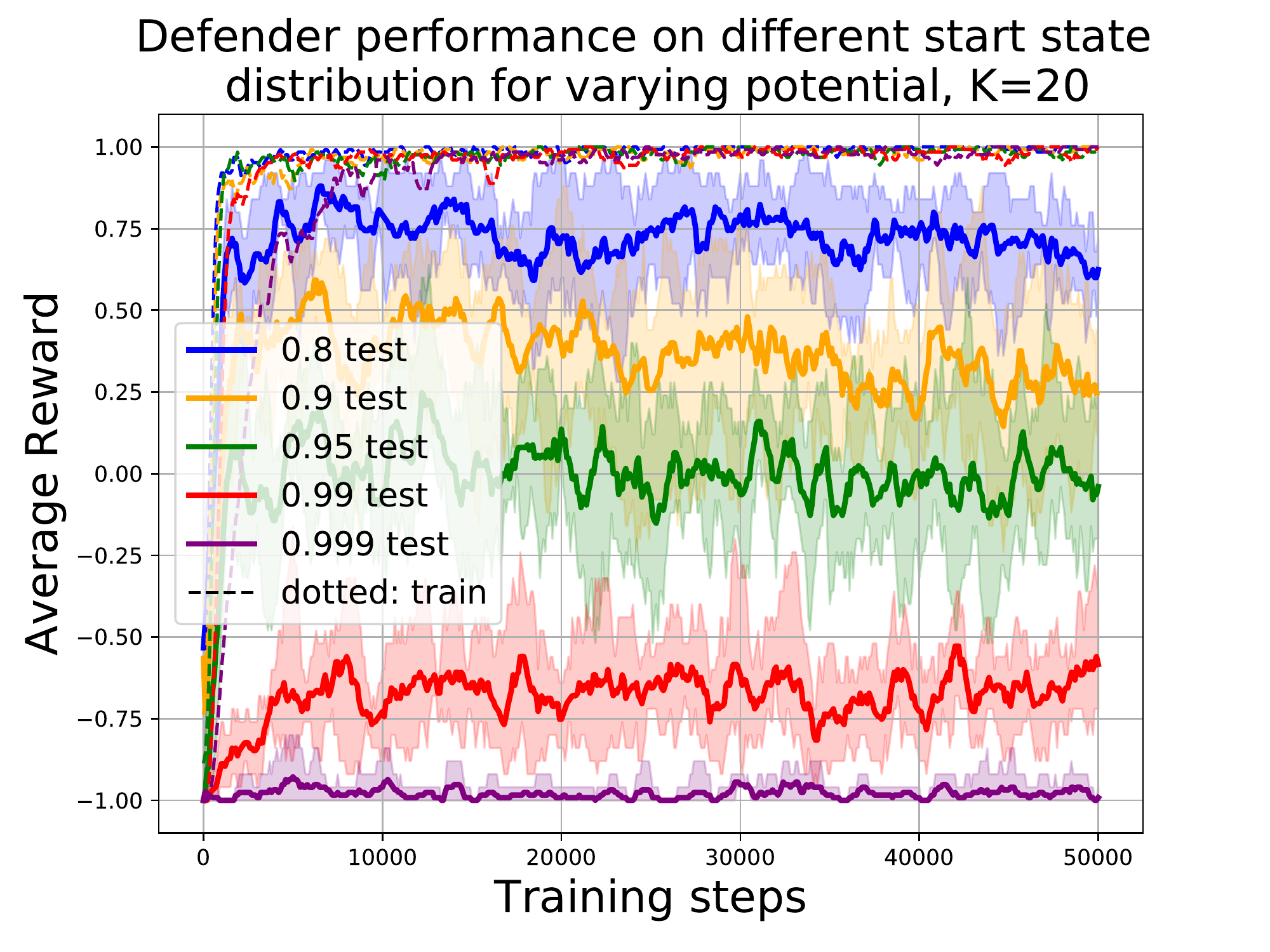}
  \end{tabular}
  \caption{Change in performance when testing on different state distributions}
  \vspace*{-1.4em}
  \label{fig-start-state}
\end{figure}

In fact, the amount of possible starting states for a given $K$ and potential $\phi(S_0) = 1$ grows super exponentially in the number of levels $K$. We can state the following theorem:

\begin{theorem}
\label{thm-states-growth}
The number of states with potential $1$ for a game with $K$ levels grows like $2^{\Theta(K^2)}$  (where $0.25 K^2 \leq \Theta(K^2) \leq 0.5 K^2$ )
\end{theorem}

We give a sketch proof.
\begin{proof}
Let such a state be denoted $S$. Then a trivial upper bound can be computed by noting that each $s_i$ can take a value up to $2^{(K-i)}$, and producting all of these together gives roughly $2^{K/2}$. 

For the lower bound, we assume for convenience that $K$ is a power of $2$ (this assumption can be avoided). Then look at the set of non-negative integer solutions of the system of simultaneous equations 
\[    a_{j-1} 2^{1-j} + a_{j} 2^{-j} = 1/K \]
where j ranges over all even numbers between $\log(K)+1$ and $K$. The equations don't share any variables, so the solution set is just a product set, and the number of solutions is just the product
 $ \prod_{j} (2^{j-1}/K)$
where, again,$ j$ ranges over even numbers between $\log(K)+1$ and $K$. This product is roughly $2^{K^2 / 4}$.
\end{proof}

\subsection{Comparison to Random Search}
Inspired by the work of \cite{mania2018randomsearch}, we also include the performance of random search in Figure \ref{}

\begin{figure}
\centering
  \begin{tabular}{cc}
  \vspace*{-5mm}
   \hspace*{-20mm}\includegraphics[width=0.5\columnwidth]{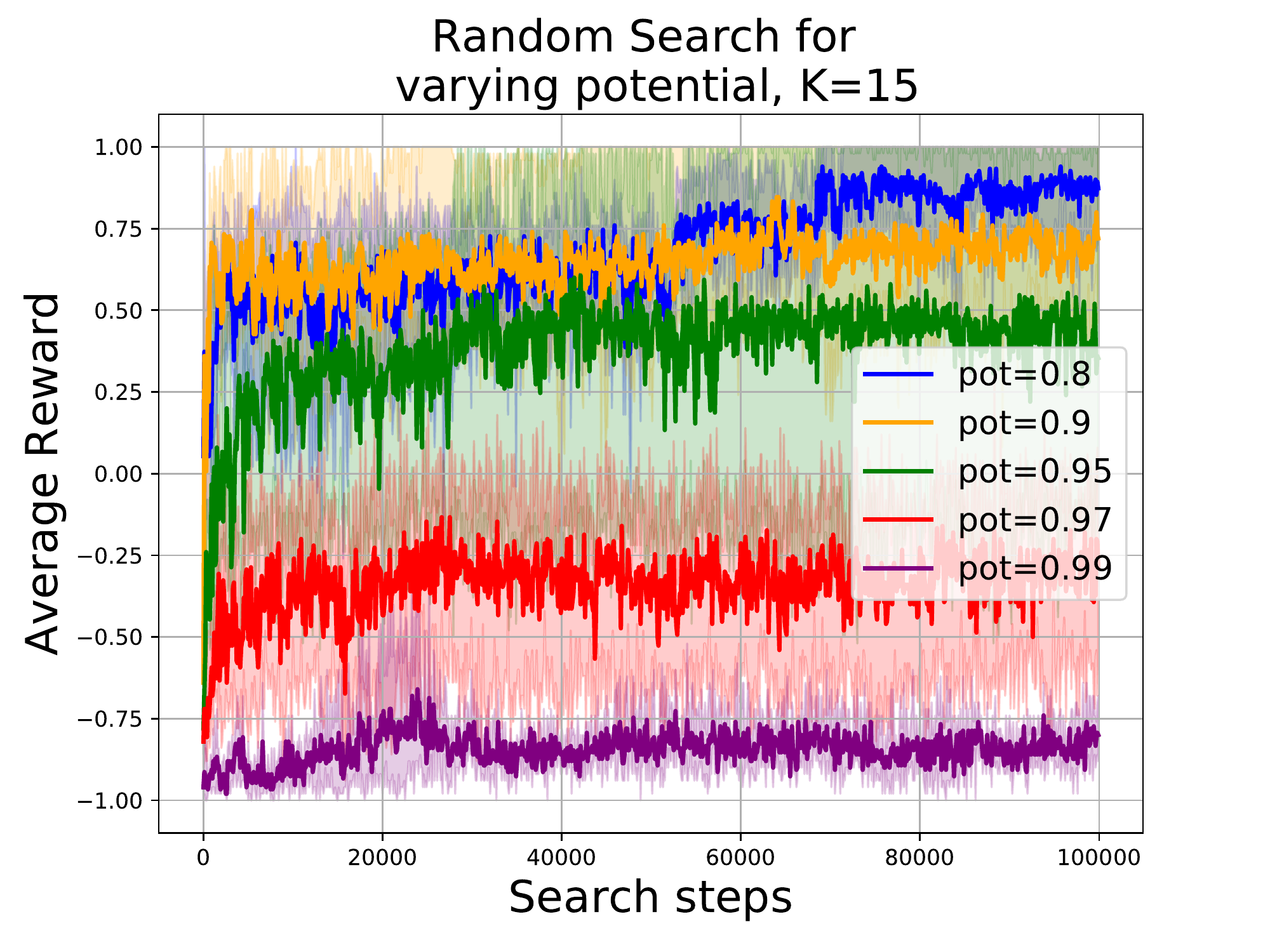}
  &
  \hspace*{-5mm}\includegraphics[width=0.5\columnwidth]{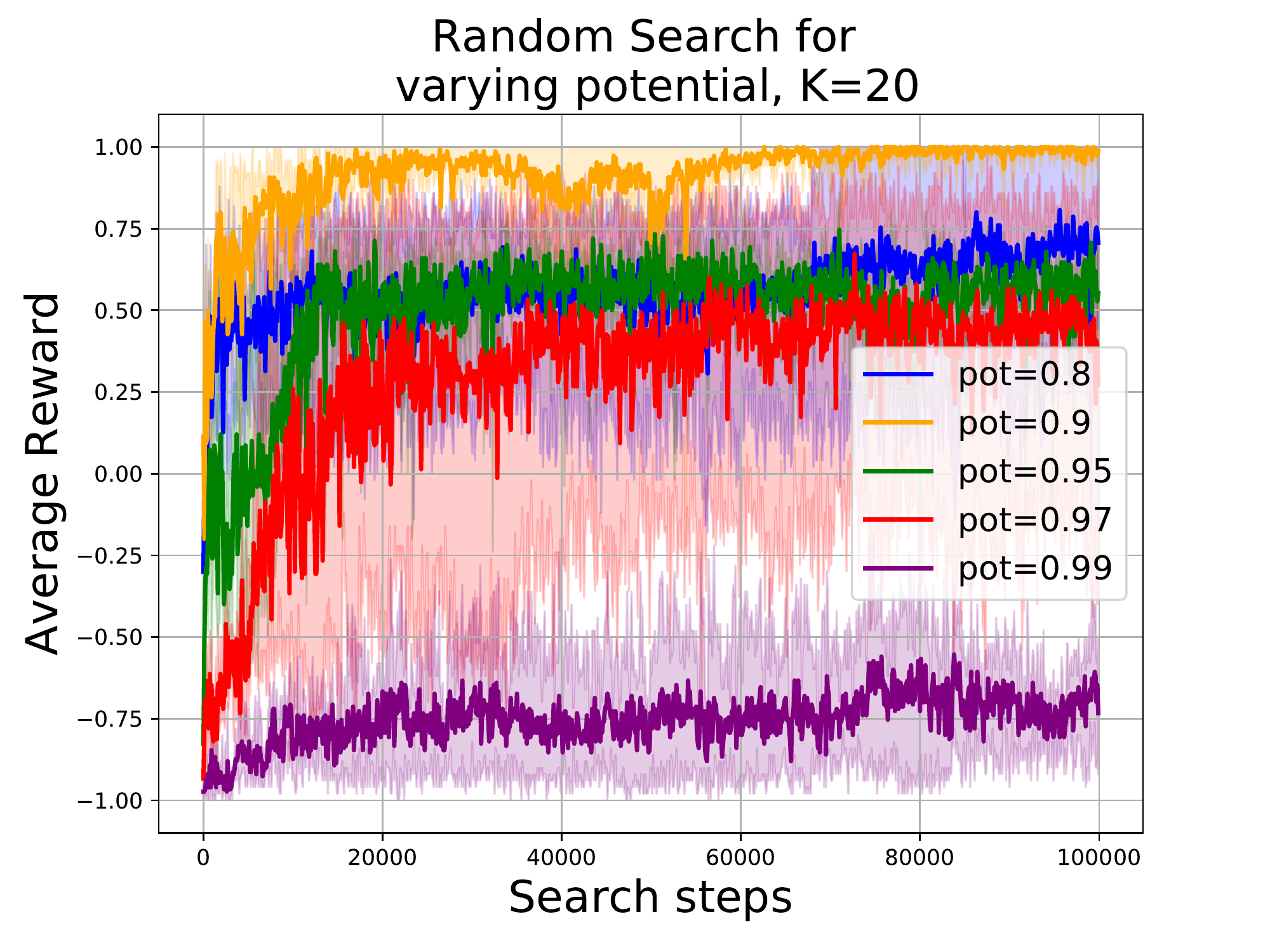}\hspace*{-9mm} 
  \\
  \hspace*{-20mm}
  \includegraphics[width=0.5\columnwidth]{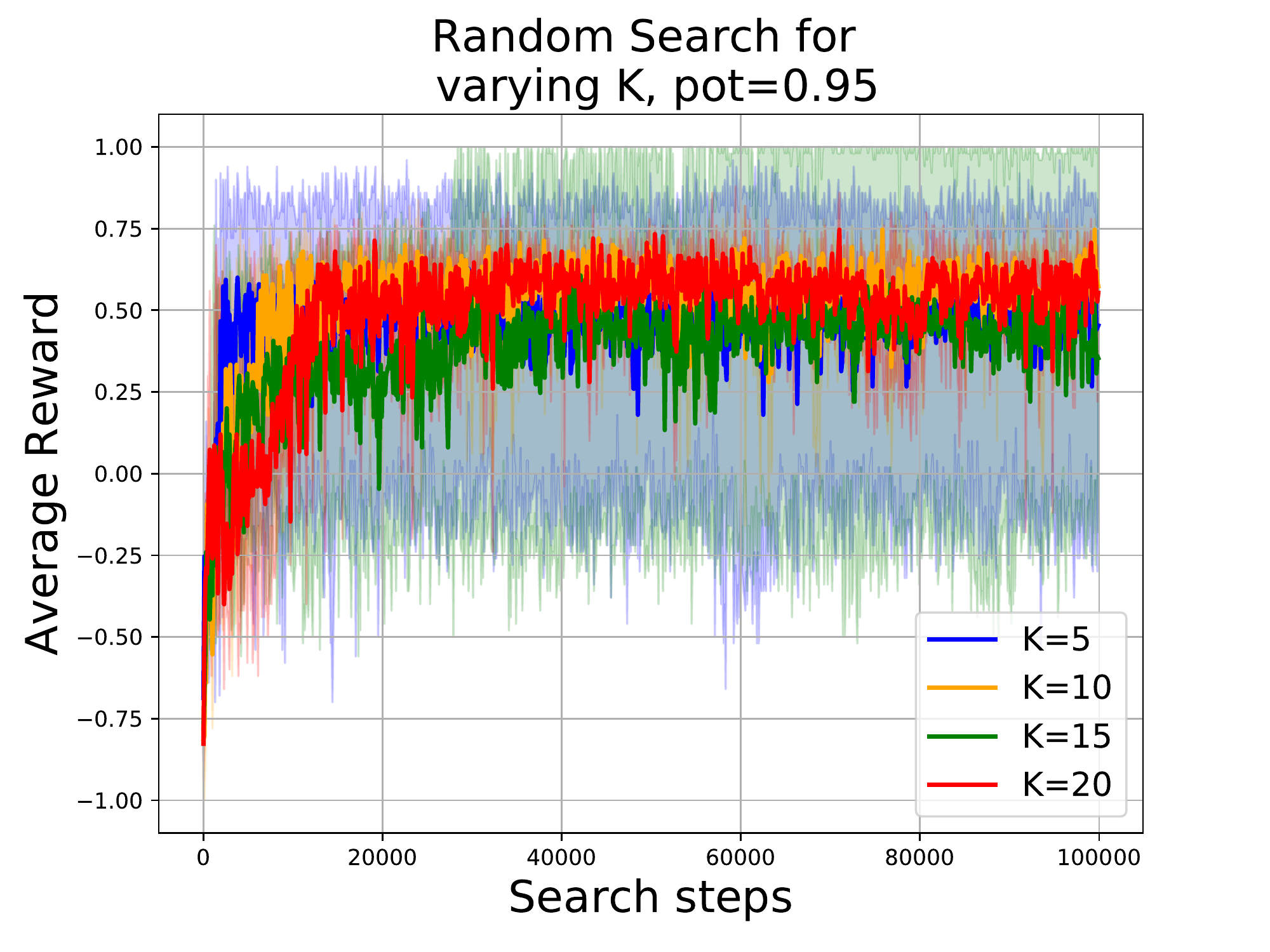}
  &
    \hspace*{-5mm}
  \includegraphics[width=0.5\columnwidth]{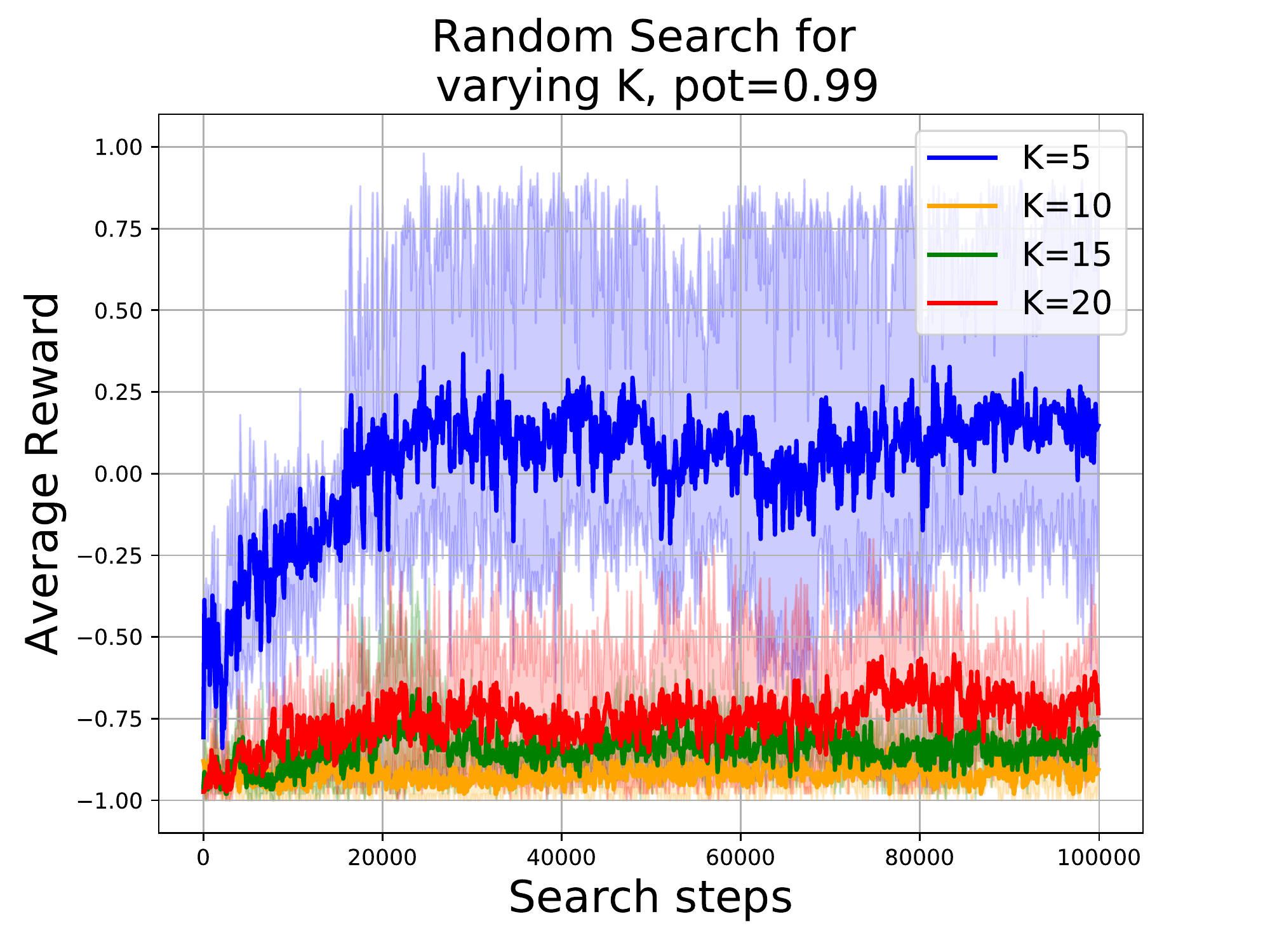}\hspace*{-9mm} 
  \\
  \end{tabular}
  \caption{\small Performance of random search from \cite{mania2018randomsearch}. The best performing RL algorithms do better.}
  \label{fig-random-search}
\end{figure}

\end{document}